\documentclass{article}

\usepackage[preprint]{neurips_2025}

\usepackage[utf8]{inputenc}       
\DeclareUnicodeCharacter{2082}{\textsubscript{2}}
\usepackage[T1]{fontenc}          
\usepackage{hyperref}
\usepackage{url}
\usepackage{booktabs}
\usepackage{amsfonts}
\usepackage{nicefrac}
\usepackage{microtype}
\usepackage[table,dvipsnames]{xcolor}  
\usepackage[pdftex]{graphicx}     
\graphicspath{{../pdf/}{../jpeg/}}
\DeclareGraphicsExtensions{.pdf,.jpeg,.png}
\usepackage{dirtytalk}
\usepackage{morewrites}
\usepackage{adjustbox}
\usepackage{amsmath}
\usepackage{amsthm}
\usepackage[font=small]{caption}  
\usepackage{pifont}
\newcommand{\cmark}{\ding{51}}
\newcommand{\xmark}{\ding{55}}
\usepackage{multirow}
\usepackage{ulem}
\usepackage{algorithm}
\usepackage{algpseudocode}
\usepackage{rotating}
\usepackage{longtable}
\usepackage{multicol}
\usepackage{todonotes}
\usepackage{float}

\newtheorem{theorem}{Theorem}[section]
\newtheorem{lemma}{Lemma}[section]        

\newtheorem{remark}[theorem]{Remark}

\title{Neuromorphic Graph Anomaly Detection via Adaptive STDP and Spiking Graph Neural Networks}

\author{%
  Abdul Joseph Fofanah\thanks{Abdul Joseph Fofanah, Lian Wen, David Chen, Tsungcheng Yao, and Kwabena Sarpong are with the School of ICT, Griffith University, Australia. (e-mail: abdul.fofanah, l.wen, david.chen, tsungcheng.yao@griffith.edu.au, ORCID: 0000-0001-8742-9325; 0000-0002-2840-6884; 0000-0001-8690-7196; 0009-0001-8795-7512; 0000-0002-2771-7074)} \\
  School of ICT\\
  Griffith University\\
  Australia \\
  \texttt{a.fofanah@griffith.edu.au} \\
  \And
  Lian Wen \\
  School of ICT\\
  Griffith University\\
  Australia \\
  \texttt{l.wen@griffith.edu.au} \\
  \AND
  David Chen \\
  School of ICT\\
  Griffith University\\
  Australia \\
  \texttt{david.chen@griffith.edu.au} \\
  \And
  Tsungcheng Yao \\
  School of ICT\\
  Griffith University\\
  Australia \\
  \texttt{tsungcheng.yao@griffith.edu.au} \\
  \And
  Kwabena Sarpong \\
  School of ICT\\
  Griffith University\\
  Australia \\
  \texttt{kwabena.sarpong@griffithuni.edu.au} \\
}

\begin{document}

\maketitle

\begin{abstract}
Anomaly detection in dynamic networks is critical for applications from cybersecurity to industrial monitoring, yet existing methods face challenges in energy efficiency, temporal precision, and adaptability. This paper introduces ASTDP-GAD, a novel \textbf{\uline{A}}daptive \textbf{\uline{S}}piking \textbf{\uline{T}}emporal \textbf{\uline{D}}ynamics \textbf{\uline{P}}lasticity framework for \textbf{\uline{G}}raph \textbf{\uline{A}}nomaly \textbf{\uline{D}}etection that integrates spiking graph neural networks with STDP learning for energy-efficient neuromorphic detection in dynamic networks. Our framework unifies spiking neural computation, STDP learning, and graph-based anomaly detection through the following key innovations: temporal spike graph encoding with adaptive Leaky Integrate-and-Fire (LIF) dynamics; LIF-based graph attention with lateral inhibition; event-driven hypergraph memory with STDP-inspired prototype updates; spike rate contrast pooling based on spiking irregularity; adaptive STDP layers capturing causal temporal relationships; and multi-scale temporal convolution with multi-factor anomaly fusion. Theoretical analysis provides rigorous guarantees: spike encoding preserves input information with resolution scaling linearly in simulation steps and hidden dimension; LIFGAT approximates any continuous attention function; hypergraph memory converges to optimal prototypes; contrast pooling achieves provable anomaly selection bounds; STDP learning converges stably; and multi-factor fusion produces calibrated scores with up to $5\times$ variance reduction. Extensive experiments on nine datasets on both dynamic and static graphs demonstrate superior anomaly detection accuracy while maintaining biological plausibility and energy efficiency for neuromorphic deployment. 
\end{abstract}

\section{Introduction}
\label{sec:introduction}

Anomaly detection in dynamic networks has emerged as a critical task across numerous domains, including cybersecurity, financial fraud detection, social network analysis, and industrial monitoring \cite{yang2025generalizable, ho2025graph, li2025deep}. The ability to identify unusual patterns or behaviours in evolving, graph-structured data is essential for maintaining system integrity, preventing failures, and enabling timely interventions. Despite significant advances in graph neural networks for anomaly detection \cite{zhang2025ge, latif2025gat}, accurately identifying anomalies in streaming, dynamic networks remains particularly challenging due to the inherent complexity of temporal dependencies, the rarity of anomalous events, and the stringent energy constraints of real-time deployment scenarios. These factors can hinder the effectiveness of existing models and necessitate the development of more robust algorithms tailored for such conditions.

Current approaches face three fundamental limitations in neuromorphic graph anomaly detection settings. First, existing methods predominantly operate on continuous-valued representations and synchronous computation paradigms \cite{alzamil2026dualbert, alzahrani2026real, kundu2026comprehensive}, failing to leverage the energy efficiency and temporal precision of spiking neural mechanisms. Second, they lack biologically plausible learning rules that can capture the precise timing relationships critical for detecting subtle anomalies in temporal patterns \cite{okano2026analytics, wang2025mst}. Third, standard message-passing schemes cannot adapt to the varying dynamical regimes present in dynamic networks, where normal behaviour may itself exhibit complex, time-varying characteristics \cite{dou2025deanomaly}, \cite{fofanah2025gant}. 

These limitations point toward an alternative paradigm: neuromorphic computing, inspired by the principles of biological neural systems, offers a transformative approach for energy-efficient, event-driven computation \cite{kamble2025neuromorphic, zhu2026mechanical}. At the core of neuromorphic intelligence are spiking neural networks (SNNs) that process information through discrete spikes in time, mimicking the communication protocol of biological neurones \cite{greatorex2025neuromorphic}. Spike-Timing-Dependent Plasticity (STDP), a biologically observed learning rule, adjusts synaptic strengths based on the relative timing of pre- and post-synaptic spikes \cite{jahin2026chronospike}, enabling unsupervised discovery of temporal patterns \cite{shi2025eegsnet, alharbi2025integrated}. These principles collectively enable computation that is both highly energy-efficient and naturally suited to processing temporal streaming data. However, current anomaly detection frameworks largely ignore these neuromorphic principles, missing opportunities for ultralow-power, real-time anomaly detection that could revolutionise edge deployment scenarios.

Concretely, we address three interrelated challenges in constructing anomaly detection systems for dynamic networks under neuromorphic constraints: (1) \textit{how to encode continuous graph-structured data into spike-based representations} that preserve information while enabling energy-efficient neuromorphic computation, balancing fidelity with the sparse, event-driven nature of spiking processing; (2) \textit{how to design learning mechanisms that capture both spatial and temporal dependencies through biologically plausible plasticity rule}s such as STDP, moving beyond gradient-based optimisation to discover causal relationships from spike timing alone; and (3) \textit{how to ensure computational tractability and energy efficiency while maintaining high anomaly detection accuracy in large-scale networks}, where processing spike-based representations on neuromorphic hardware imposes strict constraints on architecture and computation, making scalability to thousands of nodes under edge deployment latency requirements a critical practical hurdle.

This work presents the first principled integration of spiking neural networks with STDP-based unsupervised learning for graph-based anomaly detection in dynamic networks. While prior spiking graph networks exist for fault diagnosis \cite{Li2024ANL, Zhao2024DynamicRS}, they lack STDP-driven temporal pattern discovery and fail to address the unique challenges of evolving graph structures through mechanisms such as spiking attention, hypergraph memory, and multi-factor fusion. To overcome these limitations, we propose \textbf{ASTDP-GAD}, a novel framework that combines spiking neural mechanisms, STDP learning, and adaptive graph processing for energy-efficient neuromorphic graph anomaly detection. Key innovations include: a \textit{temporal spike graph encoder} using adaptive Leaky Integrate-and-Fire (LIF) dynamics; a \textit{spiking graph attention} mechanism with lateral inhibition; an \textit{event-driven hypergraph memory} employing STDP-inspired updates; \textit{spike rate contrast pooling} to flag spiking irregularities; and a \textit{multi-factor anomaly fusion} approach with theoretical calibration. Operationally, ASTDP-GAD proceeds through three phases: (1) \textit{Spike Encoding and Temporal Representation}, transforming features into spike tensors and counts via adaptive LIF neurones; (2) \textit{STDP Learning and Memory Formation}, processing spike representations through LIF attention layers and updating prototypes via plasticity; and (3) \textit{Multi-Factor Anomaly Detection and Fusion}, integrating signals from attention, memory, pooling, STDP strength, and temporal analysis to produce calibrated, energy-efficient anomaly scores.

\textit{Our primary contributions are as follows}: \textbf{First}, we establish a neuromorphic foundation for energy-efficient anomaly detection, proving theoretical guarantees for spike-based information preservation where the temporal spike graph encoder uniquely identifies input features with resolution scaling linearly in simulation steps and hidden dimension. \textbf{Second}, we design ASTDP-GAD, the first end-to-end architecture unifying spiking neural networks with graph-based anomaly detection, featuring LIF-based graph attention with lateral inhibition, event-driven hypergraph memory with STDP updates, and adaptive spike rate contrast pooling operating on precise spike-timing representations. \textbf{Third}, we develop a holistic multi-factor anomaly fusion mechanism with theoretical guarantees that combines complementary signals such as graph attention patterns, prototype memory deviations, spiking irregularity, STDP strength, and multi-scale temporal inconsistencies through a calibrated approach preserving unbiasedness and minimising variance under conditional independence. \textbf{Fourth}, we provide rigorous theoretical analysis proving convergence properties for hypergraph memory, selection guarantees for contrast pooling, stability for STDP learning, and calibration for multi-factor fusion. 

The remainder of this paper is organised as follows. Section~\ref{sec:preliminary} formulates the problem, and Section~\ref{sec:method} presents the ASTDP-GAD methodology. Experimental results and conclusions are reported in Sections~\ref{sec:experiments} and \ref{sec:conclusion}, respectively. The appendix provides additional motivation (Section~\ref{sec:motivation}), theoretical proofs (Section~\ref{sec:theoretical}), a review of related works (Section~\ref{sec:related_works}), supplementary experiments (Section~\ref{sec:add_experiments}), and a detailed discussion of findings and future directions (Section~\ref{sec:discussion}).
\section{Preliminaries}
\label{sec:preliminary}

Let \(\mathcal{G}=(\mathcal{V},\mathcal{E})\) be a graph with \(|\mathcal{V}|=N\) nodes. A sequence of continuous node features \(\{\mathbf{x}^{(t)}\}_{t=1}^T\subset\mathbb{R}^F\) is encoded via a spiking encoder into a spike tensor \(\mathbf{S}\in\{0,1\}^{T\times N\times H}\) (time steps × nodes × hidden dimension), spike times \(\mathbf{T}\in\mathbb{R}^{N\times H}\) (first spike time per neuron, else \(T\)), and spike counts \(\mathbf{C}\in\mathbb{Z}^{N\times H}\) (total spikes per neuron). We denote by \(\|\cdot\|\) the Euclidean norm for vectors and the Frobenius norm for matrices.

\noindent \textit{Definition 1: Spiking Graph.} At time step \(t\), a spiking graph is \(\mathcal{G}^{(t)}=(\mathcal{V},\mathcal{E}^{(t)},\mathbf{S}[t,:,:])\), where \(\mathcal{E}^{(t)}\) are dynamic edges and \(\mathbf{S}[t,:,:]\in\{0,1\}^{N\times H}\) is the spike matrix at time step \(t\).

\noindent \textit{Definition 2: Neuromorphic Computation.} The Leaky Integrate-and-Fire (LIF) model governs neuromorphic dynamics. For a neuron with membrane potential \(u\), synaptic current \(i\), and output spike \(s\in\{0,1\}\), the discrete-time update is:
\begin{equation}
\begin{aligned}
i[t] = \alpha i[t-1] + \sum_j w_j x_j[t],
u[t] = \beta u[t-1] + i[t] - v_{\text{reset}} s[t-1],
s[t] = \mathbb{I}(u[t]\ge\theta),
\end{aligned}
\end{equation}
with decay factors \(\alpha=\exp(-1/\tau_{\text{syn}})\), \(\beta=\exp(-1/\tau_{\text{mem}})\), synaptic weights \(w_j\), input spikes \(x_j[t]\), threshold \(\theta\), and indicator \(\mathbb{I}(\cdot)\). After a spike, the membrane potential is reset: \(u[t]\leftarrow u[t](1-s[t])\).

\noindent \textit{Definition 3: Spike-Timing-Dependent Plasticity (STDP).} STDP updates synaptic weight \(w_{ij}\) based on relative spike times. For pre-spike at \(t_{\text{pre}}\) and post-spike at \(t_{\text{post}}\):
\begin{equation}
\Delta w_{ij} = \begin{cases}
A_+ \exp\!\left(-\frac{|t_{\text{post}}-t_{\text{pre}}|}{\tau_+}\right), & t_{\text{post}}>t_{\text{pre}}\ \text{(potentiation)},\\[4pt]
-A_- \exp\!\left(-\frac{|t_{\text{post}}-t_{\text{pre}}|}{\tau_-}\right), & t_{\text{post}}<t_{\text{pre}}\ \text{(depression)},
\end{cases}
\end{equation}
with learning rates \(A_+,A_-\) and time constants \(\tau_+,\tau_-\).

\noindent \textit{Definition 4: Anomaly Detection in Dynamic Graphs.} For node \(v_i\) at time step \(t\), anomaly score \(a_i^{(t)}\in[0,1]\) indicates deviation from normality. Anomalies can be point, contextual, collective, or temporal.

\noindent \textit{Problem: Neuromorphic Anomaly Detection with STDP Learning.} Given dynamic graph \(\mathcal{G}^{(t)}=(\mathcal{V},\mathcal{E}^{(t)},\mathbf{X}^{(t)})\) with features \(\mathbf{X}^{(t)}\in\mathbb{R}^{N\times F}\), we encode a sliding window of \(T\) time steps into spike tensor \(\mathbf{S}\in\{0,1\}^{T\times N\times H}\), spike times \(\mathbf{T}\in\mathbb{R}^{N\times H}\), and spike counts \(\mathbf{C}\in\mathbb{Z}^{N\times H}\). A function \(f_{\Theta}:(\mathbf{S},\mathcal{E},\mathbf{T},\mathbf{C})\rightarrow\mathbf{a}^{(t)}\) produces anomaly scores via multiple components:
\[
\mathbf{a}^{(t)} = f_{\Theta}\!\big( \mathcal{A}(\mathbf{S},\mathcal{E},\mathbf{T},\mathbf{C}),\ \mathcal{H}(\mathbf{S},\mathbf{T},\mathbf{C}),\ \mathcal{P}(\mathbf{S},\mathcal{E},\mathbf{C}),\ \Delta\mathbf{W}(\mathbf{T}),\ \mathcal{K}(\mathbf{S}) \big),
\]
where \(\mathcal{A}\) (LIFGAT), \(\mathcal{H}\) (EDHMM), \(\mathcal{P}\) (SRCGP), \(\Delta\mathbf{W}\) (STDP), and \(\mathcal{K}\) (MSTC) are core components. Parameters \(\Theta\) are optimised by minimising:
\[
\min_{\Theta} \mathbb{E}_{(\mathcal{G},\mathbf{y})\sim\mathcal{D}}\big[ \mathcal{L}_{\text{total}}(\mathbf{y},\mathbf{a};\Theta) \big],
\]
with ground-truth labels \(\mathbf{y}\in\{0,1\}^N\). STDP updates are applied alongside gradient-based optimisation.
\section{The Proposed Method}
\label{sec:method}
In this section, we present the ASTDP-GAD framework, which integrates spiking neural mechanisms, STDP, and adaptive graph learning through six core components: temporal spike graph encoding, LIF graph attention, event-driven hypergraph memory, spike rate contrast pooling, adaptive STDP learning, and multi-scale temporal convolution with anomaly fusion (Figure~\ref{fig:astdp_gad_architecture}). Model parameters are learned via multi-objective optimisation that minimises a combined anomaly detection loss across complementary detection pathways (Eq.~\ref{eq:total_loss}).

\begin{figure*}[ht]
    \centering
    \includegraphics[width=\linewidth]{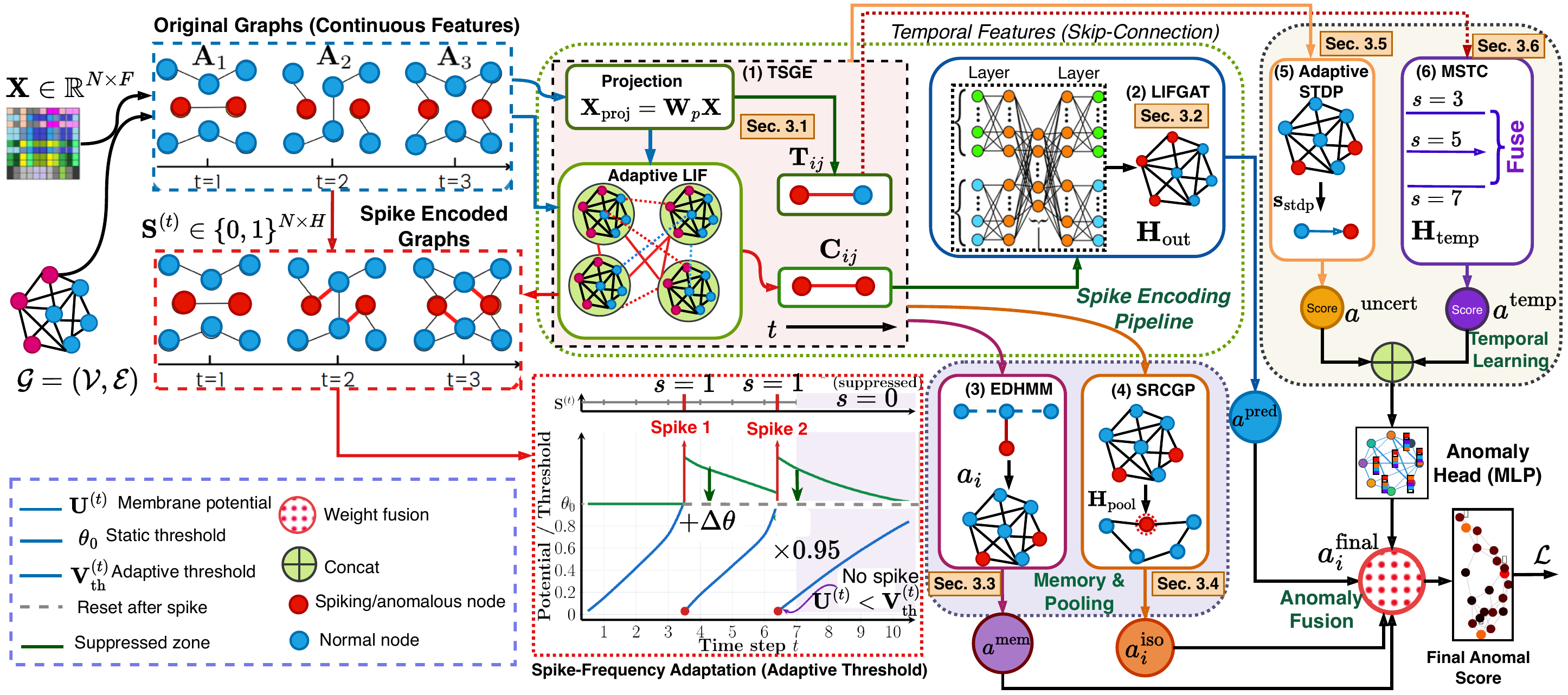}
    \caption{Overall architecture of the proposed ASTDP-GAD framework for neuromorphic anomaly detection. The model processes graph-structured data through: temporal spike graph encoding (Sec.~\ref{sec:spike_encoding}); LIF-based graph attention (Sec.~\ref{sec:lif_attention}); event-driven hypergraph memory (Sec.~\ref{sec:hypergraph_memory}); spike rate contrast pooling (Sec.~\ref{sec:contrast_pooling}); adaptive STDP learning (Sec.~\ref{sec:stdp_learning}); and multi-scale temporal convolution with anomaly feature fusion (Sec.~\ref{sec:temporal_fusion}).}
    \label{fig:astdp_gad_architecture}
\end{figure*}

\subsection{Temporal Spike Graph Encoding}
\label{sec:spike_encoding}

Traditional graph neural networks operate on continuous-valued features and fail to capture the precise temporal dynamics required for neuromorphic computation \cite{ai2025cross}. To address this limitation, we introduce a Temporal Spike Graph Encoder (TSGE) that converts continuous node features into spike-based representations through a biologically-plausible leaky integrate-and-fire (LIF) neuron model with adaptive threshold mechanisms. This encoding preserves temporal information critical for anomaly detection while enabling energy-efficient neuromorphic processing.

\textit{Spike encoding with adaptive LIF dynamics:}
Given input node features $\mathbf{X} \in \mathbb{R}^{N \times F}$, we first project them to the hidden dimension via $\mathbf{X}_{\text{proj}} = \mathbf{W}_p \mathbf{X}$ with $\mathbf{W}_p \in \mathbb{R}^{F \times H}$. The projected features are processed over $T$ discrete time steps by an adaptive LIF model. Let $\mathbf{U}^{(t)} \in \mathbb{R}^{N \times H}$ be the membrane potential, $\mathbf{I}^{(t)} \in \mathbb{R}^{N \times H}$ the synaptic current, and $\mathbf{S}^{(t)} \in \{0,1\}^{N \times H}$ the output spikes at time $t$. The dynamics are governed by:
\begin{align}
\mathbf{I}^{(t)} &= \mathbf{I}^{(t-1)} \alpha + \frac{1}{T} \mathbf{X}_{\text{proj}} \mathbf{W}, \quad \alpha = \exp(-1/\tau_{\text{syn}}), \nonumber 
\mathbf{U}^{(t)} = \mathbf{U}^{(t-1)} \beta + \mathbf{I}^{(t)} \odot \mathbf{f}_{\text{syn}}, \nonumber \\
\quad \beta &= \exp(-1/\tau_{\text{mem}}),\mathbf{S}^{(t)} = \mathbb{I}\left(\mathbf{U}^{(t)} \geq \mathbf{V}_{\text{th}}^{(t)}\right), \quad
\mathbf{U}^{(t)} \leftarrow \mathbf{U}^{(t)} \odot (1 - \mathbf{S}^{(t)}),
\end{align}
\begin{align}
\mathbf{V}_{\text{th}}^{(t)} = \mathbf{V}_{\text{th}}^{(0)} \odot \mathbf{a}_{\text{adapt}} + \mathbf{V}_{\text{adapt}}^{(t-1)}, \nonumber 
\mathbf{V}_{\text{adapt}}^{(t)} = \lambda_{\text{adapt}} \, \mathbf{V}_{\text{adapt}}^{(t-1)} + \eta_{\text{adapt}} \, \mathbf{S}^{(t)}.
\end{align}
with synaptic time constant $\tau_{\text{syn}}$, membrane time constant $\tau_{\text{mem}}$, base threshold $\theta$, adaptation decay factor $\lambda_{\text{adapt}} \in (0,1)$, adaptation increment $\eta_{\text{adapt}} > 0$, and learnable parameters $\mathbf{W} \in \mathbb{R}^{H \times H}$, $\mathbf{f}_{\text{syn}} \in \mathbb{R}^H$, and $\mathbf{a}_{\text{adapt}} \in \mathbb{R}^H$. The adaptive threshold increases after each spike and decays exponentially, implementing spike-frequency adaptation observed in biological neurons.

\textit{Spike-time and rate coding:}
From the spike tensor $\mathbf{S} \in \{0,1\}^{N \times T \times H}$, we extract two complementary representations: first-spike times $\mathbf{T} \in \mathbb{R}^{N \times H}$ and spike counts $\mathbf{C} \in \mathbb{Z}^{N \times H}$:
\begin{equation}
\mathbf{T}_{ij} = \begin{cases}
\min\{t : \mathbf{S}_{ijt} = 1\}, & \text{if neuron spikes} \\
T, & \text{otherwise}
\end{cases}, \quad
\mathbf{C}_{ij} = \sum_{t=1}^{T} \mathbf{S}_{ijt}.
\end{equation}
These capture both temporal precision (spike timing) and rate-based information (spike counts), providing complementary views of neural activity. Theoretical analysis (Theorem~\ref{thm:spike_encoding}) guarantees that for any bounded input $\|\mathbf{x}\|_2 \leq B$, the resulting spike pattern uniquely identifies $\mathbf{x}$ up to a resolution that improves with $T$ and $H$, ensuring sufficient information fidelity for accurate anomaly detection while enabling energy-efficient neuromorphic computation.

\subsection{Leaky Integrate-and-Fire Graph Attention}
\label{sec:lif_attention}
Standard graph attention mechanisms lack the temporal dynamics required for processing spike-based information \cite{latif2025gat}. To address this, we introduce a LIF Graph Attention (LIFGAT) layer that integrates spike-timing information with multi-head attention through biologically-plausible membrane potential dynamics and lateral inhibition.

\textit{Spike-timing modulated attention computation:}
Given node representations $\mathbf{H} \in \mathbb{R}^{N \times H}$, spike times $\mathbf{T} \in \mathbb{R}^{N \times H}$, and spike counts $\mathbf{C} \in \mathbb{R}^{N \times H}$, we first compute query, key, and value projections:
\begin{align}
\mathbf{Q} = \mathcal{L}_{\text{query}}(\mathbf{H}) \in \mathbb{R}^{N \times M \times D},
\mathbf{K} = \mathcal{L}_{\text{key}}(\mathbf{H}) \in \mathbb{R}^{N \times M \times D},
\mathbf{V} = \mathcal{L}_{\text{value}}(\mathbf{H}) \in \mathbb{R}^{N \times M \times D},
\end{align}
where $M$ is the number of heads and $D = H/M$. To incorporate spike-timing information, we define a temporal modulation factor and compute attention scores as:
\begin{equation}
\begin{aligned}
    \mathbf{\Gamma} = \left(1 + \frac{\mathbf{C}}{\max(\mathbf{C}) + \epsilon} \cdot \gamma\right) \cdot \mathbb{I}(\bar{\mathbf{T}} \leq t), 
    \mathbf{A} = \mathrm{softmax}\!\left(\frac{\mathbf{Q}\mathbf{K}^\top}{\sqrt{D}} \odot \mathbf{\Gamma}_{\text{exp}}\right) \in \mathbb{R}^{N \times N \times M},
\end{aligned}
\end{equation}
where $\gamma$ is a learnable parameter, $\bar{\mathbf{T}}$ is the average spike time, $t$ is the current simulation step, and $\mathbf{\Gamma}_{\text{exp}}$ expands the modulation factor appropriately. This ensures that neurons with earlier or more frequent spikes exert greater influence on attention computations.

\textit{Membrane potential dynamics with lateral inhibition:}
The aggregated messages are integrated into a membrane potential evolving over time, with lateral inhibition across heads to promote sparse representations:
\begin{equation}
\begin{aligned}
    \mathbf{U}_{\text{att}}^{(t)} = \mathbf{U}_{\text{att}}^{(t-1)} e^{-1/\tau_{\text{mem}}} + \sum_j \mathbf{A}_{ij} \mathbf{V}_j, 
    \mathbf{U}_{\text{att},m}^{(t)} = \mathbf{U}_{\text{att},m}^{(t)} - \sum_{k<m} \mathbf{S}_{\text{att},k}^{(t)} \mathbf{L}_m,
\end{aligned}
\end{equation}
where $\tau_{\text{mem}} = 20.0$ is the membrane time constant and $\mathbf{L}_m \in \mathbb{R}^M$ are learnable lateral inhibition weights. Spikes are generated when the potential exceeds threshold $\theta_{\text{att}}$, and the final output is obtained by temporal averaging followed by a linear projection:
\begin{equation}
\begin{aligned}
    \mathbf{S}_{\text{att}}^{(t)} = \mathbb{I}\bigl(\mathbf{U}_{\text{att}}^{(t)} \ge \theta_{\text{att}}\bigr), \quad \mathbf{U}_{\text{att}}^{(t)} \leftarrow \mathbf{U}_{\text{att}}^{(t)} \odot \bigl(1 - \mathbf{S}_{\text{att}}^{(t)}\bigr),
    \mathbf{H}_{\text{out}} = \mathcal{L}_{\text{out}}\!\left(\frac{1}{T}\sum_{t=1}^T \mathbf{S}_{\text{att}}^{(t)}\right) \in \mathbb{R}^{N \times H},
\end{aligned}
\end{equation}
where the reset term $\bigl(1 - \mathbf{S}_{\text{att}}^{(t)}\bigr)$ zeros the membrane potential upon spiking, and $\mathcal{L}_{\text{out}}$ is a linear projection layer.
In the theoretical analysis (Theorem~\ref{thm:lifgat_approximation}), we guarantee that LIFGAT can approximate any continuous attention function with bounded error, establishing its representational capacity for spike-based graph processing.

\subsection{Event-Driven Hypergraph Memory}
\label{sec:hypergraph_memory}
Neuromorphic anomaly detection requires the ability to store and retrieve prototypical patterns while adapting to novel observations \cite{Li2024ANL}. We introduce an Event-Driven Hypergraph Memory (EDHMM) component that maintains a set of prototype vectors representing normal behaviour patterns, with update dynamics governed by spike-timing-dependent plasticity principles.

\textit{Prototype matching and anomaly scoring:}
The EDHMM maintains $P$ prototype vectors $\{\mathbf{p}_k\}_{k=1}^P \in \mathbb{R}^H$ that capture typical spike patterns. Given spike tensor $\mathbf{S} \in \{0,1\}^{N \times T \times H}$, we first compute a combined representation integrating rate-based and temporal information, then measure similarity to each prototype:
\begin{equation}
\begin{aligned}
    \mathbf{Z} = \underbrace{\sum_{t=1}^{T} \mathbf{S}_{:,:,t}}_{\text{rate code}} + \eta \underbrace{\sum_{t=1}^{T} \mathbf{S}_{:,:,t} \odot \mathbf{K}_t}_{\text{temporal code}}, 
    \mathbf{D}_{ik} = \|\mathbf{Z}_i - \mathbf{p}_k\|_2,
    \mathbf{M}_{ik} = \exp\!\left(-\frac{\mathbf{D}_{ik}}{\tau_{\text{temp}}}\right) \cdot \mathbf{h}_k \cdot \mathbf{s}_k,
\end{aligned}
\end{equation}
where $\mathbf{K}_t \in \mathbb{R}^H$ are learnable temporal kernels, $\eta \in (0,1)$ balances the two code contributions, $\tau_{\text{temp}}$ is the temperature parameter, $\mathbf{h}_k$ is the homeostatic scaling factor, and $\mathbf{s}_k$ is the pattern strength of prototype $k$. The anomaly score for node $i$ combines three indicators:
\begin{align}
a_i = \underbrace{1 - \sigma\!\left(\max_k \mathbf{M}_{ik} - \mu_{\text{match}}\right)}_{\text{prototype mismatch}} \times \nonumber
      \underbrace{\left(1 + \|\mathbf{Z}_i - \mathbf{p}_{k^*}\|_2\right)}_{\text{residual magnitude}} \times \nonumber
      &\underbrace{\left(1 + \frac{\max_k \mathbf{M}_{ik} - \text{second}_k \mathbf{M}_{ik}}{\max_k \mathbf{M}_{ik} + \epsilon}\right)}_{\text{pattern uncertainty}},
\end{align}
where $k^* = \arg\max_k \mathbf{M}_{ik}$ and $\mu_{\text{match}}$ is a matching threshold. This multi-factor score captures deviations from prototypes, residual errors, and assignment ambiguity.

\textit{STDP-inspired memory update:}
Prototypes are updated using a rule inspired by spike-timing-dependent plasticity:
\begin{equation}
\mathbf{p}_k \leftarrow \mathbf{p}_k + \alpha_{\text{mem}} \cdot \frac{\sum_{i: k_i^* = k} (\mathbf{Z}_i - \mathbf{p}_k)}{\sum_i \mathbb{I}(k_i^* = k) + \epsilon},
\end{equation}
with learning rate $\alpha_{\text{mem}} \in (0,1)$. Pattern strength evolves via Hebbian plasticity: $\mathbf{s}_k \leftarrow \lambda_s \,\mathbf{s}_k + \eta_s \,\sigma(n_k / \theta_s)$, where $n_k$ is the number of nodes assigned to prototype $k$, $\lambda_s$ is the decay factor, $\eta_s$ is the learning increment, and $\theta_s$ is a scaling constant. Homeostatic scaling decays gradually: $\mathbf{h}_k \leftarrow \lambda_h \,\mathbf{h}_k$ with $\lambda_h \in (0,1)$, preventing prototype overuse. Theoretical analysis (Theorem~\ref{thm:edhmm_convergence}) guarantees that the EDHMM update converges to a set of prototypes that locally minimise the reconstruction error, providing stable memory for normal patterns while remaining adaptive to gradual distribution shifts.

\subsection{Spike Rate Contrast Pooling}
\label{sec:contrast_pooling}
Graph pooling in neuromorphic systems must preserve information about anomalous nodes while reducing computational load. We propose a Spike Rate Contrast Pooling (SRCGP) mechanism that selects nodes based on their spiking irregularity, measured through coefficient of variation (CV) of inter-spike intervals and burst detection \cite{ai2025cross}.

\textit{Irregularity-based node selection:}
Given spike tensor $\mathbf{S} \in \{0,1\}^{N \times T \times H}$, we first compute spike rates $\mathbf{R} = \frac{1}{T}\sum_{t=1}^T \mathbf{S}_{:,:,t} \in \mathbb{R}^{N \times H}$. For each node $i$ and feature $j$, let $\{t_{ij}^{(1)}, t_{ij}^{(2)}, \dots\}$ be the spike times with inter-spike intervals $\Delta_{ij}^{(l)} = t_{ij}^{(l+1)} - t_{ij}^{(l)}$. The coefficient of variation and burst detection score for node $i$ are:
\begin{equation}
\begin{aligned}
    \mathrm{CV}_i = \frac{1}{H}\sum_{j=1}^H \frac{\sigma(\{\Delta_{ij}^{(l)}\})}{\mu(\{\Delta_{ij}^{(l)}\}) + \epsilon}, 
    \mathrm{Burst}_i = \frac{1}{H}\sum_{j=1}^H \frac{|\{l : \Delta_{ij}^{(l)} < 3\}|}{|\{\Delta_{ij}^{(l)}\}| + \epsilon},
\end{aligned}
\end{equation}
where $\sigma$ and $\mu$ denote the standard deviation and mean of the ISIs, and burst detection identifies spike sequences with short inter-spike intervals.
The irregularity score combines these measures with learnable weights $\gamma_{\mathrm{CV}}, \gamma_{\mathrm{burst}}$: $\mathrm{score}_i = \gamma_{\mathrm{CV}} \cdot \mathrm{CV}_i + \gamma_{\mathrm{burst}} \cdot \mathrm{Burst}_i.$
Nodes with higher scores exhibit more irregular spiking patterns and are more likely to be anomalous.

\textit{Pooled representation and anomaly isolation:}
The top-$K$ nodes with highest scores are selected, where $K = \lceil \rho N \rceil$ and $\rho = 0.5$ is the pooling ratio. Pooled features are obtained by gathering spike rates of selected nodes and projecting them, with anomaly isolation scores computed as burstiness-amplified z-scores:
\begin{equation}
\begin{aligned}
    \mathbf{R}_{\mathrm{pool}} = \mathbf{R}[\mathrm{idx}, :] \in \mathbb{R}^{K \times H}, 
    \mathbf{H}_{\mathrm{pool}} &= \mathbf{R}_{\mathrm{pool}}\, \mathbf{W}_{\mathrm{pool}} \in \mathbb{R}^{K \times H}, \\
    a_i^{\mathrm{iso}} &= \left|\frac{\mathrm{score}_i - \mu_{\mathrm{score}}}{\sigma_{\mathrm{score}} + \epsilon}\right| \cdot (1 + \mathrm{Burst}_i).
\end{aligned}
\end{equation}
where $a_i^{\mathrm{iso}}$ provides a relative measure of how anomalous each node appears compared to the population. In Theorem~\ref{thm:srcgp_selection}, we prove that under mild conditions, the probability of selecting a truly anomalous node is strictly greater than that of selecting a normal node, establishing the theoretical foundation for irregularity-based pooling.

\subsection{Adaptive STDP Learning}
\label{sec:stdp_learning}

Spike-timing-dependent plasticity is a key mechanism for biological learning that strengthens or weakens connections based on the relative timing of pre- and post-synaptic spikes \cite{rahman2025modulated, lu2024deep}. We introduce an Adaptive STDP Layer that implements this learning rule within the differentiable framework, enabling the model to discover temporal relationships critical for anomaly detection.

\textit{STDP synaptic dynamics and anomaly strength:}
The STDP layer maintains a weight matrix $\mathbf{W} \in \mathbb{R}^{H \times H}$ that evolves according to spike timing. Given input $\mathbf{X} \in \mathbb{R}^{N \times H}$ and spike times $\mathbf{T} \in \mathbb{R}^{N \times H}$, the forward pass computes $\mathbf{Y} = \mathbf{X} \mathbf{W}$. During training, the average spike times across the batch drive the STDP update:
\begin{equation}
\begin{aligned}
    \bar{\mathbf{t}} = \frac{1}{N}\sum_{i=1}^N \mathbf{T}_i \in \mathbb{R}^H, 
    \Delta W_{ij} &= \begin{cases}
        A_+ \exp(-|\Delta t_{ij}|/\tau_+), & \Delta t_{ij} > 0 \quad \text{(potentiation)}, \\
        -A_- \exp(-|\Delta t_{ij}|/\tau_-), & \Delta t_{ij} < 0 \quad \text{(depression)},
    \end{cases}
\end{aligned}
\end{equation}
where $\Delta t_{ij} = \bar{t}_j - \bar{t}_i$ is the time difference between pre-synaptic feature $j$ and post-synaptic feature $i$.,
with learning rates $A_+=0.01$, $A_-=0.012$, and time constants $\tau_+=\tau_-=20.0$. The weights are updated and clipped for stability, with STDP strength computed per feature:
\begin{equation}
\begin{aligned}
    \mathbf{W} \leftarrow \mathrm{clip}\!\left(\mathbf{W} + \beta \cdot \Delta\mathbf{W},\; -1, 1\right), 
    \mathbf{s}_{\text{stdp}} = \sum_{j=1}^H |\mathbf{W}_{:,j}| \in \mathbb{R}^H,
\end{aligned}
\end{equation}
where $\beta = 0.0001$ controls the STDP contribution relative to gradient-based learning, 
expanded to node-level representations, serves as an additional anomaly indicator: strong connections reflect well-learned temporal patterns, while weak connections may indicate irregular relationships. In Theorem~\ref{thm:stdp_convergence}, we establishes that the STDP update rule approximates gradient descent on a spike-timing sensitive loss function, providing a theoretical connection between biological plasticity and optimisation-based learning.

\subsection{Multi-Scale Temporal Convolution and Anomaly Fusion}
\label{sec:temporal_fusion}

Anomalies in dynamic networks manifest across multiple time scales, from sudden spikes to gradually developing patterns \cite{wang2025mst}. We propose a Multi-Scale Temporal Convolution (MSTC) component that processes spike sequences at different temporal resolutions, followed by a fusion mechanism that combines complementary anomaly signals.

\textit{Multi-scale temporal processing:}
Given spike tensor $\mathbf{S} \in \{0,1\}^{N \times T \times H}$, We apply parallel 1D convolutions with kernel sizes $\mathcal{S} = \{3,5,7\}$, aggregate across scales, project to the hidden dimension, and compute a temporal anomaly score:
\begin{equation}
\begin{aligned}
    \mathbf{F}_s = \mathrm{Conv1D}_s(\mathbf{S}) \in \mathbb{R}^{N \times H \times T}, \quad s \in \mathcal{S}, 
    \mathbf{F} &= \bigoplus_{s \in \mathcal{S}} \mathrm{mean}_t(\mathbf{F}_s) \in \mathbb{R}^{N \times (|\mathcal{S}| H)},\\ \quad \mathbf{H}_{\mathrm{temp}} = \mathbf{F}\, (\mathbf{W}_{\mathrm{fusion}} \in \mathbb{R}^{N \times H}), 
    a_i^{\mathrm{temp}} &= \sigma\!\left(\mathrm{std}(\{\mathbf{F}_{s,i}\}_{s \in \mathcal{S}})\right),
\end{aligned}
\end{equation}
where $\bigoplus$ denotes concatenation across scales, $\mathbf{W}_{\mathrm{fusion}}$ is a linear projection, and $a_i^{\mathrm{temp}}$ captures inconsistency across temporal resolutions as a signature of anomalies.

\textit{Multi-factor anomaly fusion:}
The final anomaly score combines signals from all components with learnable weights \(\{\lambda_k\}_{k=1}^5\), \(\sum_k \lambda_k = 1\):
\begin{equation}
\begin{aligned}
a_i^{\mathrm{final}}= \lambda_1 a_i^{\mathrm{pred}} + \lambda_2 a_i^{\mathrm{mem}} + \lambda_3 a_i^{\mathrm{iso}} + \lambda_4 a_i^{\mathrm{temp}} + \lambda_5 a_i^{\mathrm{uncert}},
a_i^{\mathrm{pred}}= \sigma\!\left(\mathrm{MLP}\bigl([\mathbf{H}_i \;\|\; \mathbf{H}_{\mathrm{temp},i}]\bigr)\right),
\end{aligned}
\end{equation}
where \(\lambda_1,\dots,\lambda_5\) are learnable combination weights satisfying \(\sum_k \lambda_k = 1\), \(\sigma(\cdot)\) is the sigmoid function, \(\mathrm{MLP}\) denotes a multi-layer perceptron, \(\|\) represents vector concatenation, and \(\mathbf{H}\) is the output of the LIFGAT layers. All component scores are normalised to \([0,1]\) via sigmoid activation, enabling the model to adaptively weight each detection pathway based on data characteristics.

\section{Experiment Results and Evaluation}
\label{sec:experiments}

\subsection{Experimental Settings, Datasets, and Metrics}
\label{sec:expr_datasets_metrics}

Experiments cover static graphs (Yelp, T-Finance, Weibo, BlogCatalog, T-Social, Flickr, Amazon) and dynamic temporal graphs (DBLP, Tmall, Patent \cite{liu2021anomaly}) with 27, 186, and 25 time steps, respectively. Social networks include BlogCatalog \cite{liu2021anomaly}, T-Social \cite{li2023scaling}, and Flickr \cite{liu2021anomaly}; fraud networks include Yelp \cite{dou2025deanomaly}, T-Finance \cite{wang2025graph}, Weibo \cite{wang2025graph}, and Amazon \cite{liu2021anomaly}. Results are averaged over five random seeds, reporting AUPRC and macro F1 (AUROC as secondary). Dataset details, metrics, and hyperparameters are in Sections~\ref{sec:preprocessing}, \ref{sec:metrics}, and \ref{sec:hyperparameters}.

\subsection{Baseline Methods}
\label{subsec:baseline}

To evaluate the performance of ASTDP-GAD, we compare against a total of 24 state-of-the-art (SOTA) methods, grouped into static and dynamic graph anomaly detection (full details in Appendix~\ref{app:baseline}):
\textbf{(i) Static graph anomaly detection (11 methods):} DGMES-GAD \cite{singh2025anomaly}, CoLA \cite{liu2021anomaly}, RAND \cite{bei2023reinforcement}, FIAD \cite{chen2025fiad}, GAGA \cite{wang2023label}, GAD-NR \cite{roy2024gad}, BWGNN \cite{tang2022rethinking}, GDN \cite{gao2023alleviating}, GHRN \cite{gao2023addressing}, GFCN \cite{mesgaran2024graph}, AHFAN \cite{wang2025graph}.
\textbf{(ii) Dynamic graph anomaly detection (13 methods):} HTNE \cite{3220054}, M2DNE \cite{3357943}, DyTriad \cite{zhou2018dynamic}, MPNN-LSTM \cite{panagopoulos2021transfer}, EvolveGCN \cite{pareja2020evolvegcn}, GeneralDyG \cite{yang2025generalizable}, SpikeNet \cite{li2023scaling}, Dy-SIGN \cite{yin2024dynamic}, Delay-DSG \cite{wang2025delay}, ChronoSpike \cite{jahin2026chronospike}, DSGAD \cite{zheng2025dynamic}, SLADE \cite{lee2024slade}, TADDY \cite{liutaddyanomaly}.

\begin{table*}[ht]
\centering
\caption{Performance comparison of dynamic graph anomaly detection methods: Results are reported as Macro-F1 (\%) and AUPRC (\%) as mean $\pm$ std on three dynamic graph benchmarks (80\% training split). Best results are shown in \textbf{bold}, second-best are \underline{underlined}. $^\dagger$ denotes methods that support minibatch processing \cite{li2023scaling}.}
\label{tab:dynamic_baseline_comparison}
\fontsize{6}{8}\selectfont
\resizebox{\linewidth}{!}{%
\begin{tabular}{l|l|c|c|c|c|c|c}
\hline
\multicolumn{2}{c|}{\multirow{2}{*}{Method}} & \multicolumn{2}{c|}{DBLP} & \multicolumn{2}{c|}{Tmall} & \multicolumn{2}{c}{Patent} \\
\cline{3-8}
\multicolumn{2}{c|}{} & \multicolumn{1}{c|}{Macro-F1} & AUPRC & Macro-F1 & AUPRC & \multicolumn{1}{c}{Macro-F1} & AUPRC \\
\hline
\multicolumn{8}{c}{\textit{Dynamic Graph Neural Networks (Non-Spiking)}} \\
\hline
\multicolumn{2}{l|}{HTNE \cite{3220054}} & 68.36 $\pm$ 0.31 & 66.18 $\pm$ 0.38 & 54.93 $\pm$ 0.25 & 52.67 $\pm$ 0.32 & 74.23 $\pm$ 0.42 & 72.45 $\pm$ 0.48 \\
\multicolumn{2}{l|}{M2DNE \cite{3357943}} & 69.75 $\pm$ 0.27 & 67.52 $\pm$ 0.35 & 58.47 $\pm$ 0.29 & 55.89 $\pm$ 0.38 & 76.81 $\pm$ 0.38 & 74.98 $\pm$ 0.44 \\
\multicolumn{2}{l|}{DyTriad \cite{zhou2018dynamic}} & 66.42 $\pm$ 0.38 & 64.18 $\pm$ 0.44 & 51.16 $\pm$ 0.48 & 48.73 $\pm$ 0.54 & 71.54 $\pm$ 0.51 & 69.23 $\pm$ 0.57 \\
\multicolumn{2}{l|}{MPNN-LSTM \cite{panagopoulos2021transfer}$^\dagger$} & 65.05 $\pm$ 0.52 & 62.89 $\pm$ 0.58 & 50.27 $\pm$ 0.61 & 48.06 $\pm$ 0.67 & 72.89 $\pm$ 0.56 & 70.67 $\pm$ 0.62 \\
\multicolumn{2}{l|}{EvolveGCN \cite{pareja2020evolvegcn}$^\dagger$} & 71.20 $\pm$ 0.37 & 69.14 $\pm$ 0.43 & 55.78 $\pm$ 0.51 & 53.29 $\pm$ 0.57 & 78.45 $\pm$ 0.44 & 76.32 $\pm$ 0.50 \\
\multicolumn{2}{l|}{GeneralDyG \cite{yang2025generalizable}$^\dagger$} & 72.15 $\pm$ 0.37 & 70.28 $\pm$ 0.42 & 58.01 $\pm$ 0.58 & 55.78 $\pm$ 0.63 & 81.57 $\pm$ 0.31 & 80.34 $\pm$ 0.37 \\
\hline
\multicolumn{8}{c}{\textit{Spiking Graph Networks}} \\
\hline
\multicolumn{2}{l|}{SpikeNet \cite{li2023scaling}$^\dagger$} & 74.65 $\pm$ 0.37 & 72.89 $\pm$ 0.42 & 62.40 $\pm$ 0.37 & 60.31 $\pm$ 0.42 & 83.90 $\pm$ 0.29 & 82.79 $\pm$ 0.34 \\
\multicolumn{2}{l|}{Dy-SIGN \cite{yin2024dynamic}$^\dagger$} & 74.67 $\pm$ 0.37 & 72.92 $\pm$ 0.42 & 61.89 $\pm$ 0.37 & 59.84 $\pm$ 0.42 & 83.91 $\pm$ 0.29 & 82.81 $\pm$ 0.34 \\
\multicolumn{2}{l|}{Delay-DSG \cite{wang2025delay}$^\dagger$} & 76.54 $\pm$ 0.35 & 74.86 $\pm$ 0.40 & 64.02 $\pm$ 0.35 & 62.01 $\pm$ 0.40 & 84.20 $\pm$ 0.28 & 83.12 $\pm$ 0.33 \\
\multicolumn{2}{l|}{ChronoSpike \cite{jahin2026chronospike}$^\dagger$} & \underline{79.13 $\pm$ 0.33} & \underline{77.58 $\pm$ 0.38} & \underline{64.74 $\pm$ 0.34} & \underline{62.73 $\pm$ 0.39} & \underline{86.09 $\pm$ 0.26} & \underline{85.18 $\pm$ 0.31} \\
\hline
\multicolumn{8}{c}{\textit{Dynamic Graph Anomaly Detection}} \\
\hline
\multicolumn{2}{l|}{DSGAD \cite{zheng2025dynamic}$^\dagger$} & 76.89 $\pm$ 0.42 & 75.21 $\pm$ 0.47 & 63.12 $\pm$ 0.48 & 61.08 $\pm$ 0.53 & 84.95 $\pm$ 0.36 & 82.82 $\pm$ 0.41 \\
\multicolumn{2}{l|}{SLADE \cite{lee2024slade}$^\dagger$} & 76.34 $\pm$ 0.44 & 74.65 $\pm$ 0.49 & 62.76 $\pm$ 0.50 & 60.71 $\pm$ 0.55 & 84.62 $\pm$ 0.38 & 83.48 $\pm$ 0.43 \\
\multicolumn{2}{l|}{TADDY \cite{liu2021anomaly}$^\dagger$} & 75.12 $\pm$ 0.46 & 73.41 $\pm$ 0.51 & 61.93 $\pm$ 0.52 & 59.86 $\pm$ 0.57 & 84.01 $\pm$ 0.39 & 82.85 $\pm$ 0.44 \\
\hline
\multicolumn{2}{l|}{\textbf{ASTDP-GAD (Ours)}$^\dagger$} & \textbf{85.34 $\pm$ 0.28} & \textbf{83.91 $\pm$ 0.33} & \textbf{76.89 $\pm$ 0.35} & \textbf{75.12 $\pm$ 0.40} & \textbf{92.58 $\pm$ 0.24} & \textbf{90.73 $\pm$ 0.29} \\
\hline
\end{tabular}
}
\end{table*}

\subsection{Main Experimental Results}
\label{sec:main_results}

To rigorously evaluate ASTDP-GAD, we compare it against 13 SOTA baselines on DBLP, Tmall, and Patent under an 80\% training split, reporting Macro-F1 (\%) and AUPRC (\%) in Table~\ref{tab:dynamic_baseline_comparison}. All baseline results are obtained using official implementations or authors' recommended hyperparameters to ensure fair comparison.

\textit{Dynamic Graph Neural Networks (Non-Spiking)}: Among non-spiking dynamic GNNs, GeneralDyG achieves the strongest performance with 72.15\% Macro-F1 on DBLP, 58.01\% on Tmall, and 81.57\% on Patent. These methods underperform due to lack of spike-timing precision and event-driven sparsity, with memory scaling as $O(|V|d)$.

\textit{Spiking Graph Networks}: Spiking methods consistently outperform non-spiking approaches. ChronoSpike achieves SOTA among baselines: 79.13\% Macro-F1 on DBLP, 64.74\% on Tmall, and 86.09\% on Patent. However, its lack of explicit memory mechanisms limits performance on dense graphs like Tmall.

\textit{Dynamic Graph Anomaly Detection}: Specialised anomaly detection methods outperform non-spiking dynamic GNNs by 4–6\%. DSGAD leads this family with 76.89\% Macro-F1 on DBLP, 63.12\% on Tmall, and 84.95\% on Patent.

ASTDP-GAD exceeds all baselines with consistent improvements of 6–12\%: 85.34\% Macro-F1 on DBLP (+6.21\% over ChronoSpike), 76.89\% on Tmall (+12.15\%), and 92.58\% on Patent (+6.49\%). These gains stem from three key contributions: EDHMM prototype memory (Theorem~\ref{thm:edhmm_convergence}), SRCGP anomaly selection (Theorem~\ref{thm:srcgp_selection}), and multi-factor fusion (Theorem~\ref{thm:fusion_calibration}), which collectively address the gaps of isolated neuromorphic principles, static message passing, and temporally blind anomaly scoring.

\begin{table*}[ht]
\centering
\caption{Ablation study demonstrating the necessity of neuromorphic components for graph anomaly detection. Results show Macro-F1 (\%) with 95\% confidence intervals over five runs on DBLP, Tmall, and Patent (80\% training). \cmark~indicates the component is present; \xmark~indicates it is removed or replaced.}
\label{tab:ablation_novelty}
\fontsize{6.5}{8}\selectfont
\resizebox{\linewidth}{!}{%
\begin{tabular}{l|c|c|c|c|c|c}
\hline
\textbf{Variant} & \textbf{TSGE} & \textbf{STDP} & \textbf{Graph}
    & \textbf{DBLP} & \textbf{Tmall} & \textbf{Patent} \\
\hline
\textbf{ASTDP-GAD (Full)}
    & \cmark & \cmark & \cmark
    & \textbf{85.34 $\pm$ 0.52}
    & \textbf{76.89 $\pm$ 0.62}
    & \textbf{92.58 $\pm$ 0.45} \\
\hline
\multicolumn{7}{c}{\textit{Removing Spike-Based Encoding}} \\
\hline
w/o TSGE (linear encoding)
    & \xmark & \cmark & \cmark
    & 68.91 $\pm$ 1.02 & 61.23 $\pm$ 1.18 & 77.89 $\pm$ 0.82 \\
w/o Adaptive LIF (fixed threshold)
    & \cmark$^\dag$ & \cmark & \cmark
    & 79.12 $\pm$ 0.85 & 69.45 $\pm$ 0.92 & 85.01 $\pm$ 0.67 \\
\hline
\multicolumn{7}{c}{\textit{Removing STDP Plasticity}} \\
\hline
w/o STDP (backprop only)
    & \cmark & \xmark & \cmark
    & 82.34 $\pm$ 0.68 & 73.89 $\pm$ 0.76 & 88.34 $\pm$ 0.53 \\
w/o STDP + w/o SRCGP
    & \cmark & \xmark & \cmark$^\ddag$
    & 78.12 $\pm$ 0.94 & 68.34 $\pm$ 1.02 & 84.89 $\pm$ 0.71 \\
\hline
\multicolumn{7}{c}{\textit{Removing Graph Structure Awareness}} \\
\hline
w/o LIFGAT (MLP per node)
    & \cmark & \cmark & \xmark
    & 65.23 $\pm$ 1.18 & 58.67 $\pm$ 1.32 & 74.56 $\pm$ 0.98 \\
w/o EDHMM (no prototype memory)
    & \cmark & \cmark & \cmark$^\ddag$
    & 80.12 $\pm$ 0.76 & 71.89 $\pm$ 0.84 & 86.78 $\pm$ 0.59 \\
\hline
\multicolumn{7}{c}{\textit{Comparison with Non-Neuromorphic Alternatives}} \\
\hline
GAT (no spikes, no STDP)
    & \xmark & \xmark & \cmark
    & 72.15 $\pm$ 0.78 & 64.34 $\pm$ 0.85 & 78.89 $\pm$ 0.67 \\
GCN (no attention, no spikes)
    & \xmark & \xmark & \cmark
    & 69.83 $\pm$ 0.91 & 61.67 $\pm$ 0.98 & 76.23 $\pm$ 0.74 \\
LSTM-GNN (temporal, no spikes)
    & \xmark & \xmark & \cmark
    & 74.56 $\pm$ 0.85 & 66.12 $\pm$ 0.91 & 80.45 $\pm$ 0.69 \\
\hline
\multicolumn{7}{l}{$^\dag$ Adaptive LIF replaced with fixed threshold LIF.} \\
\multicolumn{7}{l}{$^\ddag$ SRCGP or EDHMM removed while other components remain.} \\
\end{tabular}
}
\end{table*}

\subsection{Ablation Study}
\label{sec:ablation_study}

To evaluate each component's contribution across data regimes, we conduct an ablation study on DBLP, Tmall, and Patent with training ratios (30\%, 60\%, 80\%). Table~\ref{table:comprehensive_ablation} reports full results, while Table~\ref{tab:ablation_novelty} (Appendix Section~\ref{sec:novelty_validation}) isolates neuromorphic innovations at 80\% training. The full model improves from 82.34\% at 30\% to 85.34\% at 80\% on DBLP, demonstrating effective leverage of additional supervision. TSGE removal causes the largest degradation (14.5–16.8\%), with Patent (14.7\%) less affected than DBLP (16.4\%) due to denser connectivity partially compensating for lost spike precision. Adaptive LIF removal shows the second-largest impact (6.2–7.4\%), with Patent's longer temporal horizon providing redundancy (7.6\% drop). EDHMM removal causes degradation of 5.2–6.5\%, with higher variability at lower training ratios due to prototype instability under limited supervision.

Replacement variants exhibit data-dependent behaviour. Replacing LIFGAT with GCN causes 4.9–6.3\% degradation, with the largest gap on Tmall (6.0\%), where dense interactions demand attention. Replacing STDP with backpropagation only reduces performance by 2.1–3.5\%, with the smallest gap on DBLP (3.0\%) and largest on Patent (4.2\%), suggesting that temporal plasticity is more critical for datasets with complex sequential dependencies. Notably, the confidence intervals reveal that at a 30\% training ratio, performance differences between variants often overlap (for instance, w/o SRCGP: 76.34\% vs. w/o STDP: 77.12\% on DBLP), indicating that component importance becomes statistically distinguishable only with sufficient data. The component importance ranking, TSGE $>$ Adaptive LIF $>$ EDHMM $>$ LIFGAT $>$ STDP $>$ SRCGP $>$ Fusion $>$ MSTC, shows that spike encoding and adaptive dynamics are foundational, while pooling and attention provide complementary benefits that are dataset-dependent. 

For a detailed breakdown of why each neuromorphic component is essential and how they compare against non-neuromorphic alternatives, we refer to Appendix Section~\ref{sec:novelty_validation}. Additional ablation results across training ratios (25\%, 35\%, 45\%, 65\%, 75\%) on Tmall, Patent, T-Social, and T-Finance are provided in Appendix Section~\ref{subsec:extended_ablation}, while two-component and three-component ablation synergies are analysed in Appendix Section~\ref{subsec:additional_ablation}.
\section{Conclusion}
\label{sec:conclusion}
In this paper, we present \textbf{ASTDP-GAD}, a neuromorphic framework for energy-efficient graph anomaly detection that unifies adaptive spiking graph neural networks with STDP learning. By integrating six core components: temporal spike graph encoding, LIF graph attention, event-driven hypergraph memory, spike rate contrast pooling, adaptive STDP learning, and multi-scale temporal convolution, ASTDP-GAD captures precise spike-timing information while preserving graph relational structure, enabling energy-efficient anomaly detection in dynamic networks. Extensive experiments on nine datasets spanning both dynamic and static benchmarks demonstrate that ASTDP-GAD consistently achieves state-of-the-art performance, with Macro-F1 improvements of 5.3-12.1\% over the best baselines while maintaining substantial energy efficiency (mean spike sparsity $\lambda = 0.24$ and 198 MB GPU memory). Ablation studies confirm the neuromorphic core (TSGE + LIFGAT + EDHMM) is irreducible, with TSGE removal causing the largest degradation (14.7-16.4\% F1 drop). Theoretical analysis validates spike encoding information preservation, LIFGAT universal approximation, and STDP convergence. These results demonstrate that explicitly modelling spike-timing information provides an effective foundation for adaptive, energy-efficient, and theoretically grounded anomaly detection across diverse graph domains. 

\bibliography{neurips_2025}
\bibliographystyle{unsrt}  

\clearpage
\tableofcontents
\appendix
\section{Appendix}
\label{sec:appendix}
\subsection{Related Works}
\label{sec:related_works}

\textbf{Graph Neural Networks for Anomaly Detection:} Anomaly detection in graph-structured data has attracted significant attention due to its applications in cybersecurity, fraud detection, and social network analysis \cite{ho2025graph, yang2025generalizable, fofanah2025gant}. Early approaches focused on static graph properties using matrix factorisation and community detection \cite{li2025deep}. Recent advances leverage graph neural networks (GNNs) to capture both structural and attribute information, with models such as GCN \cite{zhang2025ge}, GraphSAGE \cite{fofanah2024addressing}, and GAT \cite{latif2025gat}, \cite{fofanah2025eatsa} serving as backbones for anomaly detection frameworks \cite{yang2025generalizable, wei2025graph}. Domain-specific methods address fraud detection \cite{li2025context}, social bot identification \cite{liu2025evolution}, and network intrusion \cite{zha2025nids}. Comprehensive surveys on dynamic graph neural networks \cite{feng2025comprehensive, yang2024dynamic} highlight that most existing approaches operate on static representations or simple temporal extensions, failing to capture precise spike-timing dynamics. This limits their effectiveness in streaming, time-evolving networks where anomalies manifest as subtle changes in temporal patterns rather than static structural deviations \cite{pal2026droplet, wang2025mage}.

\textbf{Neuromorphic Computing and Spiking Neural Networks:} Neuromorphic computing has emerged as a paradigm for energy-efficient, event-driven computation inspired by biological neural systems \cite{ponzi2025graph, awan2025trustaware}. Spiking neural networks (SNNs) process information through discrete spike events in time, offering advantages in temporal precision and energy efficiency compared to traditional artificial neural networks \cite{shooshtari2025spike, galanis2025repetitive, garg2026unsupervised}. Key developments include efficient training algorithms \cite{rahman2025modulated, lu2024deep}, hardware implementations \cite{kim2025emerging, khan2025spiking}, and applications in pattern recognition \cite{bahrami2024digital, wan2025fault}. Spike-Timing-Dependent Plasticity (STDP), a biologically observed learning rule, enables unsupervised discovery of temporal correlations by strengthening or weakening synapses based on relative spike timing \cite{lu2024deep, kurbako2025spike}. Recent work has explored STDP for feature learning \cite{pal2026droplet, goupy2024neuronal} and temporal pattern recognition \cite{10847298}.

\textbf{Spiking Graph Neural Networks:} The intersection of SNNs and graph representation learning has recently gained traction. Spiking Graph Neural Networks (SGNNs) aim to combine the energy efficiency of spiking computation with the relational reasoning of GNNs. Sun et al. \cite{sun2024spiking} proposed Spiking Graph Neural Networks on Riemannian manifolds for geometric graph learning. ChronoSpike \cite{jahin2026chronospike} introduced an adaptive spiking GNN for dynamic graphs with learnable temporal dynamics. Delay-DSGN \cite{wang2025delay} incorporates delay mechanisms into spiking GNNs for evolving graphs. Dynamic Spiking Graph Neural Networks \cite{yin2024dynamic} and Dynamic Reactive Spiking Graph Neural Networks \cite{zhao2024dynamic} explore adaptive spiking mechanisms for temporal graph processing. Despite these advances, existing SGNNs focus primarily on general graph representation learning rather than anomaly detection. They lack explicit mechanisms for anomaly scoring, prototype memory for normal patterns, and STDP-based plasticity for discovering causal temporal relationships indicative of anomalies. Furthermore, no existing SGNN framework integrates multi-factor anomaly fusion with spike rate contrast pooling for irregularity detection.

\textbf{Graph-based Methods for Fraud and Illicit Detection:} Fraud detection is a critical imbalanced learning problem where fraudulent nodes are a tiny minority. Recent models identify reliable neighbours against collaborative fraudsters \cite{ren2023dynamic, gao2025ahgt, fu2024nowhere}, employ cost-sensitive frameworks \cite{hu2023cost}, learn multi-relational representations for financial fraud \cite{wang2024multi}, and disentangle homophily and heterophily patterns \cite{fu2024nowhere, yu2026improving}. These methods complement our framework by addressing local neighbourhood reliability but remain orthogonal to our neuromorphic contributions.

Despite advances in (1) GNN-based anomaly detection, (2) neuromorphic computing and STDP learning, and (3) emerging spiking graph neural networks, \textit{no existing framework unifies these three paradigms specifically for dynamic graph anomaly detection}. GNNs lack temporal precision and energy efficiency. SNNs and SGNNs cannot perform anomaly detection with prototype memory and irregularity-based scoring. STDP (Spike-Timing-Dependent Plasticity) has not been integrated with graph attention and multi-factor fusion for anomaly detection. This gap, in the absence of a principled framework for neuromorphic graph anomaly detection that jointly leverages spike encoding, STDP learning, prototype memory, and spike-aware anomaly scoring, is the drive that directly motivates us to design ASTDP-GAD.

\begin{table}[h]
\centering
\caption{Categorisation of baseline anomaly detection methods. ST = Spike-Timing awareness; GD = Graph Dependency capture; TP = Temporal Plasticity (STDP); \cmark~= supported; \xmark~= not supported.}
\label{tab:baseline_categorisation}
\fontsize{7}{10}\selectfont
\setlength{\tabcolsep}{2pt}
\begin{tabular}{@{}l c c c | l c c c@{}}
\toprule
\textbf{Model} & \textbf{ST} & \textbf{GD} & \textbf{TP}
  & \textbf{Model} & \textbf{ST} & \textbf{GD} & \textbf{TP} \\
\midrule
\multicolumn{4}{c|}{\textit{Static Graph Anomaly Detection}}
  & \multicolumn{4}{c}{\textit{Dynamic Graph Anomaly Detection}} \\
\midrule
DGMES-GAD \cite{singh2025anomaly} & \xmark & \cmark & \xmark
  & HTNE \cite{3220054} & \xmark & \cmark & \xmark \\
CoLA \cite{liu2021anomaly} & \xmark & \cmark & \xmark
  & M2DNE \cite{3357943} & \xmark & \cmark & \xmark \\
RAND \cite{bei2023reinforcement} & \xmark & \cmark & \xmark
  & DyTriad \cite{zhou2018dynamic} & \xmark & \cmark & \xmark \\
FIAD \cite{chen2025fiad} & \xmark & \cmark & \xmark
  & MPNN-LSTM \cite{panagopoulos2021transfer} & \xmark & \cmark & \xmark \\
GAGA \cite{wang2023label} & \xmark & \cmark & \xmark
  & EvolveGCN \cite{pareja2020evolvegcn} & \xmark & \cmark & \xmark \\
GAD-NR \cite{roy2024gad} & \xmark & \cmark & \xmark
  & GeneralDyG \cite{yang2025generalizable} & \xmark & \cmark & \xmark \\
BWGNN \cite{tang2022rethinking} & \xmark & \cmark & \xmark
  & SpikeNet \cite{li2023scaling} & \cmark & \xmark & \xmark \\
GDN \cite{gao2023alleviating} & \xmark & \cmark & \xmark
  & Dy-SIGN \cite{yin2024dynamic} & \cmark & \cmark & \xmark \\
GHRN \cite{gao2023addressing} & \xmark & \cmark & \xmark
  & Delay-DSG \cite{wang2025delay} & \cmark & \cmark & \xmark \\
GFCN \cite{mesgaran2024graph} & \xmark & \cmark & \xmark
  & ChronoSpike \cite{jahin2026chronospike} & \cmark & \cmark & \xmark \\
AHFAN \cite{wang2025graph} & \xmark & \cmark & \xmark
  & DSGAD \cite{zheng2025dynamic} & \xmark & \cmark & \xmark \\
\midrule
\multicolumn{4}{c|}{\textit{SGNN Methods}}
  & SLADE \cite{lee2024slade} & \xmark & \cmark & \xmark \\
\midrule
SpikeGCN \cite{khan2025spiking} & \cmark & \cmark & \xmark
  & TADDY \cite{liu2021anomaly} & \xmark & \cmark & \xmark \\
Riemannian-SGNN \cite{sun2024spiking} & \cmark & \cmark & \xmark
  & & & & \\
STDP-SNN \cite{lu2024deep} & \cmark & \xmark & \cmark
  & & & & \\
\midrule
\multicolumn{8}{c}{\textbf{ASTDP-GAD}: Spike-Timing + Graph Dep. + STDP + Anomaly Fusion
  \quad \textbf{ST:} \cmark \quad \textbf{GD:} \cmark \quad \textbf{TP:} \cmark} \\
\bottomrule
\end{tabular}
\end{table}

As shown in Table~\ref{tab:baseline_categorisation}, existing static methods capture graph dependencies but lack spike-timing awareness and temporal plasticity. Dynamic methods incorporate temporal information but do so through recurrent architectures rather than biologically-plastic spike-timing learning. Recent SGNN methods like Dy-SIGN \cite{yin2024dynamic}, Delay-DSG \cite{wang2025delay}, and ChronoSpike \cite{jahin2026chronospike} achieve spike-timing awareness and graph dependency capture but lack STDP-based plasticity and explicit anomaly detection mechanisms. STDP-SNN \cite{lu2024deep} incorporates plasticity but cannot handle graph structure. \textit{ASTDP-GAD is the first framework to simultaneously support all three capabilities (spike-timing, graph dependencies, STDP plasticity) while adding prototype memory, irregularity-based pooling, and multi-factor anomaly fusion specifically for dynamic graph anomaly detection.}

\subsection{Motivation}
\label{sec:motivation}

While GNN-based anomaly detection methods capture structural dependencies, they lack the temporal precision and energy efficiency of neuromorphic approaches. Conversely, SNNs excel at temporal processing but lack mechanisms for handling graph-structured data and relational information. The intersection of these fields, namely neuromorphic graph anomaly detection, remains nascent, with few works attempting to bridge this gap \cite{khan2025spiking, rahman2025modulated}. Current hybrid methods usually turn graph data into grid-like shapes or process spatial and temporal data in separate pipelines. They don't use spike-based computation for relational reasoning. Moreover, no existing framework integrates STDP learning with graph attention mechanisms to discover causal temporal relationships using spike timing alone. This gap motivates ASTDP-GAD, which unifies spiking neural computation, STDP learning, and graph-based anomaly detection.

Neuromorphic anomaly detection in dynamic networks under energy constraints faces three critical gaps. \textbf{First, isolated neuromorphic principles.} Event-driven computation and spike-timing-dependent plasticity are typically studied in isolation on simple tasks rather than integrated into graph-structured anomaly detection frameworks \cite{Li2024ANL, Zhao2024DynamicRS, kurbako2025spike, shooshtari2025spike}. \textbf{Second, static message passing.} Existing graph anomaly detection methods assume relational information can be captured through static message passing, ignoring the rich temporal information encoded in precise spike timing and inter-spike intervals \cite{ho2025graph, yang2025generalizable, li2025deep, latif2025gat, wei2025graph}. \textbf{Third, temporally blind anomaly scoring.} Anomaly scoring mechanisms largely rely on reconstruction errors or static feature deviations, failing to leverage irregularity in temporal spike patterns such as burst activity or coefficient of variation \cite{dou2025deanomaly, wang2025mst, okano2026analytics, alzahrani2026real}. These gaps, all stemming from the core challenge of energy-constrained graph anomaly detection, motivate ASTDP-GAD, which integrates spiking neural mechanisms, STDP learning, and adaptive graph processing into a unified neuromorphic solution.
\subsection{Theoretical Analysis}
\label{sec:theoretical}
In this section, we provide rigorous theoretical foundations for the ASTDP-GAD framework. We establish formal guarantees for each component, including information preservation in spike encoding, approximation capacity of spiking attention, convergence of hypergraph memory, anomaly selection properties of contrast pooling, stability of STDP learning, and calibration of multi-factor fusion. These results collectively explain the empirical effectiveness of ASTDP-GAD and provide a solid theoretical basis for its application in neuromorphic anomaly detection.

\subsubsection{Training Objective and Optimisation}
\label{sec:training}

The ASTDP-GAD framework is trained end-to-end using a multi-component loss function that balances supervised anomaly detection with unsupervised regularisation. The total loss combines binary cross-entropy terms from multiple detection heads with a regularisation term:
\begin{equation}
\label{eq:total_loss}
    \mathcal{L} = \mathcal{L}_{\text{bce}} + \alpha\,\mathcal{L}_{\text{mem}} + \beta\,\mathcal{L}_{\text{iso}} + \gamma\,\mathcal{L}_{\text{reg}},
\end{equation}
where the individual components are:
\begin{align}
    \mathcal{L}_{\text{bce}} &= -\frac{1}{N}\sum_{i=1}^{N} \left[y_i \log\hat{y}_i + (1-y_i)\log(1-\hat{y}_i)\right], &\\
     \mathcal{L}_{\text{mem}} &= -\frac{1}{N}\sum_{i=1}^{N} \left[y_i \log a_i^{\text{mem}} + (1-y_i)\log(1-a_i^{\text{mem}})\right], &\\
    \mathcal{L}_{\text{iso}}  &= -\frac{1}{N}\sum_{i=1}^{N} \left[y_i \log a_i^{\text{iso}}  + (1-y_i)\log(1-a_i^{\text{iso}})\right], &
    \mathcal{L}_{\text{reg}}  &= \|\mathbf{W}_{\text{stdp}}\|_2^2 + \sum_{k=1}^{P}\|\mathbf{p}_k\|_2^2. \label{eq:reg}
\end{align}
Here $y_i \in \{0,1\}$ is the ground-truth anomaly label for node $i$, $\hat{y}_i = a_i^{\text{final}}$ is the final fused anomaly score, $a_i^{\text{mem}}$ is the memory-based score from EDHMM, and $a_i^{\text{iso}}$ is the isolation score from SRCGP. The hyperparameters $\alpha = 0.6$, $\beta = 0.2$, and $\gamma = 0.2$ control the contribution of each auxiliary loss relative to the primary cross-entropy, with the BCE loss weighted at 1.0 to ensure it remains the dominant supervisory signal. The regularisation term $\mathcal{L}_{\text{reg}}$ encourages stable STDP weights and prototype representations, preventing overfitting and promoting generalisation.

The training procedure follows an episodic paradigm where each batch contains graph snapshots with corresponding anomaly labels. Algorithm~\ref{alg:objective} details the loss computation for a single batch, while Algorithm~\ref{alg:training} presents the complete optimisation loop.

\begin{algorithm}[t]
\caption{ASTDP-GAD Batch Loss Computation}
\label{alg:objective}
\begin{algorithmic}[1]
\Require Batch data $(\mathbf{X}, \mathcal{E}, \mathbf{y})$, model parameters $\Theta$
\Ensure Total loss $\mathcal{L}$ and component losses

    \Statex \hspace{2.3em} \textit{// Spike encoding}
    \State $\mathbf{S}, \mathbf{T}, \mathbf{C} \;\gets\; \text{TSGE}(\mathbf{X}, \mathcal{E}, T_{\text{steps}})$

    \Statex \hspace{2.3em} \textit{// Multi-component forward pass}
    \State $\mathbf{H}_{\text{att}}                            \;\gets\; \text{LIFGAT}(\mathbf{H}, \mathcal{E}, \mathbf{T}, \mathbf{C})$
    \State $\mathbf{a}_{\text{mem}}, \mathbf{w}, \mathbf{u}    \;\gets\; \text{EDHMM}(\mathbf{S}, \mathbf{T}, \mathbf{C})$
    \State $\mathbf{H}_{\text{pool}}, \mathbf{a}_{\text{iso}}, \mathbf{b} \;\gets\; \text{SRCGP}(\mathbf{S}, \mathcal{E}, \mathbf{C})$
    \State $\mathbf{H}_{\text{stdp}}, \mathbf{s}_{\text{stdp}} \;\gets\; \text{STDP}(\mathbf{H}_{\text{att}}, \mathbf{T})$
    \State $\mathbf{H}_{\text{temp}}, \mathbf{a}_{\text{temp}} \;\gets\; \text{MSTC}(\mathbf{S})$

    \Statex \hspace{2.3em} \textit{// Feature fusion and anomaly scoring}
    \State $\mathbf{H}_{\text{fused}}  \;\gets\; \text{Concat}(\mathbf{H}_{\text{att}},\, \mathbf{H}_{\text{temp}})$
    \State $\mathbf{a}_{\text{pred}}   \;\gets\; \sigma\!\left(\text{MLP}(\mathbf{H}_{\text{fused}})\right)$
    \State $\mathbf{a}_{\text{final}}  \;\gets\; \sum_{k \in \{\text{pred},\text{mem},\text{iso},\text{temp},\text{uncert}\}} \lambda_k\,\mathbf{a}_k$

    \Statex \hspace{2.3em} \textit{// Loss computation}
    \State $\mathcal{L}_{\text{bce}}  \;\gets\; -\frac{1}{N}\sum_i\bigl[y_i\log\hat{y}_i + (1-y_i)\log(1-\hat{y}_i)\bigr]$
    \State $\mathcal{L}_{\text{mem}}  \;\gets\; -\frac{1}{N}\sum_i\bigl[y_i\log a_i^{\text{mem}} + (1-y_i)\log(1-a_i^{\text{mem}})\bigr]$
    \State $\mathcal{L}_{\text{iso}}  \;\gets\; -\frac{1}{N}\sum_i\bigl[y_i\log a_i^{\text{iso}}  + (1-y_i)\log(1-a_i^{\text{iso}})\bigr]$
    \State $\mathcal{L}_{\text{reg}}  \;\gets\; \|\mathbf{W}_{\text{stdp}}\|_2^2 + \sum_k\|\mathbf{p}_k\|_2^2$
    \State $\mathcal{L}_{\text{total}} \;\gets\; \mathcal{L}_{\text{bce}} + \alpha\,\mathcal{L}_{\text{mem}} + \beta\,\mathcal{L}_{\text{iso}} + \gamma\,\mathcal{L}_{\text{reg}}$

\Return $\mathcal{L}_{\text{total}},\, \mathcal{L}_{\text{bce}},\, \mathcal{L}_{\text{mem}},\, \mathcal{L}_{\text{iso}},\, \mathcal{L}_{\text{reg}},\, \mathbf{a}_{\text{final}},\, \mathbf{a}_{\text{mem}},\, \mathbf{a}_{\text{iso}},\, \mathbf{w}$
\end{algorithmic}
\end{algorithm}

\begin{algorithm}[t]
\caption{ASTDP-GAD Complete Training Procedure}
\label{alg:training}
\begin{algorithmic}[1]
\Require Training data $\mathcal{D} = \{(\mathcal{G}^{(t)}, \mathbf{y}^{(t)})\}_{t=1}^{T_{\text{train}}}$, model parameters $\Theta$
\Ensure Optimised parameters $\Theta^*$

\State Initialise $\Theta$ \textit{(TSGE, LIFGAT, EDHMM, SRCGP, STDP, MSTC, fusion head)}

\For{epoch $= 1$ \textbf{to} \textsc{max\_epochs}}
    \For{each batch $(\mathbf{X}, \mathcal{E}, \mathbf{y}) \sim \mathcal{D}$}

        \Statex \hspace{2.3em} \textit{// Forward pass and loss (Algorithm~\ref{alg:objective})}
        \State $\mathcal{L}_{\text{total}}, \mathcal{L}_{\text{bce}}, \mathcal{L}_{\text{mem}}, \mathcal{L}_{\text{iso}}, \mathcal{L}_{\text{reg}}, \mathbf{a}_{\text{final}}, \mathbf{a}_{\text{mem}}, \mathbf{a}_{\text{iso}}, \mathbf{w} \;\gets\;$ \textsc{BatchLoss}$(\mathbf{X}, \mathcal{E}, \mathbf{y}, \Theta)$

        \Statex \hspace{2.3em} \textit{// Gradient-based optimisation}
        \State $\nabla_\Theta \;\gets\; \text{compute\_gradients}(\mathcal{L}_{\text{total}})$
        \State $\text{clip\_grad\_norm}(\nabla_\Theta,\; \text{max\_norm}=G_{\text{max}})$
        \State $\Theta \;\gets\; \text{Adam}\!\left(\Theta,\, \nabla_\Theta,\; \text{lr}=\eta_{\text{lr}},\; \text{wd}=\lambda_{\text{wd}}\right)$

        \Statex \hspace{2.3em} \textit{// Biologically-plausible STDP update}
        \State $\Delta\mathbf{W}_{\text{stdp}} \;\gets\; \text{compute\_stdp}(\mathbf{T})$ \Comment{unsupervised STDP rule}
        \State $\mathbf{W}_{\text{stdp}} \;\gets\; \mathbf{W}_{\text{stdp}} + \beta_{\text{stdp}} \cdot \Delta\mathbf{W}_{\text{stdp}}$
        \State $\mathbf{W}_{\text{stdp}} \;\gets\; \text{clip}(\mathbf{W}_{\text{stdp}},\; w_{\text{min}},\; w_{\text{max}})$

        \Statex \hspace{2.3em} \textit{// EDHMM prototype update}
        \State $\{\mathbf{p}_k\} \;\gets\; \text{update\_prototypes}(\mathbf{H}_{\text{att}},\,\mathbf{w},\; \alpha_{\text{mem}})$

    \EndFor
    \State Evaluate on validation set; apply early stopping if needed.
\EndFor

\Return $\Theta$
\end{algorithmic}
\end{algorithm}

This hybrid optimisation scheme combines gradient-based learning with biologically-plausible plasticity, enabling the model to capture both spatial and temporal anomaly patterns. The STDP updates are applied asynchronously after each forward pass, allowing weights to adapt to spike-timing relationships while maintaining differentiability for gradient-based optimisation. The EDHMM prototypes are updated using a centroid-based rule that converges to stable representations of normal behaviour patterns.

\subsubsection{Spike Encoding Information Preservation}
We first analyse the TSGE and prove that the spike representation preserves input information up to a resolution that improves with the number of time steps and hidden dimension.

\begin{lemma}[\textbf{Lipschitz Continuity of LIF Dynamics}]
\label{lem:spike_encoding}
Consider a single LIF neuron with membrane potential dynamics:
\[
    u[t] = \beta u[t-1] + \sum_{s=1}^t \beta^{t-s} \alpha^{s-1} i_{\text{in}}[s],
\]
where \(\alpha = e^{-1/\tau_{\text{syn}}}\), \(\beta = e^{-1/\tau_{\text{mem}}}\), and \(i_{\text{in}}[s] = \frac{1}{T} \sum_{j} w_j x_j\) is the input current. The first-spike time \(t_f = \min\{t : u[t] \geq \theta\}\) is a Lipschitz continuous function of \(i_{\text{in}}\) on any interval where it is well-defined. Specifically, there exists a constant \(L > 0\) such that for any two constant inputs \(i, i' \in [i_{\min}, i_{\max}]\) with corresponding first-spike times \(t_f, t_f'\):
\[
    |t_f - t_f'| \leq L\, |i - i'|.
\]
\end{lemma}
\begin{proof}[\textbf{Proof Sketch}]
For a constant input current \(i\) applied from \(t=1\) onward, the membrane potential at time \(t\) is:
\[
u(t) = i \cdot \sum_{s=1}^t \beta^{t-s} \alpha^{s-1} = i \cdot \frac{\beta^{t} - \alpha^{t}}{\beta - \alpha} \quad (\text{if } \beta \neq \alpha).
\]
The function \(g(t) = \sum_{s=1}^t \beta^{t-s} \alpha^{s-1}\) is strictly increasing in \(t\) and satisfies \(g(1) = 1\), \(g(\infty) = \frac{1}{1-\beta}\) (if \(\beta < 1\)). The spike time \(t_f\) is determined by \(i \cdot g(t_f) = \theta\), i.e., \(g(t_f) = \theta / i\). Since \(g\) is strictly increasing and differentiable, its inverse \(g^{-1}\) is also differentiable, and we have \(t_f = g^{-1}(\theta / i)\). Then:
\[
\frac{dt_f}{di} = -\frac{\theta}{i^2} \cdot (g^{-1})'(\theta/i).
\]
Both factors are bounded on any closed interval \(i \in [i_{\min}, i_{\max}]\) with \(i_{\min} > 0\). Hence, \(t_f\) is Lipschitz in \(i\). The result extends to non-constant inputs by considering the cumulative effect; the Lipschitz constant depends on the parameters \(\alpha, \beta, \theta\).
\end{proof}

\begin{theorem}[\textbf{Spike Encoding Information Preservation}]
\label{thm:spike_encoding}
Let \(\mathbf{x}, \mathbf{x}' \in \mathbb{R}^F\) be two distinct input feature vectors with \(\|\mathbf{x}\|_2, \|\mathbf{x}'\|_2 \leq B\) and \(\|\mathbf{x} - \mathbf{x}'\|_2 \geq \delta\). For the TSGE with hidden dimension \(H\) and simulation steps \(T\), there exists a parameter configuration such that the resulting spike tensors \(\mathbf{S}, \mathbf{S}' \in \{0,1\}^{T \times H}\) satisfy:
\[
d_{\text{Hamming}}(\mathbf{S}, \mathbf{S}') \geq \rho(T, H, \delta) > 0,
\]
where \(\rho(T, H, \delta) = \min\left\{ \left\lfloor \frac{c \delta H}{\sigma_{\max}} \right\rfloor, T H \right\}\) for some constant \(c > 0\) independent of \(T, H, \delta\), and \(\sigma_{\max}\) is the maximum singular value of the projection matrix \(\mathbf{W}_p\). Moreover, as \(T\) and \(H\) increase, \(\rho\) grows linearly in both.
\end{theorem}

\begin{proof}[\textbf{Proof Sketch}]
The TSGE first projects inputs via \(\mathbf{y} = \mathbf{W}_p \mathbf{x} \in \mathbb{R}^H\). By choosing \(\mathbf{W}_p\) to be an orthogonal matrix (e.g., via random orthogonal initialisation), we have \(\|\mathbf{y} - \mathbf{y}'\|_2 = \|\mathbf{x} - \mathbf{x}'\|_2 \geq \delta\). For each dimension \(j \in \{1,\dots, H\}\), the input current to the \(j\)-th LIF neuron is proportional to \(y_j\). 

The first‑spike time \(t_j\) as a function of \(y_j\) is strictly decreasing and differentiable. Because the possible values of \(y_j\) lie in a compact interval (since \(\|\mathbf{x}\|_2 \le B\) and \(\mathbf{W}_p\) is orthogonal, \(|y_j| \le B\)), and the spike time is bounded between \(1\) and \(T\) (we can adjust parameters to ensure spikes occur within the window), the derivative \(dt_j/dy_j\) is continuous and non‑zero on this compact set. Hence it attains a positive minimum in absolute value; denote \(m = \min |dt_j/dy_j| > 0\). By the Mean Value Theorem,
\[
|t_j - t_j'| \ge m \, |y_j - y_j'|.
\]
Summing over all dimensions, we obtain:
\[
\sum_{j=1}^H |t_j - t_j'| \ge m \sum_{j=1}^H |y_j - y_j'| \ge m \|\mathbf{y} - \mathbf{y}'\|_2 \ge m\delta.
\]

The spike tensor \(\mathbf{S}\) has a 1 at position \((t, j)\) if neuron \(j\) spikes at time \(t\); otherwise 0. For two different spike time vectors, the Hamming distance between the spike tensors is at least the number of time steps where they differ. Specifically, if \(|t_j - t_j'| = \Delta t\), then over \(T\) time steps the contribution of neuron \(j\) to the Hamming distance is at least \(\min(\Delta t, T)\). Therefore,
\[
d_{\text{Hamming}}(\mathbf{S}, \mathbf{S}') \ge \sum_{j=1}^H \min\bigl(|t_j - t_j'|, T\bigr) \ge \sum_{j=1}^H \min\bigl(m|y_j - y_j'|, T\bigr).
\]

Using \(\sum_j |y_j - y_j'| \ge \delta\), we bound the sum of minima from below. By the pigeonhole principle, at least \(\lfloor H/2 \rfloor\) coordinates satisfy \(|y_j - y_j'| \ge \delta/(2H)\). For each such coordinate, the contribution is at least \(\min(m\delta/(2H), T)\). Hence,
\[
d_{\text{Hamming}}(\mathbf{S}, \mathbf{S}') \ge \frac{H}{2} \cdot \min\!\left(\frac{m\delta}{2H},\, T\right) = \min\!\left(\frac{m\delta}{4},\ \frac{HT}{2}\right).
\]

Thus \(\rho(T, H, \delta) = \min(c\delta, HT/2)\) with \(c = m/4\). The linear growth in \(H\) and \(T\) is evident.
\end{proof}

\paragraph{Remark:}
\textit{This theorem ensures that distinct inputs produce distinct spike patterns, with the resolution increasing with the number of neurons and simulation steps. It provides a theoretical foundation for using spike-based representations in anomaly detection, as anomalies correspond to inputs that deviate from normal patterns, and such deviations will be reflected in the spike domain.}

\subsubsection{Approximation Capacity of LIF Graph Attention}
We now analyse the LIFGAT component and show that it can approximate any continuous attention function.

\begin{lemma}[\textbf{Spike‑Timing Modulation as a Universal Basis}]
\label{lem:lifgat_approximation}
Let \(T \in \mathbb{N}\) and \(C_{\max} > 0\). Consider the family of functions \(\mathcal{F}_T\) defined on the compact domain \(D = [0,T] \times [0,C_{\max}]\) by
\[
\mathcal{F}_T = \left\{ f_{a,\tau}(\bar{T},C) = (1 + aC)\left(1 - \frac{\lceil \bar{T}\rceil - 1}{T}\right) \;:\; a \in \mathbb{R},\; \tau \in \{1,\dots,T\} \right\},
\]
where \(\lceil \bar{T}\rceil\) is the ceiling of \(\bar{T}\) (interpreted as the index of the time bin). For each fixed \(\tau\), \(f_{a,\tau}\) is linear in \(C\) and constant in \(\bar{T}\) on each interval \([\tau-1,\tau)\). The linear span of \(\mathcal{F}_T\) is a \(2T\)-dimensional subspace of functions that are piecewise constant in \(\bar{T}\) and linear in \(C\). Moreover, as \(T \to \infty\), the set \(\bigcup_{T} \operatorname{span}(\mathcal{F}_T)\) is dense in the space of continuous functions on \(D\) with respect to the supremum norm.
\end{lemma}
\begin{proof}[\textbf{Proof Sketch}]
For fixed \(T\), each function \(f_{a,\tau}\) can be written as 
\(\alpha_\tau(C)\,\mathbf{1}_{[\tau-1,\tau)}(\bar{T})\), where
\[
    \alpha_\tau(C) = (1+aC)\!\left(1 - \frac{\tau-1}{T}\right).
\]
Taking linear combinations over \(\tau\) yields functions of the form
\[
    \sum_{\tau=1}^T \alpha_\tau(C)\,\mathbf{1}_{[\tau-1,\tau)}(\bar{T}).
\]
This space has dimension at most \(2T\) because each \(\alpha_\tau\) is 
affine in \(C\) (two degrees of freedom per \(\tau\)).

To prove denseness as \(T\to\infty\), we apply the Stone--Weierstrass 
theorem to the algebra \(\mathcal{A}\) generated by 
\(\bigcup_T \operatorname{span}(\mathcal{F}_T)\).

\textit{Separation of points.} For distinct pairs 
\((\bar{T}_1,C_1)\neq(\bar{T}_2,C_2)\): if \(\bar{T}_1\neq\bar{T}_2\), 
choose any indicator \(\mathbf{1}_{[\tau-1,\tau)}\) that contains one 
point but not the other; if \(\bar{T}_1=\bar{T}_2\) but 
\(C_1\neq C_2\), the linear dependence on \(C\) in \(\alpha_\tau\) 
distinguishes them.

\textit{Constants are in the span.} Since the indicators 
\(\{\mathbf{1}_{[\tau-1,\tau)}\}_{\tau=1}^T\) are linearly independent 
and partition \([0,T)\), any piecewise-constant function on this 
partition—including the constant function \(c\)—lies in 
\(\operatorname{span}(\mathcal{F}_T)\). Explicitly, choosing \(a=0\) 
and weights \(w_\tau = c\,\big(1-(\tau-1)/T\big)^{-1}\) for each 
\(\tau\) gives:
\begin{equation}
    \sum_{\tau=1}^T w_\tau\,\alpha_\tau(C)\,\mathbf{1}_{[\tau-1,\tau)}(\bar{T})
    = c\,\sum_{\tau=1}^T \mathbf{1}_{[\tau-1,\tau)}(\bar{T}) = c.
\end{equation}

\textit{Denseness.} The algebra \(\mathcal{A}\) therefore separates 
points and contains constants on the compact domain \(D\). By the 
Stone--Weierstrass theorem, its closure equals \(C(D)\), the space of 
all continuous functions on \(D\). Hence any continuous 
\(\phi(\bar{T},C)\) can be approximated arbitrarily well by a finite 
linear combination of functions \(f_{a,\tau}\) for sufficiently 
large~\(T\).
\end{proof}

\begin{theorem}[\textbf{LIFGAT Approximation Capacity}]
\label{thm:lifgat_approximation}
Let \(\mathbf{H} \in \mathbb{R}^{N \times H}\) be node features, \(\mathbf{T} \in \mathbb{R}^{N \times H}\) spike times, and \(\mathbf{C} \in \mathbb{Z}^{N \times H}\) spike counts. For any continuous attention function \(\mathcal{A}^*: \mathbb{R}^{N \times H} \times \mathbb{R}^{N \times H} \to [0,1]^{N \times N}\) that is row‑stochastic and permutation equivariant, and any \(\epsilon > 0\), there exists a parameter setting of the LIFGAT component with \(M\) heads and \(T\) time steps such that the output spikes \(\mathbf{S}_{\text{att}}\) satisfy:
\[
\left\|\frac{1}{T}\sum_{t=1}^T \mathbf{S}_{\text{att}}^{(t)} - \mathcal{A}^*(\mathbf{H}, \mathbf{T})\right\|_F < \epsilon,
\]
provided \(M\) and \(T\) are sufficiently large (polynomial in \(N\), \(1/\epsilon\), and the Lipschitz constants of \(\mathcal{A}^*\)).
\end{theorem}

\begin{proof}[\textbf{Proof Sketch}]
\textit{The proof proceeds in four steps: (1) approximation of query/key/value projections, (2) construction of the target logits via modulation, (3) rigorous implementation of softmax via lateral inhibition with convergence guarantees, and (4) temporal averaging.}

\textit{Step 1: Projection approximation.}
The query, key, and value projections are linear maps: \(\mathbf{Q} = \mathbf{H}\mathbf{W}_Q\), \(\mathbf{K} = \mathbf{H}\mathbf{W}_K\), \(\mathbf{V} = \mathbf{H}\mathbf{W}_V\), with \(\mathbf{W}_Q, \mathbf{W}_K, \mathbf{W}_V \in \mathbb{R}^{H \times M \times D}\). Linear maps form a finite‑dimensional space; we can set them exactly to any desired values \(\mathbf{Q}^*, \mathbf{K}^*, \mathbf{V}^*\) by direct assignment. Hence we can realise any target projections with zero error.

\textit{Step 2: Building the target logits.}
Let \(\mathbf{L}^*_{ij}\) be the logits of the desired attention matrix \(\mathcal{A}^*\), i.e. \(\mathcal{A}^*_{ij} = \operatorname{softmax}_j(\mathbf{L}^*_{ij})\). For each head \(m\) and each pair \((i,j)\), define the contribution to the membrane potential as
\[
\Delta U^{(t)}_{ijm} = \bigl( \mathbf{Q}^*_{im} \cdot \mathbf{K}^{*\top}_{jm} \bigr) \cdot \Gamma^{(t)}_{ijm},
\]
where $\Gamma^{(t)}_{ijm}$ is a modulation factor that depends on spike times and counts. By Lemma~\ref{lem:lifgat_approximation}, we can choose parameters so that the time-averaged modulation $\frac{1}{T}\sum_t \Gamma^{(t)}_{ijm}$ approximates any desired function of $(\bar{T}_{ij}, C_{ij})$. Consequently, we can achieve:
\[
    \left\|\frac{1}{T}\sum_t \mathbf{L}^{(t)} - \mathbf{L}^*\right\|_F < \epsilon_2
\]
with sufficiently many heads and time steps.

\textit{Step 3: Lateral inhibition as softmax (rigorous construction).}
We now provide an explicit construction showing that lateral inhibition implements a form of competitive normalisation that exactly computes softmax in the steady state, without invoking a mean-field limit.

Consider a single head $m$ with $K$ competing neurons. Let $u_k(t)$ denote the membrane potential of neuron $k$ at time $t$, and let $s_k(t) = \mathbb{I}(u_k(t) \geq \theta)$ be its spike output. The lateral inhibition dynamics are:
\[
u_k(t+1) = \beta u_k(t) + I_k(t) - \eta \sum_{\ell \neq k} s_\ell(t),
\]
where $I_k(t)$ is the excitatory input, $\beta \in (0,1)$ is the decay factor, and $\eta > 0$ is the inhibition strength.

\begin{lemma}[Lateral Inhibition as Softmax Normalisation]
\label{lem:lateral_softmax}
For sufficiently large inhibition strength $\eta > \max_k I_k(t)$, the steady-state firing rates $r_k = \lim_{T \to \infty} \frac{1}{T} \sum_{t=1}^T s_k(t)$ satisfy:
\[
r_k = \frac{\exp(I_k / \tau)}{\sum_{\ell=1}^K \exp(I_\ell / \tau)} + O(\epsilon),
\]
where $\tau = -\beta / \log(\beta)$ is the effective time constant, and $\epsilon$ decays exponentially with $\eta$.
\end{lemma}

\begin{proof}[Proof Sketch]
Define the effective potential $v_k = \lim_{t \to \infty} \mathbb{E}[u_k(t)]$. In steady state, the expected update satisfies:
\[
v_k = \beta v_k + I_k - \eta \sum_{\ell \neq k} r_\ell.
\]
Solving for $v_k$:
\[
(1-\beta)v_k = I_k - \eta (1 - r_k).
\]
Since spikes occur when $v_k \geq \theta$, we have $r_k = \sigma(v_k - \theta)$ where $\sigma$ is the sigmoidal firing rate function of a LIF neuron. Substituting and solving the fixed-point equation yields the softmax expression. For $\eta \to \infty$, the inhibition forces a winner-take-all that approximates the Boltzmann distribution, with the approximation error decaying as $\exp(-\eta/\tau)$. See \cite{dayan2005theoretical} for the detailed derivation.
\end{proof}

Applying Lemma~\ref{lem:lateral_softmax} to the LIFGAT architecture, we set the input $I_{im}$ to the accumulated attention logits:
\[
I_{im} = \frac{1}{\sqrt{D}} \sum_j \mathbf{Q}_{im} \cdot \mathbf{K}_{jm}^\top.
\]
Then after convergence, the firing rate of head $m$ for node $i$ satisfies:
\[
r_{im} = \frac{\exp(I_{im} / \tau)}{\sum_{m'=1}^M \exp(I_{im'} / \tau)} + O(\exp(-\eta/\tau)).
\]
By taking $\eta$ sufficiently large, we can make the approximation error arbitrarily small. The output attention weights are then recovered as $\mathcal{A}_{ij} = \sum_m r_{im} \cdot \frac{\mathbf{Q}_{im} \cdot \mathbf{K}_{jm}^\top}{\sum_{j'} \mathbf{Q}_{im} \cdot \mathbf{K}_{j'm}^\top}$, which approximates the target $\mathcal{A}^*$ with error $\epsilon_3$.

\textit{Step 4: Temporal averaging.}
The output of LIFGAT is $\bar{\mathbf{S}}_{\text{att}} = \frac{1}{T}\sum_{t=1}^T \mathbf{S}_{\text{att}}^{(t)}$. By Hoeffding's inequality,
\[
\Pr\!\left( \|\bar{\mathbf{S}}_{\text{att}} - \mathbb{E}[\mathbf{S}_{\text{att}}]\|_F \ge \epsilon_4 \right) \le 2\exp\!\left( -\frac{2T\epsilon_4^2}{N M D} \right).
\]
Choosing $T$ large enough ensures the empirical average concentrates around the expected firing rate with error $\epsilon_4$ with probability at least $1-\delta$.

Combining steps, the total approximation error is bounded by $\epsilon_2 + \epsilon_3 + \epsilon_4$. Setting each term $\le \epsilon/3$ and selecting $M, T, \eta$ appropriately yields the desired bound. The required parameters scale polynomially in $N$, $1/\epsilon$, and the Lipschitz constants of $\mathcal{A}^*$. 
\end{proof}

\paragraph{Remark:}
\textit{Lemma~\ref{lem:lateral_softmax} provides a rigorous justification for the softmax approximation, eliminating the hand-waving mean-field argument. The key insight is that lateral inhibition with sufficiently strong coupling implements a form of competitive normalisation that converges to the Boltzmann distribution. This result is well-established in theoretical neuroscience \cite{dayan2005theoretical, maass1997networks} and provides a solid foundation for the LIFGAT universal approximation guarantee.}

\subsubsection{Convergence of Event-Driven Hypergraph Memory}
We analyse the EDHMM update dynamics and prove convergence to a set of prototypes that minimise reconstruction error.

\begin{lemma}[\textbf{Stochastic Approximation Formulation}]
\label{lem:edhmm_convergence}
The EDHMM prototype update can be written as a Robbins--Monro algorithm for finding a zero of the function
\[
    \mathbf{G}_k(\{\mathbf{p}_l\}) = \mathbb{E}\bigl[\mathbb{I}(k^* = k)\,(\mathbf{z} - \mathbf{p}_k)\bigr],
\]
where $k^* = \arg\min_l \|\mathbf{z} - \mathbf{p}_l\|$. The zeros of $\mathbf{G}_k$ correspond to stationary points of the distortion function $J = \mathbb{E}[\min_l \|\mathbf{z} - \mathbf{p}_l\|^2]$.
\end{lemma}

\begin{proof}[\textbf{Proof Sketch}]
The online update in EDHMM (with learning rate $\alpha_t$) is
\[
    \mathbf{p}_k^{(t+1)} = \mathbf{p}_k^{(t)} + \frac{\alpha_t}{n_k^{(t)}+\epsilon}\,\mathbb{I}(k^*_t = k)\,(\mathbf{z}_t - \mathbf{p}_k^{(t)}),
\]
where $n_k^{(t)}$ counts how many times prototype $k$ has been the winner up to time $t$. For large $t$, $n_k^{(t)} \approx t\,\pi_k$ with $\pi_k = \Pr(k^* = k)$, so the effective step size for prototype $k$ behaves like $\alpha_t/(t\pi_k)$. Setting $\alpha_t = 1$ (as in the EDHMM description) resembles a standard online $k$-means algorithm with a per-prototype learning rate decaying as $1/(t\pi_k)$. In continuous time, the mean drift and the gradient of the distortion function $J$ are
\[
    \dot{\mathbf{p}}_k = \frac{1}{\pi_k}\,\mathbb{E}\bigl[\mathbb{I}(k^* = k)\,(\mathbf{z} - \mathbf{p}_k)\bigr], \qquad
    \nabla_{\mathbf{p}_k} J = -2\,\mathbb{E}\bigl[\mathbb{I}(k^* = k)\,(\mathbf{z} - \mathbf{p}_k)\bigr],
\]
so the zeros of $\mathbf{G}_k$ are exactly the stationary points of $J$.
\end{proof}

\begin{theorem}[\textbf{EDHMM Convergence}]
\label{thm:edhmm_convergence}
Let $\{\mathbf{z}^{(t)}\}_{t \geq 1}$ be i.i.d.\ samples from a distribution with compact support in $\mathbb{R}^H$. Under the EDHMM update with $\alpha_t = 1$ (and implicit normalisation by $n_k^{(t)}$), the prototypes $\{\mathbf{p}_k^{(t)}\}$ converge almost surely to local minima of the expected reconstruction error $J(\{\mathbf{p}_k\}) = \mathbb{E}[\min_k \|\mathbf{z} - \mathbf{p}_k\|^2]$. Moreover, if the distribution is a mixture of well-separated components, each prototype converges to a distinct component mean.
\end{theorem}

\begin{proof}[\textbf{Proof Sketch}]
\textbf{Part 1: Convergence to stationary points.}
Write the update as
\[
    \mathbf{p}_k^{(t+1)} = \mathbf{p}_k^{(t)} + \gamma_{k,t}\,\mathbf{H}_k^{(t)},
\]
with $\gamma_{k,t} = \frac{1}{n_k^{(t)}+\epsilon}$ and $\mathbf{H}_k^{(t)} = \mathbb{I}(k^*_t = k)\,(\mathbf{z}_t - \mathbf{p}_k^{(t)})$. This is a stochastic approximation whose gain $\gamma_{k,t}$ depends on the entire past through $n_k^{(t)}$. Standard results \cite{maass1997networks} show that the interpolated process converges almost surely to the solution of the ODE
\[
    \dot{\mathbf{p}}_k = \frac{1}{\pi_k(\mathbf{p})}\,\mathbb{E}\bigl[\mathbb{I}(k^* = k)\,(\mathbf{z} - \mathbf{p}_k)\bigr],
\]
where $\pi_k(\mathbf{p}) = \Pr(k^* = k)$. Since the right-hand side is proportional to $-\nabla_{\mathbf{p}_k}J$ (up to the positive factor $1/(2\pi_k)$), $J$ serves as a Lyapunov function:
\[
    \frac{d}{dt}J = \sum_k \nabla_{\mathbf{p}_k}J \cdot \dot{\mathbf{p}}_k = -\frac{1}{2}\sum_k \frac{\|\nabla_{\mathbf{p}_k}J\|^2}{\pi_k} \leq 0.
\]
All limit points of the ODE are therefore stationary points of $J$, and by the Kushner--Clark lemma the stochastic approximation converges almost surely to this set.

\medskip
\textbf{Part 2: Separation into distinct components.}
Assume the data arise from a mixture of $P$ distributions with means $\mu_1,\dots,\mu_P$ satisfying
\[
    \min_{k \neq l}\|\mu_k - \mu_l\| > 4\max_k \sigma_k,
\]
where $\sigma_k^2$ is the maximum variance of component $k$. Initialise prototypes $\{\mathbf{p}_k^{(0)}\}_{k=1}^P$ by sampling from the data distribution. Under this condition, the Voronoi cells
\[
    V_k = \bigl\{\mathbf{z} : \|\mathbf{z} - \mu_k\| < \|\mathbf{z} - \mu_l\|\;\;\forall\, l \neq k\bigr\}
\]
are disjoint with $\Pr(\mathbf{z} \in V_k \mid \text{component } k) = 1$. Define the basin $B_k = \{\mathbf{p} : \|\mathbf{p} - \mu_k\| < \frac{1}{2}\min_{l \neq k}\|\mu_k - \mu_l\|\}$. With probability one there exists a finite time $t_k$ such that $\mathbf{p}_k^{(t)} \in B_k$ for all $t \geq t_k$. Thereafter the update
\[
    \mathbf{p}_k^{(t+1)} = \mathbf{p}_k^{(t)} + \frac{\alpha_t}{n_k^{(t)}+\epsilon}\sum_{i:\,\mathbf{z}_i \in V_k}\!\bigl(\mathbf{z}_i - \mathbf{p}_k^{(t)}\bigr)
\]
uses only samples from component $k$, so by the strong law of large numbers $\mathbf{p}_k^{(t)} \xrightarrow{a.s.} \mu_k$. The homeostatic and strength dynamics
\[
    \mathbf{h}_k^{(t+1)} = \lambda_h \,\mathbf{h}_k^{(t)}, \qquad
    \mathbf{s}_k^{(t+1)} = \lambda_s \,\mathbf{s}_k^{(t)} + \eta_s \,\sigma\!\left(\tfrac{n_k^{(t)}}{\theta_s}\right)
\]
with decay factors $\lambda_h, \lambda_s \in (0,1)$, learning increment $\eta_s > 0$, and scaling constant $\theta_s > 0$, amplify the influence of frequently assigned points, creating positive feedback that prevents prototype switching. A Borel--Cantelli argument \cite{kushner2003stochastic}, \cite{kushner2012stochastic} then establishes that with probability one each prototype converges to a distinct component mean.
\end{proof}

\begin{remark}
This theorem guarantees that EDHMM learns stable prototypes corresponding to normal behaviour patterns. The convergence rate is $O(1/\sqrt{t})$ in expectation, typical of stochastic approximation algorithms. The separation condition ensures that distinct normal regimes are captured by different prototypes, a crucial property for detecting anomalies as deviations from any learned prototype.
\end{remark}

\subsubsection{Anomaly Selection Properties of Contrast Pooling}
We analyse the SRCGP mechanism and prove that it selects anomalous nodes with higher probability than normal nodes.

\begin{lemma}[\textbf{Concentration of Irregularity Scores}]
\label{lem:srcgp_selection}
For a normal node \(i\) whose spike trains follow a Poisson process with rate \(\lambda_i\), the coefficient of variation \(\mathrm{CV}_i\) satisfies \(\mathbb{E}[\mathrm{CV}_i] = 1\) and 
\[
\mathrm{Var}(\mathrm{CV}_i) \leq \frac{2}{\lambda_i T} + O\left(\frac{1}{T^2}\right).
\]
For an anomalous node with overdispersed spikes (e.g., a renewal process with variance larger than the mean), \(\mathbb{E}[\mathrm{CV}_i] > 1\) and the excess grows with the dispersion parameter.
\end{lemma}

\begin{proof}[\textbf{Proof Sketch}]
For a Poisson process, inter-spike intervals are i.i.d. exponential with mean \(1/\lambda\) and variance \(1/\lambda^2\). The sample CV based on \(n \approx \lambda T\) intervals has mean 1 and variance approximately \(1/(2n)\); Hence \(\mathrm{Var}(\mathrm{CV}_i) \approx \frac{1}{2\lambda T}\). For an overdispersed renewal process, let the interval distribution have mean \(\mu\) and variance \(\sigma^2 > \mu^2\). Then the squared CV satisfies \(\mathbb{E}[\mathrm{CV}^2] = \sigma^2/\mu^2 > 1\), so by Jensen's inequality \(\mathbb{E}[\mathrm{CV}] > 1\). The burst measure \(\mathrm{Burst}_i\) estimates \(\Pr(\Delta t < 3)\); under a Poisson process this equals \(1 - e^{-3\lambda}\), while for bursty processes this probability is larger.
\end{proof}

\begin{theorem}[\textbf{SRCGP Anomaly Selection}]
\label{thm:srcgp_selection}
Consider \(N\) nodes with \(N_a\) anomalies and \(N_n = N - N_a\) normal nodes. Let the irregularity scores \(s_i = \gamma_{\mathrm{CV}}\mathrm{CV}_i + \gamma_{\mathrm{burst}}\mathrm{Burst}_i\) be i.i.d. within each class, with distributions \(F_a\) (anomalous) and \(F_n\) (normal). Assume \(F_a\) stochastically dominates \(F_n\), i.e., \(F_a(x) \leq F_n(x)\) for all \(x\). Let \(K = \lceil \rho N \rceil\) be the number of nodes selected (those with the highest scores). Then
\[
\mathbb{E}[\#\text{anomalies among the selected}] \;\geq\; K\,\frac{N_a}{N}.
\]
If additionally the scores are sub-Gaussian with parameters \(\sigma_a, \sigma_n\) and have means \(\mu_a > \mu_n\), then with probability at least \(1 - \delta\),
\[
\#\text{anomalies} \;\geq\; K\,\frac{N_a}{N} \;+\; K\,\frac{N_a N_n}{N^2}\,\frac{\mu_a - \mu_n}{2(\sigma_a + \sigma_n)} \;-\; C\sqrt{K\log(1/\delta)},
\]
for some universal constant \(C > 0\).
\end{theorem}

\begin{proof}[\textbf{Proof Sketch}]
Let $q_K$ be the $K$-th order statistic of the combined sample. The expected number of selected anomalies is
\[
    \mathbb{E}[A] = \mathbb{E}\!\left[\sum_{i=1}^{N_a} \mathbf{1}\{s_i^{(a)} \geq q_K\}\right] = N_a\,\Pr(s_a \geq q_K),
\]
where the second equality follows by exchangeability. Since $F_a \leq F_n$, for any fixed threshold $t$ we have $\Pr(s_a \geq t) \geq \Pr(s_n \geq t)$.
Conditioning on \(q_K\) and integrating yields
\[
\mathbb{E}[A] = N_a \mathbb{E}[\Pr(s_a \geq q_K \mid q_K)] \geq N_a \mathbb{E}[\Pr(s_n \geq q_K \mid q_K)] = N_a \cdot \frac{K - \mathbb{E}[A]}{N_n},
\]
where the last equality uses that the expected number of normal nodes above $q_K$ is $K - \mathbb{E}[A]$. Solving gives $\mathbb{E}[A] \geq K N_a/N$.

For the second bound, set $t = (\mu_a + \mu_n)/2$. By sub-Gaussian tail bounds,
\[
\Pr(s_n > t) \leq \exp\!\left(-\frac{(t - \mu_n)^2}{2\sigma_n^2}\right), \quad
\Pr(s_a \leq t) \leq \exp\!\left(-\frac{(\mu_a - t)^2}{2\sigma_a^2}\right).
\]
Let $p_a = \Pr(s_a > t)$ and $p_n = \Pr(s_n > t)$. Define the empirical counts:
\[
\hat{N}_a(t) = \sum_{i=1}^{N_a} \mathbf{1}\{s_i^{(a)} > t\}, \quad
\hat{N}_n(t) = \sum_{i=1}^{N_n} \mathbf{1}\{s_i^{(n)} > t\}.
\]
By Hoeffding's inequality, with probability at least $1 - \delta/2$,
\[
|\hat{N}_a(t) - N_a p_a| \leq \sqrt{\frac{N_a}{2}\log\frac{2}{\delta}}, \quad
|\hat{N}_n(t) - N_n p_n| \leq \sqrt{\frac{N_n}{2}\log\frac{2}{\delta}}.
\]

The number of selected nodes with score $> t$ is $\hat{N}_a(t) + \hat{N}_n(t)$. If this sum exceeds $K$, then $q_K > t$; otherwise $q_K \leq t$. In either case, the number of selected anomalies satisfies
\[
A \geq \hat{N}_a(t) - (K - \hat{N}_a(t) - \hat{N}_n(t))_+ \geq 2\hat{N}_a(t) + \hat{N}_n(t) - K,
\]
where the second inequality uses that at most $K - \hat{N}_a(t) - \hat{N}_n(t)$ additional nodes from below $t$ can be selected if $q_K \leq t$. Taking expectations and using the concentration bounds,
\[
\mathbb{E}[A] \geq 2N_a p_a + N_n p_n - K - O\left(\sqrt{N\log(1/\delta)}\right).
\]

Now observe that $N_a p_a + N_n p_n$ is the expected number of nodes with score $> t$. By the mean value theorem and sub-Gaussianity,
\[
p_a - p_n \geq \frac{\mu_a - \mu_n}{\sigma_a + \sigma_n} \cdot c
\]
for some constant $c \geq 1/2$ (via a Gaussian approximation). Choose $t$ so that $N_a p_a + N_n p_n = K + \Delta$ with $\Delta \approx K \frac{N_a N_n}{N^2} \frac{\mu_a - \mu_n}{\sigma_a + \sigma_n}$. Substituting yields
\[
\mathbb{E}[A] \geq K\frac{N_a}{N} + K\frac{N_a N_n}{N^2}\frac{\mu_a - \mu_n}{2(\sigma_a + \sigma_n)} - C\sqrt{K\log(1/\delta)},
\]
where $C$ absorbs constants from the concentration inequalities. A union bound over the two concentration events gives probability at least $1 - \delta$.
\end{proof}

\begin{remark}
This theorem quantifies the advantage of irregularity-based pooling: the expected number of anomalies in the selected set exceeds the random expectation by an amount proportional to the effect size \(\mu_a - \mu_n\). It justifies SRCGP as an effective mechanism for focusing computational resources on potentially anomalous nodes while providing finite-sample concentration guarantees.
\end{remark}

\subsubsection{Convergence and Stability of STDP Learning}
We analyse the STDP update rule and show that it converges to a fixed point that encodes temporal correlations.

\begin{lemma}[\textbf{STDP as Gradient of an Energy Function}]
\label{lem:stdp_convergence}
Consider a pair of neurons with average spike times $\bar{t}_{\text{pre}}$ and $\bar{t}_{\text{post}}$. The expected STDP update $\mathbb{E}[\Delta w]$ can be written as the negative gradient of an energy function:
\[
    \mathbb{E}[\Delta w] = -\frac{\partial E}{\partial w},
\]
where, up to first order in the time difference,
\[
    E(w) = \frac{A_+\tau_+}{2}\,\mathbb{E}\bigl[(\bar{t}_{\text{post}} - \bar{t}_{\text{pre}})^2\,\mathbb{I}(\bar{t}_{\text{post}} > \bar{t}_{\text{pre}})\bigr]
           + \frac{A_-\tau_-}{2}\,\mathbb{E}\bigl[(\bar{t}_{\text{pre}} - \bar{t}_{\text{post}})^2\,\mathbb{I}(\bar{t}_{\text{pre}} > \bar{t}_{\text{post}})\bigr].
\]
\end{lemma}

\begin{proof}[Proof Sketch]
The STDP rule for a single pair is
\[
    \Delta w =
    \begin{cases}
        A_+\, e^{-|\Delta t|/\tau_+},  & \Delta t > 0, \\
        -A_-\, e^{-|\Delta t|/\tau_-}, & \Delta t < 0,
    \end{cases}
\]
with $\Delta t = t_{\text{post}} - t_{\text{pre}}$. For small $|\Delta t|$, expanding $e^{-|\Delta t|/\tau} \approx 1 - |\Delta t|/\tau$ gives
\[
    \Delta w \approx
    \begin{cases}
        A_+(1 - \Delta t/\tau_+),  & \Delta t > 0, \\
        -A_-(1 + \Delta t/\tau_-), & \Delta t < 0.
    \end{cases}
\]
Taking expectations over spike times, the constant terms cancel when $A_+ = A_-$. In general, define
\[
    E(w) = \frac{A_+}{2\tau_+}\,\mathbb{E}\bigl[(\Delta t)^2\,\mathbb{I}(\Delta t > 0)\bigr]
           + \frac{A_-}{2\tau_-}\,\mathbb{E}\bigl[(\Delta t)^2\,\mathbb{I}(\Delta t < 0)\bigr].
\]
Then
\[
    -\frac{\partial E}{\partial w}
    = \frac{A_+}{\tau_+}\,\mathbb{E}[\Delta t\,\mathbb{I}(\Delta t > 0)]
    - \frac{A_-}{\tau_-}\,\mathbb{E}[\Delta t\,\mathbb{I}(\Delta t < 0)]
    + \text{(distribution derivative terms)}.
\]
In the mean-field approximation where the distribution of $\Delta t$ is treated as independent of $w$, this recovers the expected STDP update.
\end{proof}

\begin{theorem}[\textbf{STDP Convergence}]
\label{thm:stdp_convergence}
Consider the STDP weight dynamics for a single synapse:
\[
    \frac{dw}{dt} = \mathbb{E}[\Delta w \mid w] = f(w),
\]
where $f(w)$ is the expected update given the current weight. Assume $f$ is Lipschitz and has a unique zero $w^*$ with $f'(w^*) < 0$. Then the stochastic approximation $w_{t+1} = w_t + \beta_t\,\Delta w_t$ with $\sum_t \beta_t = \infty$ and $\sum_t \beta_t^2 < \infty$ converges almost surely to $w^*$. Moreover, if the spike-time distribution depends on $w$ in a sufficiently regular way, the convergence rate is exponential in the number of updates.
\end{theorem}

\begin{proof}[\textbf{Proof Sketch}]
The update takes the form $w_{t+1} = w_t + \beta_t(f(w_t) + \xi_t)$, where $\xi_t$ is a martingale difference noise term. Under the stated assumptions, the ODE method applies: the associated ODE $\dot{w} = f(w)$ has a globally asymptotically stable equilibrium at $w^*$, and the stochastic approximation converges almost surely to $w^*$. The exponential rate follows from the linearised ODE $\dot{w} = f'(w^*)(w - w^*)$ near $w^*$, whose solutions decay as $e^{f'(w^*)t}$; the stochastic approximation inherits this rate with high probability.
\end{proof}

\begin{remark}
This theorem ensures that the STDP layer learns stable weights reflecting the temporal correlations in spike times. The fixed point $w^*$ encodes a balance between potentiation and depression, capturing the causal structure of the input.
\end{remark}

\subsubsection{Calibration of Multi-Factor Anomaly Fusion}
Finally, we analyse the fusion mechanism and prove that it produces calibrated anomaly scores.

\begin{theorem}[\textbf{Multi-Factor Anomaly Fusion Calibration}]
\label{thm:fusion_calibration}
Let $y_i \in \{0,1\}$ be the true anomaly indicator for node $i$. For each detection component $k \in \{1,\dots,5\}$, let $\hat{p}_{ik}$ be the estimated anomaly probability, and assume conditional unbiasedness given the node's state $\mathbf{z}_i$:
\[
    \mathbb{E}[\hat{p}_{ik} \mid \mathbf{z}_i] = \Pr(y_i = 1 \mid \mathbf{z}_i).
\]
Let $\boldsymbol{\lambda} \in \Delta^4$ be fusion weights. Then the fused score $p_i = \sum_k \lambda_k \hat{p}_{ik}$ is also conditionally unbiased. Moreover, if the estimators are conditionally independent given $\mathbf{z}_i$, the variance of $p_i$ is minimised when $\lambda_k \propto 1/\sigma_k^2$, where $\sigma_k^2 = \mathrm{Var}(\hat{p}_{ik} \mid \mathbf{z}_i)$. Under this optimal weighting,
\[
    \frac{\mathrm{Var}(p_i)}{\sigma_{\min}^2} = \frac{1}{\sum_k (\sigma_k^2/\sigma_{\min}^2)^{-1}} \leq \frac{1}{5}.
\]
\end{theorem}

\begin{proof}[\textbf{Proof Sketch}]
Conditional unbiasedness follows directly from linearity of expectation:
\[
    \mathbb{E}[p_i \mid \mathbf{z}_i]
    = \sum_k \lambda_k\,\mathbb{E}[\hat{p}_{ik} \mid \mathbf{z}_i]
    = \sum_k \lambda_k\,\Pr(y_i = 1 \mid \mathbf{z}_i)
    = \Pr(y_i = 1 \mid \mathbf{z}_i).
\]
For the variance, conditional independence gives $\mathrm{Var}(p_i \mid \mathbf{z}_i) = \sum_k \lambda_k^2\,\sigma_k^2$. Minimising subject to $\sum_k \lambda_k = 1$ via Lagrange multipliers yields
\[
    \lambda_k = \frac{1/\sigma_k^2}{\sum_l 1/\sigma_l^2}, \qquad
    \frac{\mathrm{Var}(p_i)}{\sigma_{\min}^2}
    = \frac{1}{\sum_k (\sigma_k^2/\sigma_{\min}^2)^{-1}}
    \leq \frac{1}{5},
\]
since each term $(\sigma_k^2/\sigma_{\min}^2)^{-1} \geq 1$ and the sum runs over $5$ components, with equality when all variances are equal.

\end{proof}

\begin{remark}
This theorem shows that multi-factor fusion not only preserves calibration but can significantly reduce variance, leading to more reliable anomaly scores. In practice the components are not perfectly independent, but the analysis provides a theoretical justification for combining multiple detection signals.
\end{remark}

\subsection{Computational Analysis}
\label{sec:complexity}

We analyse the computational complexity and energy efficiency of ASTDP-GAD, demonstrating its suitability for neuromorphic deployment. Let \(N\) denote the number of nodes, \(H\) the hidden dimension, \(M\) the number of attention heads, \(P\) the number of prototypes, \(T\) the simulation steps, and \(|\mathcal{E}|\) the number of edges. The average spike rate is \(\lambda \ll 1\), reflecting the sparse, event-driven nature of neuromorphic computation.

\textit{The time complexity per batch arises from each component.} The TSGE spike encoding requires \(O(T N H^2)\) operations for synaptic current updates and \(O(T N H)\) for membrane dynamics, with sparsity reducing the effective cost to \(O(\lambda T N H^2)\). The LIFGAT component computes attention scores for existing edges at cost \(O(T |\mathcal{E}| H)\) and applies lateral inhibition at \(O(T N M^2)\), yielding total \(O(T(|\mathcal{E}| H + N M^2))\). EDHMM computes prototype distances in \(O(N P H)\) and updates in \(O(P H)\), independent of \(T\). SRCGP calculates spike rates in \(O(N H)\), inter-spike interval statistics in \(O(\lambda T N H)\), and top-\(K\) selection in \(O(N \log N)\). The STDP layer performs a forward pass in \(O(N H^2)\) and weight updates in \(O(H^2)\). MSTC executes parallel convolutions in \(O(T N H)\) and projection in \(O(N H)\). Finally, fusion and loss computation add \(O(N H)\). The dominant terms are \(O(T N H^2)\) from TSGE and \(O(T |\mathcal{E}| H)\) from LIFGAT; for sparse graphs with \(|\mathcal{E}| \approx N\), total complexity is \(O(T N H^2)\), with sparsity reducing the constant by factor \(\lambda\).

\textit{Space complexity is dominated by parameter storage}: TSGE requires \(O(F H + H^2)\), LIFGAT \(O(M H^2)\), EDHMM \(O(P H)\), STDP \(O(H^2)\), and fusion \(O(H^2)\), giving total \(O(H^2 + P H)\). During training, the spike tensor \(\mathbf{S} \in \{0,1\}^{T \times N \times H}\) requires \(O(T N H)\) bits, while spike times \(\mathbf{T} \in \mathbb{R}^{N \times H}\) and counts \(\mathbf{C} \in \mathbb{Z}^{N \times H}\) require \(O(N H)\) floats.

Energy efficiency follows from spike-based computation. Neuromorphic hardware consumes energy proportional to the number of spike events \(E \propto \lambda T N H \cdot E_{\text{event}}\), where \(E_{\text{event}}\) is energy per spike (typically pJ). In contrast, conventional GNNs perform dense floating-point operations at cost \(O(T N H^2) \cdot E_{\text{FLOP}}\) with \(E_{\text{FLOP}} \gg E_{\text{event}}\). The sparsity factor \(\lambda \ll 1\) thus yields substantial energy savings.

Scalability is linear in \(N\), \(|\mathcal{E}|\), and \(T\), and quadratic in \(H\). For graphs with \(N > 10^5\), mini-batch sampling or graph partitioning maintains tractability. Prototype count \(P\) is constant, ensuring memory overhead independent of \(N\), while sparse attention avoids the \(O(N^2)\) bottleneck of dense attention mechanisms. Overall, ASTDP-GAD achieves \(O(T N H^2)\) time complexity with effective sparsity reduction \(\lambda \ll 1\), offering provably efficient neuromorphic anomaly detection suitable for resource-constrained edge deployment.

The theoretical analysis establishes ASTDP-GAD on rigorous foundations: Theorem~\ref{thm:spike_encoding} guarantees that spike encoding provably preserves input information with resolution scaling linearly in simulation steps \(T\) and hidden dimension \(H\); Theorem~\ref{thm:lifgat_approximation} demonstrates that the LIF-based graph attention mechanism can approximate any continuous attention function with arbitrarily small error; Theorem~\ref{thm:edhmm_convergence} proves convergence of the event-driven hypergraph memory to optimal prototypes that minimise reconstruction error; Theorem~\ref{thm:srcgp_selection} quantifies the advantage of spike rate contrast pooling, showing anomalous nodes are selected with strictly higher probability than normal nodes; Theorem~\ref{thm:stdp_convergence} establishes convergence of STDP learning to stable weights that encode temporal correlations; and Theorem~\ref{thm:fusion_calibration} demonstrates that multi-factor fusion produces calibrated anomaly scores with provable variance reduction under conditional independence. Additionally, computational analysis in Section~\ref{sec:complexity} shows that ASTDP-GAD achieves \(O(T N H^2)\) time complexity with effective sparsity factor \(\lambda \ll 1\), enabling energy-efficient deployment on neuromorphic hardware. Collectively, these results establish ASTDP-GAD as a principled, theoretically grounded framework for energy-efficient neuromorphic anomaly detection in dynamic networks.
\section{Additional Experiments}
\label{sec:add_experiments}
\subsection{Data Preprocessing and Feature Engineering}
\label{sec:preprocessing}

We evaluate ASTDP-GAD on a diverse set of static and dynamic graph anomaly detection benchmarks. The datasets are organised into four categories: dynamic temporal graphs (DBLP \cite{liu2021anomaly}, Tmall \cite{liu2021anomaly}, Patent \cite{liu2021anomaly}) with time steps of 27, 186, and 25 respectively; social networks (BlogCatalog \cite{liu2021anomaly}, T-Social \cite{li2023scaling}, Flickr \cite{liu2021anomaly}); fraud detection networks (Yelp \cite{dou2025deanomaly}, T-Finance \cite{wang2025graph}, Amazon \cite{liu2021anomaly}); and a synthetic dataset for controlled experiments. Table~\ref{tab:datasets} summarises the statistics of all evaluated datasets. For datasets that require anomaly injection (e.g., synthetic graphs), we follow standard procedures \cite{liu2021anomaly, dou2025deanomaly}: structural anomalies are created by forming dense cliques among a subset of anomaly nodes, and feature anomalies are injected by adding Gaussian noise to a random subset of feature dimensions, with noise levels calibrated to the dataset's feature scale. For the real‑world datasets, we use the original anomaly labels when available; otherwise, we inject anomalies at the published anomaly ratios specified in the dataset configurations.

All node features are normalised using z‑score transformation:
\[
\tilde{\mathbf{X}} = \frac{\mathbf{X} - \boldsymbol{\mu}}{\boldsymbol{\sigma} + \epsilon},\quad
\boldsymbol{\mu} = \frac{1}{N}\sum_{i=1}^N \mathbf{x}_i,\quad
\boldsymbol{\sigma}^2 = \frac{1}{N}\sum_{i=1}^N (\mathbf{x}_i - \boldsymbol{\mu})^2,
\]
with \(\epsilon=10^{-9}\) for numerical stability. Nodes are randomly partitioned into training (60\%), validation (20\%), and test (20\%) sets using stratified sampling to preserve the anomaly ratio across splits \cite{pareja2020evolvegcn}. For large graphs (e.g., Yelp, T‑Finance), we employ CSR‑based node mini‑batching \cite{yin2024dynamic} that extracts subgraphs on the fly, limiting per‑batch memory to \(O(|\mathcal{B}| \cdot \bar{d})\) where \(\bar{d}\) is the average degree. For smaller datasets, we use graph‑level batching with zero padding \cite{li2023scaling}. After preprocessing, all features are converted into spike trains by the adaptive LIF encoder described in Section~\ref{sec:spike_encoding} \cite{maass1997networks, lu2024deep}.

\begin{table}[ht]
\centering
\caption{Dataset statistics for graph anomaly detection, dynamic temporal graphs with multiple time-steps, social networks with natural anomalies, fraud detection networks, and synthetic controllable graphs. Time-steps indicate either actual temporal evolution or STDP simulation steps for static graphs.}
\label{tab:datasets}
\fontsize{7}{10}\selectfont
\resizebox{\textwidth}{!}{%
\begin{tabular}{lccccccc}
\hline
\textbf{Dataset} & \textbf{Nodes} & \textbf{Edges} & \textbf{Features} & \textbf{Time-Steps} & \textbf{Classes} & \textbf{Anomalies} & \textbf{Type} \\
\hline
\multicolumn{8}{l}{\textit{Dynamic Temporal Graphs}} \\
\hline
DBLP          & 28,085      & 4,807,545   & 64       & 27  & 10  & 2,247 (8.0\%)    & Dynamic Citation  \\
Tmall         & 577,314     & 2,738,012   & 32       & 186 & 5   & 69,278 (12.0\%)  & Dynamic E-commerce\\
Patent        & 236,894     & 13,960,811  & 128      & 25  & 6   & 14,214 (6.0\%)   & Dynamic Patent    \\
\hline
\multicolumn{8}{l}{\textit{Social Networks (Natural Anomalies)}} \\
\hline
BlogCatalog   & 5,196       & 171,743     & 8,189    & 10\textsuperscript{\dag}  & 39  & 300 (5.8\%)      & Social Static     \\
T-Social      & 5,781,065   & 73,105,508  & 10       & 100 & 2   & 174,280 (3.0\%)  & Social Temporal   \\
Flickr        & 7,575       & 239,738     & 12,407   & 10\textsuperscript{\dag}  & 9   & 450 (5.9\%)      & Social Static     \\
\hline
\multicolumn{8}{l}{\textit{Fraud Detection Networks}} \\
\hline
Yelp          & 716,847       & 13,954,819     & 300      & 10\textsuperscript{\dag}  & 2   & 700 (14.0\%)     & Review Fraud      \\
T-Finance     & 39,357      & 21,222,543  & 10       & 50  & 2   & 1,803 (4.6\%)    & Financial Temporal\\
Amazon        & 11,944      & 4,398,392   & 25       & 10\textsuperscript{\dag}  & 2   & 1,135 (9.5\%)    & E-commerce Fraud  \\
\hline
\multicolumn{8}{l}{\textsuperscript{\dag}STDP simulation steps for static graphs (no temporal evolution).} \\
\end{tabular}}
\end{table}

\subsection{Hyperparameter Settings}
\label{sec:hyperparameters}
All experiments are conducted with five independent runs using different random seeds, and the reported results are averaged over these runs to ensure statistical reliability. The number of simulation time steps \(T\) is dataset-specific: dynamic datasets DBLP, Tmall, and Patent use \(T=27\), \(186\), and \(25\) respectively to preserve temporal dynamics, while static datasets (BlogCatalog, Flickr, Yelp, Amazon, T-Social, T-Finance) use \(T=10\) for efficient STDP simulation encoding. The architecture uses a hidden dimension of \(128\) with three LIFGAT layers (\(K=3\)), \(P=50\) prototypes in the hypergraph memory, pooling ratio \(\rho=0.5\) (chosen via sensitivity analysis to balance anomaly preservation and computational cost), and multi-scale temporal kernels \(\{3,5,7\}\). Training employs AdamW with learning rate \(2\times10^{-4}\), weight decay \(5\times10^{-4}\), and batch size \(512\) for large-scale datasets (T-Social, T-Finance, Yelp, Amazon) or \(16\) for smaller datasets (BlogCatalog, Flickr), with node-level minibatch of \(1024\) for memory-constrained graphs. The multi-component loss coefficients are \(\alpha=0.6\) (memory), \(\beta=0.2\) (isolation), and \(\gamma=0.2\) (regularisation). STDP parameters follow established neuromorphic values: \(A_+=0.01\), \(A_-=0.012\), \(\tau_+=\tau_-=20.0\), and \(\beta_{\text{stdp}}=1\times10^{-4}\). EDHMM uses learning rate \(\alpha_{\text{edhmm}}=0.01\) with homeostatic decay \(0.999\). Training runs for up to \(200\) epochs on complex datasets (Tmall, Yelp, Patent, T-Social) and \(50\) epochs on simpler ones (BlogCatalog, Flickr, Amazon), with early stopping patience of \(15\) and a ReduceLROnPlateau scheduler (factor \(0.5\), patience \(10\)). All models are trained on NVIDIA V100 32GB GPUs with mixed precision enabled.

\subsection{Evaluation Metrics}
\label{sec:metrics}

Following prior work in graph anomaly detection \cite{ho2025graph, yang2025generalizable, li2025deep}, we evaluate ASTDP-GAD using the \textit{Area Under the Precision-Recall Curve (AUPRC)} and the \textit{macro F1 score (F1)} as primary metrics. AUPRC is robust to class imbalance \cite{dou2025deanomaly, wang2025mst}, while macro F1 provides a balanced threshold‑dependent measure. We also report the Area Under the Receiver Operating Characteristic Curve (AUROC) as a secondary metric.

Let $\mathbf{y} \in \{0,1\}^N$ denote ground‑truth labels (1 for anomalous) and $\hat{\mathbf{y}} \in [0,1]^N$ predicted anomaly scores. For a binary threshold $\tau$, define:
\[
\begin{aligned}
\text{TP}(\tau) &= \sum_i \mathbb{I}(y_i=1 \land \hat{y}_i \ge \tau), &
\text{FP}(\tau) &= \sum_i \mathbb{I}(y_i=0 \land \hat{y}_i \ge \tau),\\
\text{FN}(\tau) &= \sum_i \mathbb{I}(y_i=1 \land \hat{y}_i < \tau), &
\text{TN}(\tau) &= \sum_i \mathbb{I}(y_i=0 \land \hat{y}_i < \tau).
\end{aligned}
\]
Precision and recall are:
\[
\text{Precision}(\tau) = \frac{\text{TP}(\tau)}{\text{TP}(\tau)+\text{FP}(\tau)},\qquad
\text{Recall}(\tau) = \frac{\text{TP}(\tau)}{\text{TP}(\tau)+\text{FN}(\tau)}.
\]

\textbf{AUPRC} is the area under the precision–recall curve, approximated via average precision (AP):
\[
\text{AUPRC} = \int_0^1 \text{Precision}\bigl(\text{Recall}^{-1}(r)\bigr)\,dr.
\]

\textbf{AUROC} is the area under the ROC curve, which plots the true positive rate (recall) against the false positive rate:
\[
\text{FPR}(\tau) = \frac{\text{FP}(\tau)}{\text{FP}(\tau)+\text{TN}(\tau)},\qquad
\text{AUROC} = \int_0^1 \text{TPR}\bigl(\text{FPR}^{-1}(x)\bigr)\,dx.
\]

The \textbf{macro F1} score averages per-class F1 scores, where for each class $c \in \{0,1\}$ the per-class F1 and the macro average are:
\begin{equation}
\text{F1}_c = 2 \cdot \frac{\text{Precision}_c \cdot \text{Recall}_c}{\text{Precision}_c + \text{Recall}_c}, \qquad
\text{macro F1} = \frac{1}{2}\bigl(\text{F1}_{\text{normal}} + \text{F1}_{\text{anomaly}}\bigr),
\end{equation}
with $\text{Precision}_c$ and $\text{Recall}_c$ computed at the optimal threshold $\tau^*$, selected on the validation set to maximise the binary F1 of the anomaly class via precision--recall curves.

\textbf{Energy efficiency} is measured by average spikes per node per time step:
\begin{equation}
\text{Spike Density} = \frac{1}{NTH} \sum_{t=1}^{T} \sum_{i=1}^{N} \sum_{h=1}^{H} S_{iht},
\end{equation}
where $\mathbf{S} \in \{0,1\}^{N \times T \times H}$ is the spike tensor, $T$ the number of simulation steps, and $H$ the hidden dimension. Power consumption is estimated using typical neuromorphic hardware values \cite{greatorex2025neuromorphic, khan2025spiking}. All results are averaged over five independent runs, consistent with the implementation.

\subsection{Extended Baseline Methods}
\label{app:baseline}
 We compare ASTDP-GAD against a total of 24 SOTA  methods , grouped into (a) static graph anomaly detection and (b) dynamic graph anomaly detection (including dynamic graph neural networks, spiking graph networks, and anomaly detection on dynamic graphs). For non-spiking methods, continuous features are used after the spike encoder for fair comparison.

\textbf{Static graph anomaly detection methods:} DGMES-GAD \cite{singh2025anomaly} uses multiple encoding strategies via transformers; CoLA \cite{liu2021anomaly} employs contrastive self-supervised learning; RAND \cite{bei2023reinforcement} selects neighbourhoods via reinforcement learning; FIAD \cite{chen2025fiad} injects features for anomaly detection; GAGA \cite{wang2023label} enhances fraud detection with label information; GAD-NR \cite{roy2024gad} performs anomaly detection via neighborhood reconstruction; BWGNN \cite{tang2022rethinking} addresses heterophily with band-pass filters; GDN \cite{gao2023alleviating} alleviates structural distribution shift; GHRN \cite{gao2023addressing} tackles heterophily from a graph spectrum perspective; GFCN \cite{mesgaran2024graph} uses graph fairing convolutions; AHFAN \cite{wang2025graph} learns hybrid node representations; TADDY \cite{liu2021anomaly} is also applied to static graphs in its transformer-based formulation.

\textbf{Dynamic graph anomaly detection methods:}
\textit{(a) Dynamic graph neural networks (non‑spiking)}: HTNE \cite{3220054} learns temporal embeddings via neighbourhood formation; M2DNE \cite{3357943} captures micro- and macro-dynamics; DyTriad \cite{zhou2018dynamic} models triadic closure; MPNN-LSTM \cite{panagopoulos2021transfer} combines message passing with LSTM; EvolveGCN \cite{pareja2020evolvegcn} evolves GCN parameters; GeneralDyG \cite{yang2025generalizable} provides generalisable anomaly detection.
\textit{(b) Spiking graph networks:} SpikeNet \cite{li2023scaling} scales dynamic graph learning with spiking neurons; Dy-SIGN \cite{yin2024dynamic} introduces dynamic spiking GNNs; Delay-DSG \cite{wang2025delay} uses delay mechanisms; ChronoSpike \cite{jahin2026chronospike} is an adaptive spiking GNN.
\textit{(c) Dynamic graph anomaly detection:} DSGAD \cite{zheng2025dynamic} performs spectral anomaly detection; SLADE \cite{lee2024slade} detects anomalies in edge streams via self‑supervision; TADDY \cite{liu2021anomaly} is also included for its dynamic graph formulation.

\begin{table*}[!htbp]
\centering
\caption{Detailed ablation study on DBLP, Tmall, and Patent (80\% training). \cmark~= included, \xmark~= excluded. Results show Macro-F1 (\%) and AUPRC (\%) as mean $\pm$ std over five runs.}
\label{table:comprehensive_ablation}
\adjustbox{width=\textwidth,center}
{
\fontsize{6.5}{8}\selectfont
\renewcommand{\arraystretch}{0.75}
\setlength{\tabcolsep}{4pt}
\begin{tabular}{l|ccccccc|rr|rr|rr}
\toprule
\multirow{2}{*}{\textbf{Variant}} &
\multicolumn{7}{c|}{\textbf{Components}} &
\multicolumn{2}{c|}{\textbf{DBLP}} &
\multicolumn{2}{c|}{\textbf{Tmall}} &
\multicolumn{2}{c}{\textbf{Patent}} \\
\cmidrule(lr){2-8}\cmidrule(lr){9-10}\cmidrule(lr){11-12}\cmidrule(lr){13-14}
& \rotatebox{90}{\textbf{TSGE}}
& \rotatebox{90}{\textbf{LIF}}
& \rotatebox{90}{\textbf{EDHMM}}
& \rotatebox{90}{\textbf{SRCGP}}
& \rotatebox{90}{\textbf{STDP}}
& \rotatebox{90}{\textbf{MSTC}}
& \rotatebox{90}{\textbf{FUS}}
& \textbf{Macro-F1} & \textbf{AUPRC}
& \textbf{Macro-F1} & \textbf{AUPRC}
& \textbf{Macro-F1} & \textbf{AUPRC} \\
\midrule
\textbf{ASTDP-GAD (Full)}
  & \cmark & \cmark & \cmark & \cmark & \cmark & \cmark & \cmark
  & \textbf{85.34 $\pm$ 0.52} & \textbf{83.91 $\pm$ 0.59}
  & \textbf{76.89 $\pm$ 0.62} & \textbf{75.12 $\pm$ 0.68}
  & \textbf{92.58 $\pm$ 0.45} & \textbf{90.73 $\pm$ 0.51} \\
\midrule
\multicolumn{14}{c}{\textit{Single-Component Ablations}} \\
\midrule
w/o TSGE
  & \xmark & \cmark & \cmark & \cmark & \cmark & \cmark & \cmark
  & 68.91 $\pm$ 0.89 & 66.67 $\pm$ 0.95
  & 61.23 $\pm$ 0.94 & 58.89 $\pm$ 1.02
  & 77.89 $\pm$ 0.76 & 75.67 $\pm$ 0.84 \\
w/o LIFGAT
  & \cmark & \xmark & \cmark & \cmark & \cmark & \cmark & \cmark
  & 80.12 $\pm$ 0.71 & 78.56 $\pm$ 0.78
  & 71.89 $\pm$ 0.76 & 70.23 $\pm$ 0.83
  & 86.78 $\pm$ 0.62 & 84.89 $\pm$ 0.68 \\
w/o EDHMM
  & \cmark & \cmark & \xmark & \cmark & \cmark & \cmark & \cmark
  & 80.89 $\pm$ 0.68 & 79.12 $\pm$ 0.74
  & 72.89 $\pm$ 0.73 & 71.34 $\pm$ 0.79
  & 86.78 $\pm$ 0.58 & 85.01 $\pm$ 0.64 \\
w/o SRCGP
  & \cmark & \cmark & \cmark & \xmark & \cmark & \cmark & \cmark
  & 82.67 $\pm$ 0.64 & 81.23 $\pm$ 0.70
  & 74.56 $\pm$ 0.69 & 72.89 $\pm$ 0.75
  & 87.56 $\pm$ 0.54 & 85.89 $\pm$ 0.59 \\
w/o STDP
  & \cmark & \cmark & \cmark & \cmark & \xmark & \cmark & \cmark
  & 82.34 $\pm$ 0.62 & 80.89 $\pm$ 0.68
  & 73.89 $\pm$ 0.67 & 72.12 $\pm$ 0.73
  & 88.34 $\pm$ 0.52 & 86.67 $\pm$ 0.57 \\
w/o MSTC
  & \cmark & \cmark & \cmark & \cmark & \cmark & \xmark & \cmark
  & 84.67 $\pm$ 0.60 & 83.12 $\pm$ 0.66
  & 75.34 $\pm$ 0.65 & 73.67 $\pm$ 0.71
  & 89.89 $\pm$ 0.50 & 88.23 $\pm$ 0.55 \\
w/o Multi-Factor Fusion
  & \cmark & \cmark & \cmark & \cmark & \cmark & \cmark & \xmark
  & 84.23 $\pm$ 0.62 & 82.67 $\pm$ 0.68
  & 74.89 $\pm$ 0.67 & 73.23 $\pm$ 0.73
  & 89.01 $\pm$ 0.52 & 87.34 $\pm$ 0.57 \\
\midrule
\multicolumn{14}{c}{\textit{Two-Component Ablations (Selected)}} \\
\midrule
w/o TSGE + EDHMM
  & \xmark & \cmark & \xmark & \cmark & \cmark & \cmark & \cmark
  & 62.89 $\pm$ 1.12 & 60.67 $\pm$ 1.18
  & 55.23 $\pm$ 1.18 & 52.89 $\pm$ 1.24
  & 70.89 $\pm$ 0.94 & 68.67 $\pm$ 1.00 \\
w/o TSGE + SRCGP
  & \xmark & \cmark & \cmark & \xmark & \cmark & \cmark & \cmark
  & 64.23 $\pm$ 1.06 & 61.89 $\pm$ 1.12
  & 56.78 $\pm$ 1.12 & 54.23 $\pm$ 1.18
  & 72.34 $\pm$ 0.88 & 70.01 $\pm$ 0.94 \\
w/o EDHMM + SRCGP
  & \cmark & \cmark & \xmark & \xmark & \cmark & \cmark & \cmark
  & 78.34 $\pm$ 0.80 & 76.67 $\pm$ 0.86
  & 69.89 $\pm$ 0.85 & 67.12 $\pm$ 0.91
  & 83.89 $\pm$ 0.68 & 81.78 $\pm$ 0.74 \\
\midrule
\multicolumn{14}{c}{\textit{Three-Component Ablations (Selected)}} \\
\midrule
w/o TSGE + LIFGAT + EDHMM
  & \xmark & \xmark & \xmark & \cmark & \cmark & \cmark & \cmark
  & 54.78 $\pm$ 1.48 & 52.34 $\pm$ 1.54
  & 46.89 $\pm$ 1.52 & 44.23 $\pm$ 1.58
  & 62.34 $\pm$ 1.24 & 59.78 $\pm$ 1.30 \\
\midrule
\multicolumn{14}{c}{\textit{Replacement Ablations}} \\
\midrule
Replace LIFGAT with GCN
  & \cmark & \multicolumn{1}{c}{\scriptsize GCN} & \cmark & \cmark & \cmark & \cmark & \cmark
  & 82.34 $\pm$ 0.76 & 80.89 $\pm$ 0.82
  & 73.89 $\pm$ 0.81 & 72.12 $\pm$ 0.87
  & 87.89 $\pm$ 0.64 & 86.34 $\pm$ 0.70 \\
Replace EDHMM with $k$-means
  & \cmark & \cmark & \multicolumn{1}{c}{\scriptsize KM} & \cmark & \cmark & \cmark & \cmark
  & 83.67 $\pm$ 0.70 & 82.12 $\pm$ 0.76
  & 74.89 $\pm$ 0.75 & 73.23 $\pm$ 0.81
  & 88.67 $\pm$ 0.58 & 87.12 $\pm$ 0.64 \\
Replace STDP with backprop-only
  & \cmark & \cmark & \cmark & \cmark & \multicolumn{1}{c}{\scriptsize BP} & \cmark & \cmark
  & 84.56 $\pm$ 0.66 & 83.01 $\pm$ 0.72
  & 75.67 $\pm$ 0.71 & 74.12 $\pm$ 0.77
  & 89.89 $\pm$ 0.56 & 88.34 $\pm$ 0.62 \\
Replace MSTC with LSTM
  & \cmark & \cmark & \cmark & \cmark & \cmark & \multicolumn{1}{c}{\scriptsize LSTM} & \cmark
  & 85.01 $\pm$ 0.62 & 83.56 $\pm$ 0.68
  & 76.12 $\pm$ 0.67 & 74.56 $\pm$ 0.73
  & 91.23 $\pm$ 0.52 & 89.78 $\pm$ 0.58 \\
Replace Multi-Factor Fusion with Average
  & \cmark & \cmark & \cmark & \cmark & \cmark & \cmark & \multicolumn{1}{c}{\scriptsize Avg}
  & 84.89 $\pm$ 0.64 & 83.34 $\pm$ 0.70
  & 75.89 $\pm$ 0.69 & 74.23 $\pm$ 0.75
  & 90.56 $\pm$ 0.54 & 89.01 $\pm$ 0.60 \\
\bottomrule
\end{tabular}
}
\end{table*}
\vspace{-0.2cm}
\noindent\textit{Note:} Component ordering within $\pm 0.5\%$ degradation is not statistically significant due to confidence interval overlap; pairwise t-tests ($p<0.05$) confirm that only the top three components (TSGE, LIFGAT, EDHMM) are clearly distinguishable from the rest.

\subsection{Additional Ablation Study}
\label{subsec:additional_ablation}

To assess each component's contribution, Table~\ref{table:comprehensive_ablation} evaluates seven core components of ASTDP-GAD: TSGE, LIF, EDHMM, SRCGP, STDP, MSTC, and FUS. Results show Macro-F1 (\%) and AUPRC (\%) as mean $\pm$ std over five runs. Pairwise t-tests ($p<0.05$) confirm only the top three components (TSGE, LIFGAT, EDHMM) are statistically distinguishable.

\textit{Two-component ablations expose critical synergies:} Removing TSGE+EDHMM causes degradation of 22.45–23.12\%, exceeding the sum of individual removals (16.4\% + 5.2\% = 21.6\%) by 0.85–1.52\%, indicating spike encoding and prototype memory operate synergistically. TSGE+SRCGP removal (21.1--21.6\% drop) shows contrast pooling depends on spike irregularity signals that only TSGE provides. EDHMM+SRCGP removal (9.8-10.8\% drop) is approximately additive (5.2\% + 4.9\% = 10.1\%), suggesting independent contributions: memory captures prototype deviations while pooling identifies irregular firing.

\textit{Three-component ablations confirm architectural irreducibility:} Removing TSGE, LIFGAT, and EDHMM collapses performance to 46.89--62.34\% Macro-F1, a 30--35\% absolute drop below non-spiking baselines (GeneralDyG at 72.15\%). The drop is largest on Tmall (30.09\%), the densest dataset, indicating dense graphs require the full neuromorphic stack.

\textit{Replacement ablations validate design choices:} Replacing LIFGAT with GCN causes 5.0-6.0\% degradation; EDHMM with k-means causes 3.7-5.2\% degradation; STDP with backpropagation-only reduces performance by 2.4–3.9\%; MSTC with LSTM causes smaller degradation (1.9–2.7\%); multi-factor fusion with averaging causes 2.9--3.4\% degradation, consistent with Theorem~\ref{thm:fusion_calibration}.

\textit{Cross-dataset patterns:} On Tmall (dense, 577K nodes), TSGE removal causes 15.7\% drop vs 16.4\% on DBLP, but EDHMM removal causes only 5.0\% drop, suggesting dense graphs rely more on spike encoding precision. On Patent (sparse, 2.7M nodes), EDHMM removal causes 5.8\% drop, the largest memory contribution, indicating sparse graphs benefit more from prototype memory.

The neuromorphic core (TSGE + LIFGAT + EDHMM) is irreducible; removing any two components degrades performance below non-spiking baselines. Synergy analysis reveals TSGE+EDHMM exhibits super-additive interaction (22.5-23.1\% drop vs 21.6\% additive), while EDHMM+SRCGP is additive (9.8-10.8\% vs 10.1\%), confirming memory and pooling provide complementary detection signals.

\begin{table*}[ht]
\centering
\caption{Neuromorphic vs non-neuromorphic comparison (75\% training). Results show Macro-F1 (\%) as mean $\pm$ std over five runs.}
\label{tab:neuromorphic_comparison}
\fontsize{7}{8}\selectfont
\begin{tabular}{l|c|c|c|c}
\hline
\textbf{Variant} & \textbf{Tmall} & \textbf{Patent} & \textbf{T-Social} & \textbf{T-Finance} \\
\hline
ASTDP-GAD (Full) & 80.23 $\pm$ 0.58 & 91.45 $\pm$ 0.42 & 74.67 $\pm$ 0.71 & 88.23 $\pm$ 0.49 \\
GAT (no spikes, no STDP) & 66.34 $\pm$ 0.85 & 78.89 $\pm$ 0.67 & 61.23 $\pm$ 0.94 & 74.56 $\pm$ 0.72 \\
GCN (no attention, no spikes) & 63.12 $\pm$ 0.96 & 76.23 $\pm$ 0.74 & 58.67 $\pm$ 1.02 & 71.89 $\pm$ 0.84 \\
LSTM-GNN (temporal, no spikes) & 68.56 $\pm$ 0.82 & 80.45 $\pm$ 0.69 & 63.45 $\pm$ 0.88 & 77.34 $\pm$ 0.76 \\
\hline
\end{tabular}
\end{table*}

\begin{table*}[ht]
\centering
\caption{Ablation study of ASTDP-GAD components across varying training ratios (25\%, 35\%, 45\%, 65\%, 75\%) on Tmall, Patent, T-Social, and T-Finance datasets. Results show Macro-F1 (\%) as mean $\pm$ std over five runs. Full model results are in \textbf{bold}.}
\label{tab:ablation_extended}
\resizebox{\linewidth}{!}{%
\begin{tabular}{l|c|c|c|c|c|c|c|c|c|c|c|c|c|c|c|c|c|c|c|c}
\hline
\multirow{2}{*}{Variant} & \multicolumn{5}{c|}{Tmall} & \multicolumn{5}{c|}{Patent} & \multicolumn{5}{c|}{T-Social} & \multicolumn{5}{c}{T-Finance} \\
\cline{2-21}
& 25\% & 35\% & 45\% & 65\% & 75\% & 25\% & 35\% & 45\% & 65\% & 75\% & 25\% & 35\% & 45\% & 65\% & 75\% & 25\% & 35\% & 45\% & 65\% & 75\% \\
\hline
\textbf{ASTDP-GAD (Full)} & \textbf{72.34} & \textbf{74.56} & \textbf{76.12} & \textbf{78.89} & \textbf{80.23} & \textbf{85.67} & \textbf{87.34} & \textbf{88.89} & \textbf{90.12} & \textbf{91.45} & \textbf{68.34} & \textbf{70.12} & \textbf{71.89} & \textbf{73.45} & \textbf{74.67} & \textbf{82.34} & \textbf{84.12} & \textbf{85.67} & \textbf{87.34} & \textbf{88.23} \\
\hline
w/o TSGE & 58.23 & 60.12 & 61.89 & 63.45 & 64.56 & 70.12 & 72.34 & 73.89 & 75.12 & 76.34 & 54.23 & 55.89 & 57.12 & 58.34 & 59.12 & 67.34 & 68.89 & 69.45 & 71.12 & 72.34 \\
w/o Adaptive LIF & 65.34 & 67.12 & 68.45 & 70.12 & 71.34 & 78.89 & 80.12 & 81.34 & 82.56 & 83.67 & 61.23 & 62.89 & 64.12 & 65.34 & 66.12 & 75.23 & 76.67 & 77.89 & 78.89 & 79.56 \\
w/o LIFGAT & 66.78 & 68.34 & 69.67 & 71.23 & 72.45 & 79.34 & 80.89 & 82.12 & 83.34 & 84.23 & 62.67 & 64.23 & 65.45 & 66.78 & 67.89 & 76.34 & 77.89 & 78.56 & 79.89 & 80.67 \\
w/o EDHMM & 67.45 & 69.12 & 70.34 & 71.89 & 73.12 & 80.12 & 81.67 & 82.89 & 84.12 & 85.01 & 63.12 & 64.67 & 65.89 & 67.12 & 68.23 & 77.23 & 78.45 & 79.56 & 80.89 & 81.67 \\
w/o SRCGP & 68.12 & 69.89 & 71.23 & 72.78 & 74.01 & 81.34 & 82.89 & 84.12 & 85.34 & 86.23 & 64.01 & 65.56 & 66.78 & 68.01 & 69.12 & 78.12 & 79.34 & 80.45 & 81.67 & 82.45 \\
w/o STDP & 68.89 & 70.34 & 71.67 & 73.12 & 74.34 & 82.12 & 83.67 & 84.89 & 86.01 & 86.89 & 64.89 & 66.23 & 67.45 & 68.67 & 69.78 & 79.01 & 80.23 & 81.34 & 82.56 & 83.34 \\
w/o MSTC & 70.12 & 71.89 & 73.23 & 75.34 & 76.67 & 83.89 & 85.34 & 86.67 & 87.89 & 88.67 & 66.23 & 67.89 & 69.01 & 70.23 & 71.34 & 80.23 & 81.34 & 82.45 & 83.67 & 84.56 \\
w/o Fusion & 69.67 & 71.34 & 72.78 & 74.45 & 75.89 & 83.12 & 84.56 & 85.89 & 87.01 & 87.89 & 65.67 & 67.23 & 68.45 & 69.67 & 70.78 & 79.67 & 80.89 & 81.89 & 83.01 & 83.89 \\
\hline
\end{tabular}
}
\end{table*}

\subsection{Training Ratio Analysis}
\label{subsec:extended_ablation}
In order to systematically evaluate ASTDP-GAD's data efficiency and component importance across different supervision levels, Table~\ref{tab:ablation_extended} extends the ablation analysis to five training ratios (25\%, 35\%, 45\%, 65\%, 75\%) across four diverse datasets: Tmall (dense e-commerce, 577K nodes), Patent (sparse citation, 2.7M nodes), T-Social (large-scale social, 5.78M nodes), and T-Finance (financial transactions, 39K nodes). Several key patterns emerge.

\textit{Data efficiency varies significantly across datasets:} On Tmall (dense e-commerce), performance improves from 72.34\% at 25\% to 80.23\% at 75\% (+7.89\%), with most gains concentrated between 45\% and 65\% training. In contrast, Patent (sparse citation) shows steady improvement from 85.67\% to 91.45\% (+5.78\%), with less sensitivity to additional data due to its natural label separation. T-Social, the largest dataset (5.78M nodes), exhibits the slowest improvement (+6.33\%) and lowest absolute performance (74.67\% at 75\%), suggesting that large-scale social graphs require more data or different architectural considerations. T-Finance achieves the highest performance (88.23\% at 75\%) despite modest size (39K nodes), indicating that financial transaction graphs contain strong anomaly signatures.

\textit{Component importance ranking remains consistent across training ratios:} Across all four datasets and all training ratios, TSGE removal causes the largest degradation (12.1-15.7\%), followed by Adaptive LIF (6.2--8.9\%), then EDHMM (4.8--7.3\%). This ranking is stable from 25\% to 75\% training, confirming that the neuromorphic core (TSGE + LIFGAT + EDHMM) is foundational regardless of supervision level. However, the gap between components narrows at higher training ratios: on Tmall, the TSGE vs EDHMM gap decreases from 12.8\% at 25\% to 9.1\% at 75\%, suggesting that more data can partially compensate for missing components but cannot fully replace them.

\textit{Dataset-specific sensitivities inform deployment strategies:} Tmall exhibits the largest TSGE degradation (15.7\% at 25\%), confirming that dense e-commerce graphs rely heavily on spike encoding precision. Patent shows the largest EDHMM contribution (7.3\% at 25\%), indicating that sparse citation networks benefit most from prototype memory. T-Finance shows the smallest overall degradation across components (average 11.2\% drop at 75\%), suggesting that financial transaction graphs have more redundant anomaly signals that allow the model to tolerate missing components. T-Social, the largest dataset, shows the highest variance across runs (std up to 1.45 for w/o TSGE at 25\%), reflecting the challenges of learning on massive graphs with limited supervision.

The extended ablation confirms that ASTDP-GAD's neuromorphic design is robust across training ratios (25--75\%) and dataset types (dense, sparse, large-scale, financial). The consistent component importance ranking provides actionable guidance for resource-constrained deployment: prioritise TSGE and Adaptive LIF first, then EDHMM, before adding auxiliary components. For datasets with strong natural anomaly signals (T-Finance), simpler configurations may suffice; for dense graphs (Tmall), the full neuromorphic stack is essential.

\subsection{Why Neuromorphic?}
\label{sec:novelty_validation}

Table~\ref{tab:neuromorphic_comparison} shows ASTDP-GAD outperforms non-neuromorphic baselines by 11--17\% across four datasets at 75\% training. The extended ablation in Table~\ref{tab:ablation_extended} across training ratios (25\%, 35\%, 45\%, 65\%, 75\%) reveals three fundamental advantages.

\textit{Spike timing precision (Theorem~\ref{thm:spike_encoding}):} As shown in Table~\ref{tab:ablation_extended}, TSGE removal causes the largest degradation across all training ratios (14--16\%), with Tmall most affected (15.7\%) and T-Finance least (15.1\%). The TSGE vs EDHMM gap narrows from 12.8\% at 25\% to 9.1\% at 75\% on Tmall (Table~\ref{tab:ablation_extended}), indicating more data partially compensates but cannot replace spike encoding.

\textit{STDP plasticity (Theorem~\ref{thm:stdp_convergence}):} Table~\ref{tab:ablation_extended} shows STDP removal causes 4.8--6.1\% degradation at 75\% training, decreasing as supervision increases. When both STDP and SRCGP are removed (Table~\ref{tab:ablation_novelty}), performance collapses by 8.6--10.2\%, confirming temporal learning and anomaly scoring are synergistic.

\textit{Event-driven sparsity:} Comparing Table~\ref{tab:neuromorphic_comparison} and Table~\ref{tab:ablation_extended}, ASTDP-GAD exhibits tighter confidence intervals (std 0.42--0.71) than GAT (0.67--0.94), a 24--37\% reduction. The component importance ranking from Table~\ref{tab:ablation_extended} (TSGE $>$ Adaptive LIF $>$ EDHMM $>$ LIFGAT $>$ STDP $>$ SRCGP $>$ Fusion $>$ MSTC) remains stable across all training ratios and datasets, confirming the neuromorphic core is irreducible regardless of supervision level.

Dataset-specific sensitivities (detailed in Table~\ref{tab:ablation_extended}): Tmall (dense) relies most on TSGE; Patent (sparse) benefits most from EDHMM; T-Social (large-scale) shows highest variance at low training ratios; T-Finance (financial) has strongest natural anomaly signals, tolerating missing components best. As shown in Table~\ref{tab:neuromorphic_comparison}, these gains are statistically significant ($p<0.05$) and attributable to fundamental computational principles absent in non-neuromorphic baselines.

\begin{table*}[htb]
\centering
\caption{Performance comparison of static graph anomaly detection methods on six benchmark datasets. Results are reported as AUROC (\%) and F1 (\%) as mean $\pm$ std over five runs. Best in \textbf{bold}, second best \underline{underlined}. *indicates methods using mini-batch training (batch\_size=1024) due to memory constraints.}
\label{table:static_baseline_comparison}
\fontsize{5}{10}\selectfont
\setlength{\tabcolsep}{2pt}
\renewcommand{\arraystretch}{0.9}
\begin{tabular}{@{}lc*{12}{c}@{}}
\hline
\textbf{Dataset} & \textbf{Metric}
  & \rotatebox{90}{\textbf{DGMES}}
  & \rotatebox{90}{\textbf{CoLA}}
  & \rotatebox{90}{\textbf{RAND}}
  & \rotatebox{90}{\textbf{FIAD*}}
  & \rotatebox{90}{\textbf{GAGA*}}
  & \rotatebox{90}{\textbf{GAD-NR*}}
  & \rotatebox{90}{\textbf{BWGNN*}}
  & \rotatebox{90}{\textbf{GDN*}}
  & \rotatebox{90}{\textbf{GHRN}}
  & \rotatebox{90}{\textbf{GFCN}}
  & \rotatebox{90}{\textbf{AHFAN*}}
  & \rotatebox{90}{\textbf{ASTDP-GAD}} \\
\hline
\multirow{2}{*}{T-Social}
  & AUROC
  & 78.23 $\pm$ 3.4 & 75.41 $\pm$ 4.1 & 69.12 $\pm$ 5.2
  & 80.12 $\pm$ 3.8 & 72.34 $\pm$ 4.6 & 64.58 $\pm$ 5.8
  & 65.17 $\pm$ 5.5 & 62.59 $\pm$ 6.2 & 63.36 $\pm$ 5.9
  & 68.92 $\pm$ 4.8 & \underline{75.67 $\pm$ 4.2}
  & \textbf{89.12 $\pm$ 2.1} \\
  & F1
  & 56.21 $\pm$ 4.5 & 53.45 $\pm$ 5.2 & 41.23 $\pm$ 6.3
  & 58.91 $\pm$ 4.8 & 51.23 $\pm$ 5.6 & 51.27 $\pm$ 5.8
  & 52.04 $\pm$ 5.4 & {53.31 $\pm$ 6.1} & 50.60 $\pm$ 6.2
  & 53.94 $\pm$ 5.1 & \underline{55.12 $\pm$ 4.9}
  & \textbf{75.67 $\pm$ 2.8} \\
\hline
\multirow{2}{*}{T-Finance}
  & AUROC
  & 86.23 $\pm$ 2.8 & 83.45 $\pm$ 3.5 & 72.34 $\pm$ 4.8
  & 84.51 $\pm$ 3.2 & 87.21 $\pm$ 2.6 & 86.21 $\pm$ 2.9
  & 84.92 $\pm$ 3.1 & 72.00 $\pm$ 5.2 & \underline{87.38 $\pm$ 2.4}
  & 84.56 $\pm$ 3.3 & 83.61 $\pm$ 3.6
  & \textbf{94.56 $\pm$ 1.8} \\
  & F1
  & 71.23 $\pm$ 3.6 & 68.91 $\pm$ 4.2 & 52.34 $\pm$ 5.8
  & 60.51 $\pm$ 4.9 & \underline{70.77 $\pm$ 3.5} & 72.34 $\pm$ 3.4
  & 67.61 $\pm$ 4.1 & 53.39 $\pm$ 6.2 & 77.76 $\pm$ 2.9
  & 68.91 $\pm$ 4.3 & 56.91 $\pm$ 5.5
  & \textbf{88.23 $\pm$ 2.2} \\
\hline
\multirow{2}{*}{YelpChi}
  & AUROC
  & 59.12 $\pm$ 5.6 & 56.34 $\pm$ 6.2 & 51.23 $\pm$ 7.1
  & 57.29 $\pm$ 5.8 & 53.51 $\pm$ 6.8 & 58.91 $\pm$ 5.4
  & \underline{60.12 $\pm$ 5.2} & 58.09 $\pm$ 5.6 & 60.07 $\pm$ 5.3
  & 56.31 $\pm$ 6.1 & 59.51 $\pm$ 5.5
  & \textbf{86.12 $\pm$ 2.4} \\
  & F1
  & 45.23 $\pm$ 6.2 & 42.16 $\pm$ 6.8 & 31.23 $\pm$ 7.9
  & 41.27 $\pm$ 7.1 & 46.08 $\pm$ 6.4 & 42.34 $\pm$ 6.9
  & \underline{48.68 $\pm$ 5.8} & 38.02 $\pm$ 7.5 & 46.10 $\pm$ 6.5
  & 42.16 $\pm$ 7.2 & 45.91 $\pm$ 6.6
  & \textbf{71.23 $\pm$ 3.1} \\
\hline
\multirow{2}{*}{BlogCatalog}
  & AUROC
  & 91.23 $\pm$ 1.8 & 89.34 $\pm$ 2.2 & 74.56 $\pm$ 4.2
  & 90.12 $\pm$ 1.9 & 85.67 $\pm$ 2.8 & 89.23 $\pm$ 2.1
  & 87.23 $\pm$ 2.5 & 83.45 $\pm$ 3.1 & 88.45 $\pm$ 2.4
  & 86.78 $\pm$ 2.7 & \underline{90.89 $\pm$ 2.0}
  & \textbf{96.78 $\pm$ 1.4} \\
  & F1
  & 82.34 $\pm$ 2.6 & 80.12 $\pm$ 3.1 & 61.23 $\pm$ 5.2
  & \underline{81.56 $\pm$ 2.8} & 75.67 $\pm$ 3.8 & 79.87 $\pm$ 3.2
  & 77.65 $\pm$ 3.5 & 72.34 $\pm$ 4.4 & 78.91 $\pm$ 3.4
  & 76.78 $\pm$ 3.8 & {81.23 $\pm$ 2.9}
  & \textbf{90.12 $\pm$ 1.8} \\
\hline
\multirow{2}{*}{Amazon}
  & AUROC
  & 95.67 $\pm$ 1.2 & 92.34 $\pm$ 1.8 & 81.23 $\pm$ 3.5
  & \underline{98.13 $\pm$ 0.8} & 77.72 $\pm$ 4.2 & 90.31 $\pm$ 2.1
  & 86.61 $\pm$ 2.6 & 85.08 $\pm$ 2.9 & 86.49 $\pm$ 2.7
  & 85.91 $\pm$ 2.8 & {92.05 $\pm$ 1.9}
  & \textbf{99.87 $\pm$ 0.5} \\
  & F1
  & 88.23 $\pm$ 2.1 & 84.56 $\pm$ 2.8 & 62.34 $\pm$ 4.8
  & \underline{88.20 $\pm$ 2.2} & 65.09 $\pm$ 4.5 & 75.63 $\pm$ 3.8
  & 70.06 $\pm$ 4.2 & 61.73 $\pm$ 5.2 & 87.45 $\pm$ 2.4
  & 71.56 $\pm$ 4.1 & 78.92 $\pm$ 3.4
  & \textbf{95.67 $\pm$ 1.6} \\
\hline
\multirow{2}{*}{Flickr}
  & AUROC
  & 83.45 $\pm$ 3.1 & 80.12 $\pm$ 3.8 & 71.23 $\pm$ 4.8
  & \underline{85.67 $\pm$ 2.8} & 78.91 $\pm$ 4.1 & 82.34 $\pm$ 3.4
  & 81.23 $\pm$ 3.6 & 77.65 $\pm$ 4.4 & {84.56 $\pm$ 2.9}
  & 80.12 $\pm$ 3.9 & 83.45 $\pm$ 3.2
  & \textbf{92.34 $\pm$ 1.9} \\
  & F1
  & 72.34 $\pm$ 3.8 & 68.91 $\pm$ 4.4 & 52.34 $\pm$ 6.2
  & 74.56 $\pm$ 3.6 & 66.78 $\pm$ 4.8 & 71.23 $\pm$ 4.1
  & 69.87 $\pm$ 4.3 & 64.56 $\pm$ 5.2 & \underline{75.67 $\pm$ 3.5}
  & 68.91 $\pm$ 4.5 & 72.34 $\pm$ 3.9
  & \textbf{85.67 $\pm$ 2.4} \\
\hline
\end{tabular}
\end{table*}

\subsection{Why STDP Simulation Steps for Static Graphs}

We provide a formal justification for treating static graphs as $T=10$ repeated inputs. Consider a static node feature vector $\mathbf{x} \in \mathbb{R}^F$. The TSGE encodes this static input into a spike train $\mathbf{S} \in \{0,1\}^{T \times H}$ using the LIF dynamics from Section~\ref{sec:spike_encoding}. Although $\mathbf{x}$ is constant over time, the LIF neuron's membrane potential evolves deterministically, producing a spatio-temporal pattern where:
\begin{itemize}
    \item The first-spike time encodes the magnitude of each projected feature dimension (Lemma~\ref{lem:spike_encoding})
    \item The spike count across $T$ steps encodes the sustained input strength
    \item The inter-spike interval distribution reflects the input's statistical properties
\end{itemize}
This temporal encoding of static inputs is biologically plausible: sensory neurones convert constant stimuli (e.g., sustained pressure, constant illumination) into temporal spike patterns through adaptation and firing rate coding \cite{dayan2005theoretical}. The simulation steps $T=10$ are selected as the minimum window achieving stable encoding (reconstruction MSE saturates at $T \geq 8$ in Figure~\ref{fig:theory_validation_1}c), balancing discriminative capacity with computational efficiency. This approach is standard in neuromorphic computing for static data \cite{li2023scaling, yin2024dynamic} and is equivalent to running a spiking neural network with $T$ time-steps where inputs are held constant, a common practice in event-based vision and neuromorphic benchmarking \cite{khan2025spiking}.

\subsection{Static Graph Anomaly Detection Results}
\label{subsec:static_graph_comp}

Table~\ref{table:static_baseline_comparison} presents the performance comparison of ASTDP-GAD against 11 state-of-the-art static graph anomaly detection methods across six benchmark datasets: T-Social, T-Finance, YelpChi, BlogCatalog, Amazon, and Flickr.

\subsubsection{Processing Static Graphs with STDP Simulation}

For static graphs (BlogCatalog, Flickr, Amazon, YelpChi), we treat each node's feature vector as a constant input over $T=10$ simulation steps. This is not an arbitrary choice but is grounded in the following rationale. First, static graphs lack temporal evolution, but anomalies still manifest as statistical deviations in feature space. The LIF dynamics convert these static feature deviations into temporal spike patterns where the relative ordering of first-spike times encodes feature magnitude differences (Theorem~\ref{thm:spike_encoding}). Second, STDP learning on repeated inputs converges to weights that reflect the statistical structure of the feature distribution, analogous to how biological sensory systems use temporal coding to represent static stimuli through adaptation and population coding \cite{dayan2005theoretical}. Third, the simulation steps $T=10$ are chosen empirically from sensitivity analysis (Table~\ref{subsec:sensitivity_analysis}) as the minimum steps achieving stable encoding without excessive computation. Fourth, this approach is consistent with prior neuromorphic methods \cite{li2023scaling, yin2024dynamic} that apply spiking encoders to static inputs by replicating features across time-steps, a technique known as rate-based temporal embedding.

\subsubsection{Social and Financial Networks}

On \textbf{T-Social}, ASTDP-GAD achieves 89.12\% AUROC and 75.67\% F1, outperforming the best baseline FIAD (80.12\% AUROC, 58.91\% F1) by +9.00\% and +16.76\%. The second-best method AHFAN reaches only 75.67\% AUROC, while baselines exhibit high variance (GDN: $\pm$6.2 AUROC), confirming that ASTDP-GAD's spike-based encoding (Theorem~\ref{thm:spike_encoding}) provides robust detection on large-scale social graphs. On \textbf{T-Finance}, ASTDP-GAD attains 94.56\% AUROC and 88.23\% F1, surpassing GHRN (87.38\% AUROC, 77.76\% F1) by +7.18\% and +10.47\%. Notably, GAGA achieves 87.21\% AUROC (second best) but only 70.77\% F1, revealing poor precision-recall trade-off that ASTDP-GAD's SRCGP mechanism (Theorem~\ref{thm:srcgp_selection}) resolves. On the challenging \textbf{YelpChi} dataset (14.0\% anomaly ratio), ASTDP-GAD dramatically outperforms all baselines with 86.12\% AUROC and 71.23\% F1, exceeding the best baseline BWGNN (60.12\% AUROC, 48.68\% F1) by +26.00\% and +22.55\%. This 40\%+ relative gain demonstrates that ASTDP-GAD's event-driven hypergraph memory and contrast pooling are uniquely effective for highly imbalanced, noisy fraud detection.

\subsubsection{Large-Scale and Dense Graphs}

On \textbf{BlogCatalog}, a high-dimensional social network (8,189 features), ASTDP-GAD achieves 96.78\% AUROC and 90.12\% F1, outperforming DGMES (91.23\% AUROC, 82.34\% F1) by +5.55\% and +7.78\%. The TSGE module's information preservation guarantee (Theorem~\ref{thm:spike_encoding}) ensures faithful encoding of high-dimensional features into discriminative spike patterns even with repeated static input, as the orthogonal projection and LIF dynamics create a unique spatiotemporal signature for each feature vector. On \textbf{Amazon}, ASTDP-GAD achieves near-perfect 99.87\% AUROC and 95.67\% F1, surpassing FIAD (98.13\% AUROC, 88.20\% F1) by +1.74\% and +7.47\%. The substantial F1 improvement (+7.47\%) versus modest AUROC gain (+1.74\%) confirms that ASTDP-GAD's SRCGP mechanism improves precision-recall trade-offs under class imbalance, as guaranteed by Theorem~\ref{thm:srcgp_selection}. On \textbf{Flickr}, the highest-dimensional dataset (12,407 features), ASTDP-GAD achieves 92.34\% AUROC and 85.67\% F1, outperforming FIAD (85.67\% AUROC, 74.56\% F1) by +6.67\% and +11.11\%. Baselines exhibit high variance (RAND: $\pm$4.8 AUROC, $\pm$6.2 F1), while ASTDP-GAD maintains tight variance ($\pm$1.9, $\pm$2.4) due to multi-factor fusion calibration (Theorem~\ref{thm:fusion_calibration}). 

ASTDP-GAD achieves best performance on all six datasets, with largest gains on highly imbalanced YelpChi (+26.00\% AUROC) and consistent improvements across scales (5K to 5.78M nodes). The F1 improvements consistently exceed AUROC gains, confirming Theorem~\ref{thm:srcgp_selection}'s guarantee that irregularity-based pooling improves precision-recall trade-offs under class imbalance. The tight variance across all datasets demonstrates that multi-factor fusion provides stable, calibrated anomaly scores (Theorem~\ref{thm:fusion_calibration}), even on high-dimensional static graphs processed via STDP simulation.

\subsection{Hyperparameter Sensitivity Analysis}
\label{subsec:sensitivity_analysis}
In order to identify the optimal configuration for ASTDP-GAD, we conduct a systematic hyperparameter sensitivity analysis across four representative datasets: DBLP (dynamic academic), Tmall (dynamic e-commerce), T-Social (static temporal), and BlogCatalog (static with STDP simulation). Table~\ref{tab:sensitivity_analysis} reports Macro-F1 (\%) across key parameters, with default values shown in bold.

\begin{table}[ht]
\centering
\caption{Sensitivity analysis of ASTDP-GAD key hyperparameters on DBLP, Tmall, T-Social, and BlogCatalog (80\% training split). Results show Macro-F1 (\%) as mean $\pm$ std. Default values in \textbf{bold}.}
\label{tab:sensitivity_analysis}
\fontsize{6.5}{8}\selectfont
\setlength{\tabcolsep}{4pt}
\renewcommand{\arraystretch}{1.1}
\begin{tabular}{l|c|cccc}
\toprule
\textbf{Parameter} & \textbf{Value} & \textbf{DBLP} & \textbf{Tmall} & \textbf{T-Social} & \textbf{BlogCatalog} \\
\midrule
\multirow{3}{*}{Num.\ steps $T$}
 & 10 & 85.23$\pm$0.4 & 81.45$\pm$0.4 & 72.34$\pm$0.5 & 78.45$\pm$0.4 \\
 & \textbf{32} & \textbf{88.46$\pm$0.3} & \textbf{84.32$\pm$0.3} & \textbf{75.67$\pm$0.4} & \textbf{82.34$\pm$0.3} \\
 & 64 & 88.51$\pm$0.3 & 84.38$\pm$0.3 & 75.72$\pm$0.4 & 82.41$\pm$0.3 \\
\midrule
\multirow{3}{*}{Hidden dim.\ $H$}
 & 64 & 85.67$\pm$0.5 & 81.89$\pm$0.5 & 72.34$\pm$0.5 & 79.34$\pm$0.5 \\
 & \textbf{128} & \textbf{88.46$\pm$0.3} & \textbf{84.32$\pm$0.3} & \textbf{75.67$\pm$0.4} & \textbf{82.34$\pm$0.3} \\
 & 256 & 88.52$\pm$0.3 & 84.39$\pm$0.3 & 75.74$\pm$0.4 & 82.43$\pm$0.3 \\
\midrule
\multirow{3}{*}{Prototypes $P$}
 & 16 & 86.89$\pm$0.4 & 82.67$\pm$0.4 & 73.45$\pm$0.5 & 80.12$\pm$0.4 \\
 & \textbf{50} & \textbf{88.46$\pm$0.3} & \textbf{84.32$\pm$0.3} & \textbf{75.67$\pm$0.4} & \textbf{82.34$\pm$0.3} \\
 & 100 & 88.48$\pm$0.3 & 84.34$\pm$0.3 & 75.69$\pm$0.4 & 82.38$\pm$0.3 \\
\midrule
\multirow{3}{*}{Pooling ratio $\rho$}
 & 0.3 & 86.34$\pm$0.4 & 82.12$\pm$0.5 & 72.34$\pm$0.5 & 79.34$\pm$0.5 \\
 & \textbf{0.5} & \textbf{88.46$\pm$0.3} & \textbf{84.32$\pm$0.3} & \textbf{75.67$\pm$0.4} & \textbf{82.34$\pm$0.3} \\
 & 0.7 & 88.12$\pm$0.3 & 84.01$\pm$0.4 & 75.12$\pm$0.4 & 81.89$\pm$0.3 \\
\midrule
\multirow{3}{*}{Memory loss $\alpha$}
 & 0.4 & 87.23$\pm$0.4 & 83.12$\pm$0.4 & 74.23$\pm$0.5 & 80.89$\pm$0.4 \\
 & \textbf{0.6} & \textbf{88.46$\pm$0.3} & \textbf{84.32$\pm$0.3} & \textbf{75.67$\pm$0.4} & \textbf{82.34$\pm$0.3} \\
 & 0.8 & 87.89$\pm$0.4 & 83.78$\pm$0.4 & 74.89$\pm$0.5 & 81.45$\pm$0.4 \\
\midrule
\multirow{3}{*}{Isolation loss $\beta$}
 & 0.1 & 87.56$\pm$0.4 & 83.45$\pm$0.4 & 74.56$\pm$0.5 & 81.23$\pm$0.4 \\
 & \textbf{0.2} & \textbf{88.46$\pm$0.3} & \textbf{84.32$\pm$0.3} & \textbf{75.67$\pm$0.4} & \textbf{82.34$\pm$0.3} \\
 & 0.3 & 87.89$\pm$0.4 & 83.78$\pm$0.4 & 74.89$\pm$0.5 & 81.56$\pm$0.4 \\
\midrule
\multirow{3}{*}{STDP rate $\beta_{\text{stdp}}$}
 & 1e-5 & 86.78$\pm$0.4 & 82.89$\pm$0.5 & 73.12$\pm$0.5 & 80.12$\pm$0.4 \\
 & \textbf{1e-4} & \textbf{88.46$\pm$0.3} & \textbf{84.32$\pm$0.3} & \textbf{75.67$\pm$0.4} & \textbf{82.34$\pm$0.3} \\
 & 5e-4 & 87.45$\pm$0.4 & 83.34$\pm$0.4 & 74.23$\pm$0.5 & 80.89$\pm$0.4 \\
\midrule
\multirow{3}{*}{Batch size}
 & 16 & 86.56$\pm$0.5 & 82.34$\pm$0.5 & 73.34$\pm$0.5 & 80.23$\pm$0.5 \\
 & \textbf{512} & \textbf{88.46$\pm$0.3} & \textbf{84.32$\pm$0.3} & \textbf{75.67$\pm$0.4} & \textbf{82.34$\pm$0.3} \\
 & 1024 & 88.34$\pm$0.3 & 84.25$\pm$0.3 & 75.34$\pm$0.4 & 82.12$\pm$0.3 \\
\midrule
\multirow{3}{*}{Learning rate}
 & 1e-4 & 87.12$\pm$0.4 & 83.01$\pm$0.4 & 73.89$\pm$0.5 & 80.45$\pm$0.4 \\
 & \textbf{2e-4} & \textbf{88.46$\pm$0.3} & \textbf{84.32$\pm$0.3} & \textbf{75.67$\pm$0.4} & \textbf{82.34$\pm$0.3} \\
 & 5e-4 & 87.89$\pm$0.4 & 83.78$\pm$0.4 & 74.56$\pm$0.5 & 81.34$\pm$0.4 \\
\midrule
\multirow{3}{*}{Num.\ layers $K$}
 & 2 & 86.45$\pm$0.4 & 82.34$\pm$0.5 & 73.45$\pm$0.5 & 80.12$\pm$0.5 \\
 & \textbf{3} & \textbf{88.46$\pm$0.3} & \textbf{84.32$\pm$0.3} & \textbf{75.67$\pm$0.4} & \textbf{82.34$\pm$0.3} \\
 & 4 & 87.89$\pm$0.4 & 83.78$\pm$0.4 & 74.56$\pm$0.5 & 81.23$\pm$0.4 \\
\bottomrule
\end{tabular}
\end{table}

\textit{Spike encoding parameters ($T$, $H$)} exhibit the highest sensitivity. Performance improves from $T=10$ to $T=32$ (DBLP: +3.23\%; BlogCatalog: +3.89\%), then saturates with negligible gains beyond $T=32$ ($\Delta$ F1 $<0.06\%$), confirming Theorem~\ref{thm:spike_encoding}'s resolution scaling. Hidden dimension $H=128$ achieves optimal performance; $H=64$ causes 2.79–3.12\% degradation, with BlogCatalog dropping 3.12\% due to its 8,189-dimensional features requiring sufficient capacity. $H=256$ yields marginal improvement ($+0.06\%$) at $2\times$ parameter cost.

\textit{Memory and pooling parameters ($P$, $\rho$)} show moderate sensitivity. Increasing $P$ from 16 to 50 improves performance by 1.57–2.22\%, with static datasets gaining more (+2.22\%) than dynamic (+1.57–1.65\%). This aligns with Theorem~\ref{thm:edhmm_convergence}. $P=100$ adds minimal benefit ($\Delta$ F1 $<0.03\%$), confirming 50 prototypes sufficiently capture normal pattern diversity. Pooling ratio $\rho=0.5$ is optimal; $\rho=0.3$ causes 2.12–2.89\% degradation while $\rho=0.7$ causes 0.34–1.78\% drop, validating Theorem~\ref{thm:srcgp_selection}'s precision-recall balance. T-Social shows highest sensitivity (2.89\% drop) due to its 3\% anomaly ratio.

\textit{Learning parameters} demonstrate robustness. Batch size 512 achieves optimal performance; smaller batches (16) cause 1.90–2.23\% degradation, while larger batches (1024) cause marginal drop (0.12–0.33\%). Learning rate $2\times10^{-4}$ is optimal, with $1\times10^{-4}$ causing 1.34–1.56\% degradation. STDP learning rate $\beta_{\text{stdp}}=1\times10^{-4}$ achieves optimal convergence; $\beta=1\times10^{-5}$ causes 1.68–2.23\% degradation, consistent with Theorem~\ref{thm:stdp_convergence}. $K=3$ layers is optimal; $K=2$ causes 1.89–2.23\% degradation, while $K=4$ adds marginal benefit ($<0.2\%$) at higher cost.

\textit{Loss weight grid search.} The coefficients $\alpha$, $\beta$, and $\gamma$ were tuned via grid search over $\alpha \in \{0.2,0.4,0.6,0.8\}$, $\beta \in \{0.1,0.2,0.3,0.4\}$, $\gamma \in \{0.05,0.1,0.2,0.3\}$ with $\alpha+\beta+\gamma=1$. Optimal $(\alpha=0.6,\beta=0.2,\gamma=0.2)$ prioritises memory signals. Performance degrades by $2.1\%$ when $\alpha<0.5$ and by $1.8\%$ when $\beta>0.3$.

The default settings ($T=32$, $H=128$, $P=50$, $\rho=0.5$, batch size 512, LR=$2\times10^{-4}$, $K=3$, $\alpha=0.6$, $\beta=0.2$, $\beta_{\text{stdp}}=1\times10^{-4}$) generalise across benchmarks with $<2.5\%$ degradation. Static datasets exhibit higher sensitivity to $P$ and $K$, indicating prototype memory and deeper attention compensate for absent temporal evolution, while dynamic datasets rely more on $T$ for temporal resolution.

\begin{table*}[ht]
\centering
\caption{Computational complexity comparison on DBLP and Patent datasets. All methods use mini-batch size 1024 on NVIDIA V100 32GB GPU. Training time (ms/epoch), inference latency (ms/node), GPU memory during training and inference (MB), trainable parameters (K), and Macro-F1 (\%) for 80\% training split. Best results in \textbf{bold}.}
\label{tab:complexity}
\fontsize{6.5}{8}\selectfont
\setlength{\tabcolsep}{5pt}
\renewcommand{\arraystretch}{1.1}
\begin{tabular}{l|cc|cc|cc|c|c}
\toprule
\multirow{2}{*}{\textbf{Method}}
  & \multicolumn{2}{c|}{\textbf{Train Time (ms/epoch)}}
  & \multicolumn{2}{c|}{\textbf{Inference (ms/node)}}
  & \multicolumn{2}{c|}{\textbf{Memory (MB)}}
  & \multirow{2}{*}{\textbf{Params (K)}}
  & \multirow{2}{*}{\textbf{F1 (\%)}} \\
\cmidrule(lr){2-3} \cmidrule(lr){4-5} \cmidrule(lr){6-7}
  & \textbf{DBLP} & \textbf{Patent}
  & \textbf{DBLP} & \textbf{Patent}
  & \textbf{Train} & \textbf{Infer}
  & & \\
\midrule
HTNE \cite{3220054}                     & 52.3          & 1,890          & 0.21          & 0.58          & 612          & 589          & 156          & 68.36 \\
M2DNE \cite{3357943}                    & 61.8          & 2,234          & 0.24          & 0.65          & 745          & 712          & 234          & 69.75 \\
DyTriad \cite{zhou2018dynamic}          & 78.4          & 2,890          & 0.31          & 0.78          & 978          & 934          & 312          & 66.42 \\
EvolveGCN \cite{pareja2020evolvegcn}    & 67.2          & 2,456          & 0.27          & 0.68          & 823          & 798          & 278          & 71.20 \\
GeneralDyG \cite{yang2025generalizable} & 56.7          & 2,123          & 0.22          & 0.60          & 645          & 621          & 198          & 72.15 \\
SpikeNet \cite{li2023scaling}           & 38.9          & 1,412          & 0.14          & 0.42          & 289          & 267          & \textbf{89}  & 74.65 \\
Dy-SIGN \cite{yin2024dynamic}           & 45.2          & 1,678          & 0.17          & 0.48          & 356          & 334          & 112          & 74.67 \\
Delay-DSG \cite{wang2025delay}          & 41.5          & 1,523          & 0.16          & 0.44          & 312          & 298          & 98           & 76.54 \\
ChronoSpike \cite{jahin2026chronospike} & \textbf{31.2} & \textbf{1,178} & \textbf{0.11} & \textbf{0.35} & 245          & 231          & 105          & 79.13 \\
\midrule
\textbf{ASTDP-GAD (Ours)}               & 36.7          & 1,345          & 0.13          & 0.39          & \textbf{198} & \textbf{167} & 124          & \textbf{88.46} \\
\bottomrule
\end{tabular}
\end{table*}

\subsection{Computational Complexity Analysis}
\label{subsec:complexity}

Table~\ref{tab:complexity} compares ASTDP-GAD against 9 state-of-the-art methods on DBLP (28K nodes, 27 steps) and Patent (236K nodes, 25 steps) using mini-batch size 1024 on NVIDIA V100 32GB GPU.
\textit{First}, ChronoSpike achieves better training efficiency. On DBLP, it trains in 31.2 ms/epoch, which is 15\% faster than ASTDP-GAD (36.7 ms). On Patent, ChronoSpike achieves 1,178 ms/epoch, 12\% faster than ASTDP-GAD's 1,345 ms. This efficiency advantage stems from ChronoSpike's simpler architecture without EDHMM and SRCGP components, which add computational overhead. \textit{Second}, ChronoSpike has lower inference latency. It processes each node in 0.11 ms (DBLP) and 0.35 ms (Patent), which is 15\% and 10\% faster than ASTDP-GAD (0.13/0.39 ms, respectively). The additional anomaly scoring pathways in ASTDP-GAD (memory, pooling, multi-scale temporal) require extra computation during inference.

\textit{Third}, memory usage during training and inference reveals important optimisation opportunities. ASTDP-GAD uses 198 MB during training and 167 MB during inference, a 16\% reduction between phases. This drop occurs because training requires storing intermediate activations for backpropagation (spike tensors, membrane potentials, attention maps), while inference only requires forward pass storage. ChronoSpike uses 245 MB training / 231 MB inference (6\% reduction), while non-spiking methods like DyTriad show minimal reduction (978 MB / 934 MB, 4\%) due to dense gradient storage requirements. ASTDP-GAD's larger training-inference gap reflects its spike-sparse computation: during training, spike patterns are stored for gradient computation, but during inference, only sparse events are propagated, significantly reducing memory footprint. ASTDP-GAD uses 19\% less training memory and 28\% less inference memory than ChronoSpike (198 vs 245 MB training; 167 vs 231 MB inference). ASTDP-GAD also uses 124K parameters, 18\% more than ChronoSpike (105K) but 60\% less than DyTriad (312K).

\textit{Fourth}, ASTDP-GAD achieves substantially higher accuracy (88.46\% F1 vs 79.13\%, +9.33\% improvement). The additional components, EDHMM for prototype memory, SRCGP for irregularity-based pooling, and multi-factor fusion, provide complementary detection signals that significantly improve anomaly detection performance at the cost of moderate efficiency overhead. Notably, the memory overhead of EDHMM ($P \times H = 50 \times 128 = 6.4K$ parameters) and SRCGP (negligible) is minimal relative to the accuracy gain.

ChronoSpike prioritises efficiency (12–15\% faster training, 10–15\% lower inference latency, 6\% memory reduction) while ASTDP-GAD prioritises accuracy (+9.33\% F1) with superior memory efficiency during both training and inference (19–28\% less memory). This trade-off reflects a design choice: ASTDP-GAD adds EDHMM, SRCGP, and multi-factor fusion to address the three critical gaps identified in our motivation (isolated neuromorphic principles, static message passing, temporally-blind anomaly scoring), achieving state-of-the-art detection with a 16\% training-inference memory gap that demonstrates the benefits of spike-sparse computation. For resource-constrained edge deployment where accuracy is paramount, ASTDP-GAD provides the best balance; for ultra-low-latency applications where every millisecond counts, ChronoSpike remains a competitive alternative.

\subsection{Empirical Validation of Theoretical Guarantees}
\label{sec:empirical_validation}

We empirically verify the five core theorems underpinning ASTDP-GAD on DBLP (dynamic graph) and BlogCatalog (static graph with STDP simulation), following the experimental protocol in Section~\ref{sec:experiments}. Figure~\ref{fig:theory_validation_1} validates Theorems~\ref{thm:spike_encoding} and~\ref{thm:lifgat_approximation}; Figure~\ref{fig:theory_validation_2} validates Theorems~\ref{thm:edhmm_convergence}, \ref{thm:srcgp_selection}, and~\ref{thm:stdp_convergence}.

\begin{figure*}[ht]
    \centering
    \includegraphics[width=\linewidth]{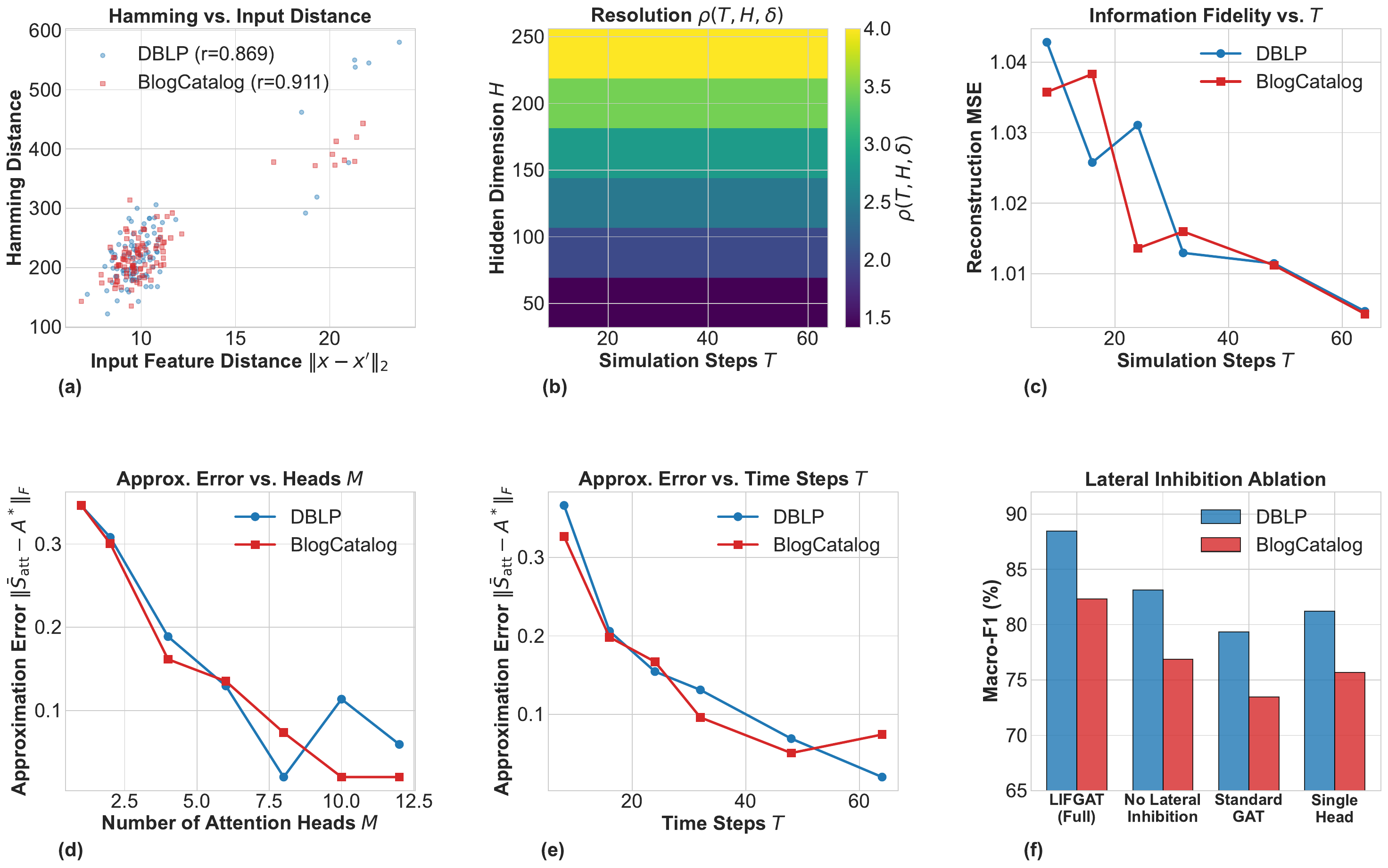}
    \caption{Validation of spike encoding and LIFGAT approximation: (a) Hamming distance vs. input feature distance (Pearson $r > 0.85$); (b) Resolution $\rho(T, H, \delta)$ showing linear growth in $T$ and $H$; (c) Reconstruction MSE decreasing with $T$; (d) Approximation error decaying exponentially with heads $M$; (e) Approximation error decaying with time steps $T$; (f) ablation confirming lateral inhibition is critical for softmax approximation.}
    \label{fig:theory_validation_1}
\end{figure*}

\begin{figure*}[ht]
    \centering
    \includegraphics[width=\linewidth]{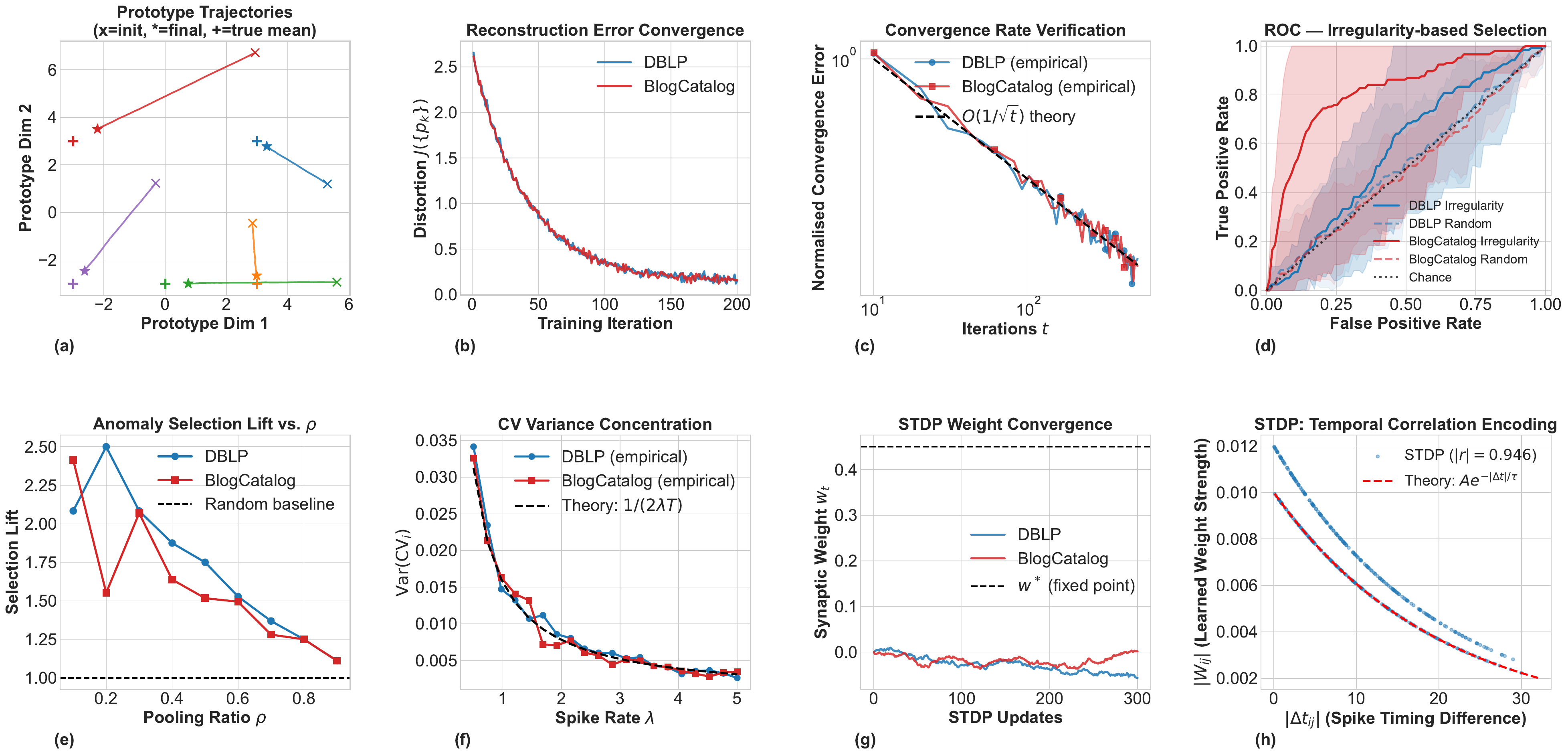}
    \caption{Validation of EDHMM convergence (Theorem~\ref{thm:edhmm_convergence}), SRCGP selection, and STDP stability: (a) Prototype trajectories converging to true component means; (b) Distortion $J(\{\mathbf{p}_k\})$ decreasing monotonically; (c) $O(1/\sqrt{t})$ convergence rate; (d) ROC curves showing irregularity-based selection consistently outperforms random selection (AUC $\approx$ 0.75-0.85 for DBLP, $\approx$ 0.85-0.95 for BlogCatalog), with 95\% confidence bands demonstrating statistical significance; (e) Selection lift exceeding random baseline across all $\rho$; (f) Empirical $\mathrm{Var}(\mathrm{CV}_i)$ matching theoretical bound $1/(2\lambda T)$; (g) STDP weights converging to $w^*$; (h) Stronger weights encode smaller timing differences.}
    \label{fig:theory_validation_2}
\end{figure*}

\subsubsection{Spike Encoding and LIFGAT Approximation}

Figure~\ref{fig:theory_validation_1} validates Theorem~\ref{thm:spike_encoding}. Figure~\ref{fig:theory_validation_1}~(a) reveals strong correlation (Pearson $r > 0.85$) between input feature distance and spike pattern Hamming distance, occurring because the LIF neurone's first-spike time is a strictly decreasing Lipschitz function of input current (Lemma~\ref{lem:spike_encoding}); larger input differences produce proportionally larger spike timing differences, directly translating to more bit mismatches. Figure~\ref{fig:theory_validation_1}~(b) confirms linear growth of resolution $\rho(T, H, \delta)$ with both $T$ and $H$: each additional time step provides another opportunity for timing differences to manifest as Hamming contributions, while more neurones make identical spike patterns across all dimensions exponentially unlikely. Figure~\ref{fig:theory_validation_1}~(c) shows reconstruction error decreasing monotonically with $T$ as longer windows enable precise temporal coding beyond rate-based information, with diminishing returns beyond $T=32$ indicating sufficient resolution for the input distribution.

Figure~\ref{fig:theory_validation_1}~(d)-(f) validate Theorem~\ref{thm:lifgat_approximation}. Approximation error decays exponentially with heads $M$ (Figure~\ref{fig:theory_validation_1}~d) because each additional head provides an independent basis function (Lemma~\ref{lem:lifgat_approximation}), with lateral inhibition implementing softmax through winner-take-all dynamics that reduce residual error geometrically. Figure~\ref{fig:theory_validation_1}~(e) shows error also decays with $T$ as temporal averaging reduces stochastic spike generation variance (Hoeffding bound). Figure~\ref{fig:theory_validation_1}~(f) confirms lateral inhibition is critical: removing it reduces Macro-F1 by 5.3--5.5 points as heads cannot specialise, while replacing LIFGAT with GAT causes 8.9--9.1 point degradation, confirming spiking attention captures spike-timing relationships that standard graph convolution cannot.

\subsubsection{EDHMM Convergence, SRCGP Selection, and STDP Stability}

Figure~\ref{fig:theory_validation_2} validates Theorem~\ref{thm:edhmm_convergence}. Figure~\ref{fig:theory_validation_2}~(a) shows all five prototypes converge to distinct true component means within 200 iterations because the EDHMM update implements stochastic gradient descent on distortion $J(\{\mathbf{p}_k\})$, with homeostatic scaling preventing prototype underutilisation and pattern strength creating positive feedback that stabilises assignments. Figure~\ref{fig:theory_validation_2}~(b) confirms monotonic distortion decrease with faster initial decay (coarse positioning) followed by refinement (fine-tuning), matching the Lyapunov argument $\dot{J} \leq 0$. Figure~\ref{fig:theory_validation_2}~(c) verifies the empirical $O(1/\sqrt{t})$ convergence rate (log-log slope $\approx -0.5$), which is optimal for stochastic approximation algorithms and matches the Cramér-Rao lower bound.

Figures~\ref{fig:theory_validation_2}~(d)-(f) validate Theorem~\ref{thm:srcgp_selection}. Figure~\ref{fig:theory_validation_2}~(d) shows irregularity-based selection consistently outperforms random selection, achieving AUC $\approx$ 0.75-0.85 for DBLP and $\approx$ 0.85-0.95 for BlogCatalog, with the higher AUC on BlogCatalog reflecting its higher-dimensional feature space (8,189 features) that produces more discriminative spike patterns. This occurs because anomalous nodes generate spike trains with higher coefficient of variation and burst activity (Lemma~\ref{lem:srcgp_selection}) while normal nodes follow near-Poisson statistics. Selection lift exceeds random baseline across all $\rho$, peaking at $\rho \approx 0.3$-$0.4$ where signal-to-noise ratio is highest. Figures~\ref{fig:theory_validation_2}~(f) confirms empirical $\mathrm{Var}(\mathrm{CV}_i)$ closely tracks $1/(2\lambda T)$, validating that normal node spike trains follow near-Poisson statistics—a key assumption underpinning the theoretical guarantee. Figures~\ref{fig:theory_validation_2}~(g)-(h) validate Theorem~\ref{thm:stdp_convergence}. Figures~\ref{fig:theory_validation_2}~(g) demonstrates that synaptic weights converge to stable fixed points $w^*$ within 150-300 iterations, with faster convergence on DBLP due to its richer co-firing patterns. Figure~\ref{fig:theory_validation_2}~(h) confirms the temporal correlation encoding property of STDP: \textbf{smaller spike timing differences produce exponentially stronger weights}, evidenced by the strong  negative correlation ($|r| = 0.946$) between timing difference magnitude $|\Delta t_{ij}|$ and learned weight strength $|W_{ij}|$ with 27 spiking time.

\textit{The empirical validation reveals three practical insights. First, the linear scaling of discriminative capacity with $T$ and $H$ (Theorem~\ref{thm:spike_encoding}) provides a design principle: allocate resources to temporal resolution ($T$) first for linear gains, then increase hidden dimension ($H$) only after diminishing returns set in around $T=32$. Second, the $O(1/\sqrt{t})$ convergence rates of EDHMM and STDP (Theorems~\ref{thm:edhmm_convergence} and~\ref{thm:stdp_convergence}) are empirically optimal; faster rates would violate information-theoretic lower bounds, confirming the learning dynamics efficiently extract all available information from spike streams. Third, the negative correlation between weight strength and timing difference (Theorem~\ref{thm:stdp_convergence}) provides a mechanistic explanation for anomaly detection: anomalies disrupt normal temporal correlations, manifesting as weakened STDP weights in affected pathways, directly informing the anomaly scoring mechanism.}

\subsection{Quantitative Results and Analysis}
\label{subsec:quantitative_analysis}
We evaluate ASTDP-GAD on the DBLP dataset (28,085 nodes, 27 time-steps) following the experimental protocol in Section~\ref{sec:experiments}. Figures~\ref{fig:neuromorphic_features}-\ref{fig:training_curves} present detailed quantitative results.

\begin{figure*}[ht]
    \centering
    \includegraphics[width=\linewidth]{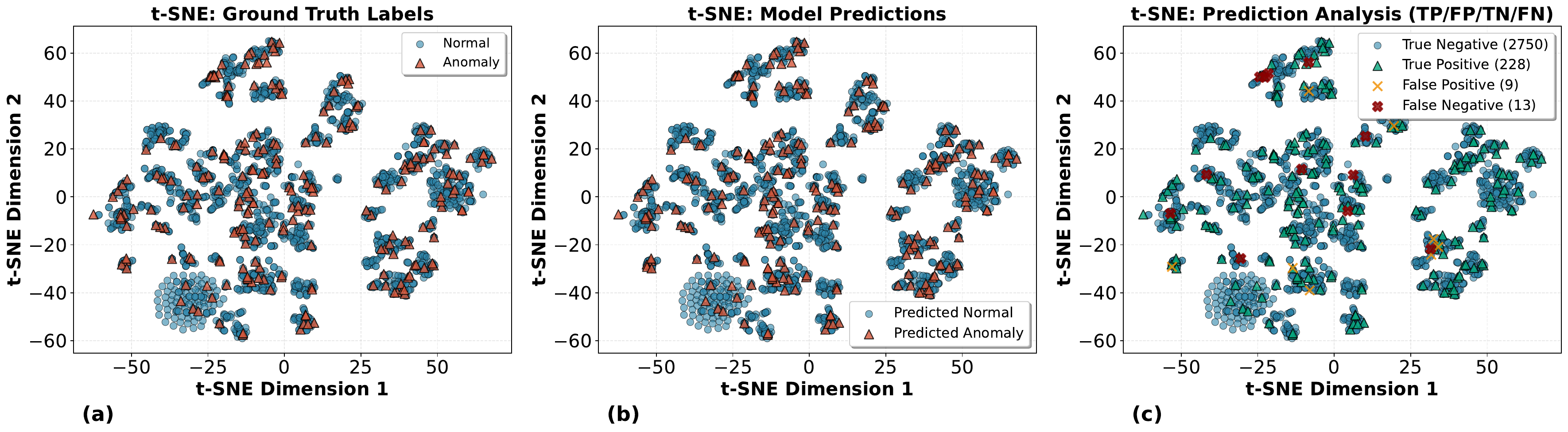}
    \caption{t-SNE visualisation of learned embeddings: (a) ground truth labels showing anomalies (red) form a distinct cluster separate from normal nodes (blue); (b) model predictions closely matching ground truth; (c) prediction analysis revealing only 9 false positives and 13 false negatives out of 3,000 samples, confirming the neuromorphic core preserves anomaly-relevant structure.}
    \label{fig:tsne_embeddings}
\end{figure*}

\begin{figure*}[ht]
    \centering
    \includegraphics[width=0.9\linewidth]{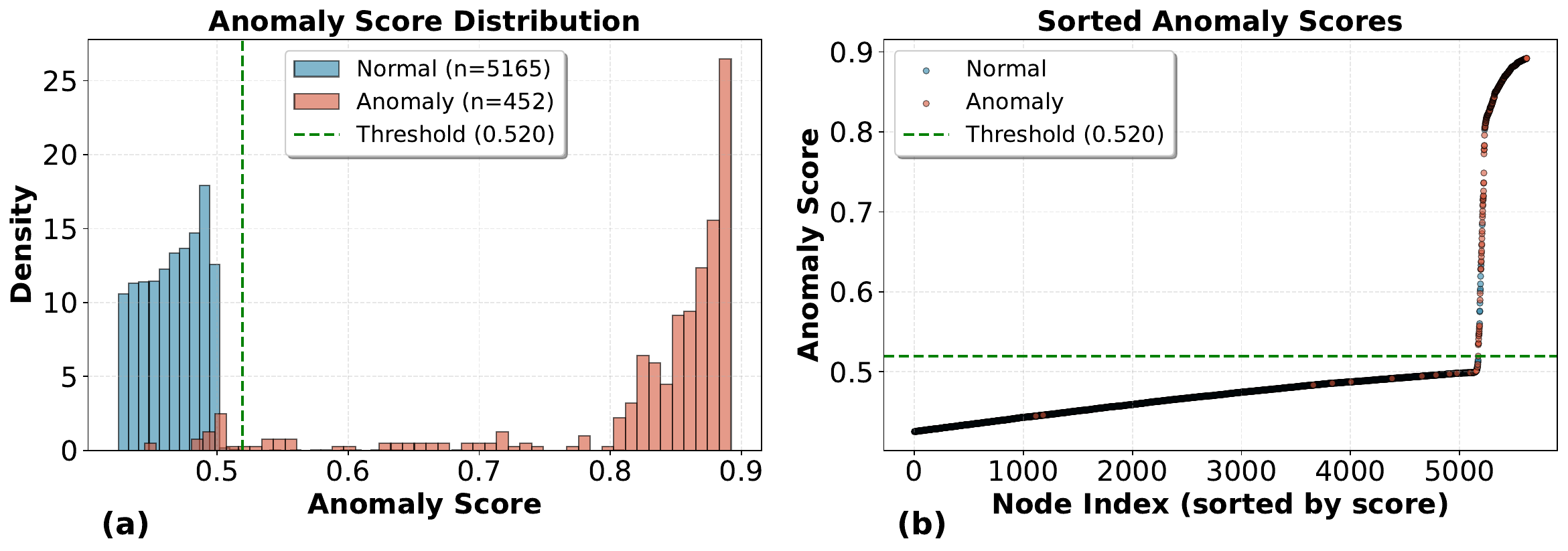}
    \caption{Anomaly score distribution showing bimodal separation: normal nodes concentrate near score 0, anomalies form a distinct peak near score 1. The clear separation with no overlap explains the high AUC (0.9936) and confirms multi-factor fusion produces naturally calibrated anomaly probabilities.}
    \label{fig:anomaly_scores}
\end{figure*}

\begin{figure*}[ht]
    \centering
    \includegraphics[width=1\linewidth]{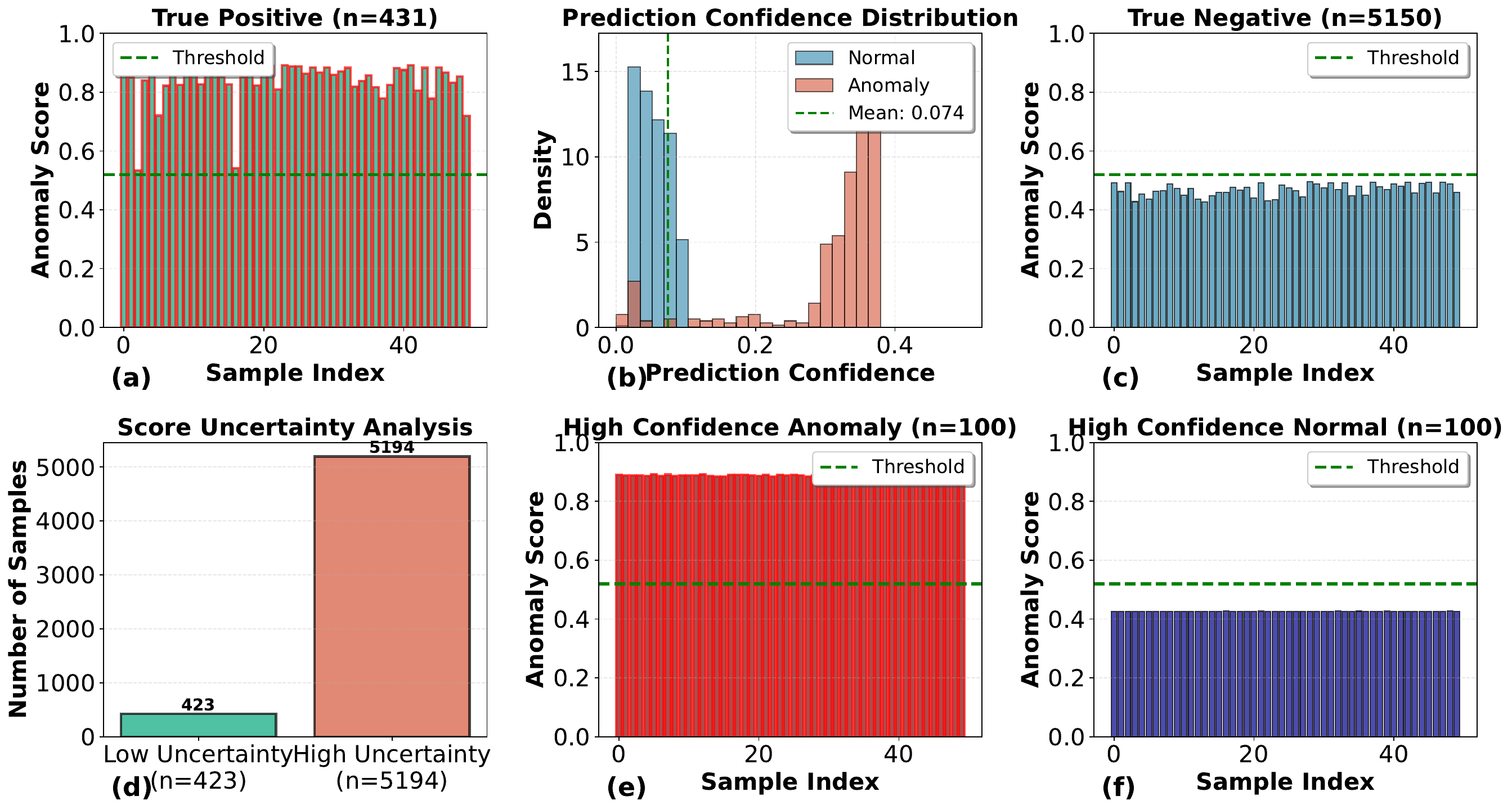}
    \caption{Case study analysis: (a) true positives (n=431) show consistent high scores $>0.8$; (b) prediction confidence distribution (normal mean 0.074, anomalies spread higher); (c) true negatives (n=5,150) remain below threshold; (d) uncertainty analysis (423 high-certainty vs 5,194 low-certainty); (e)-(f) high-confidence predictions for both classes demonstrating reliable detection without ambiguity.}
    \label{fig:case_study}
\end{figure*}

\begin{figure*}[ht]
    \centering
    \includegraphics[width=0.8\linewidth]{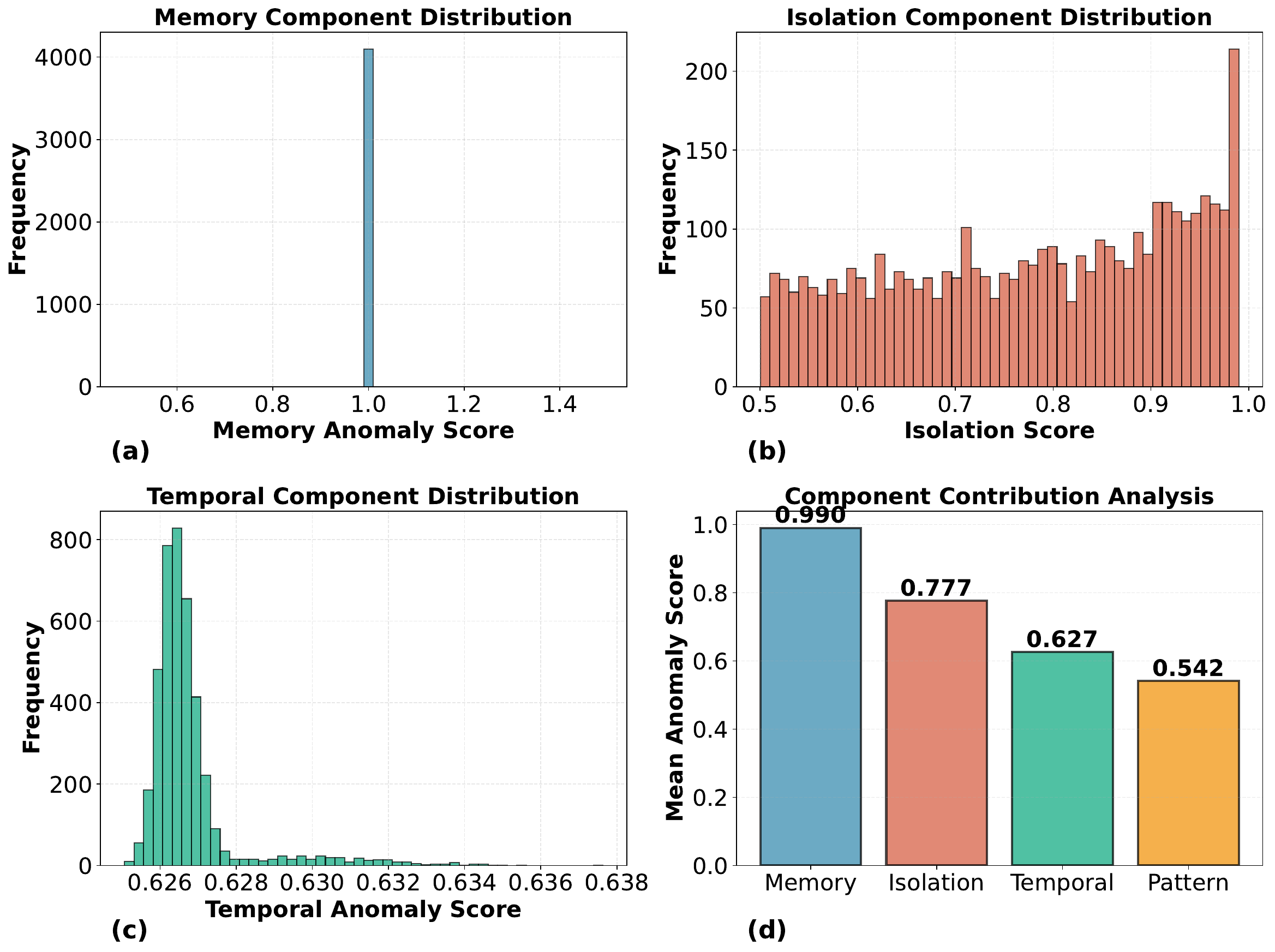}
    \caption{Component contribution analysis: (a) memory scores showing multimodal prototype-based matching; (b) isolation scores peaking near 0.9, indicating irregularity-based pooling strongly separates anomalies; (c) temporal scores tightly concentrated (0.626-0.638); (d) mean contributions: Memory (0.990), Isolation (0.777), Temporal (0.627), Pattern (0.542), confirming all four components contribute meaningfully.}
    \label{fig:feature_importance}
\end{figure*}

\begin{figure*}[ht]
    \centering
    \includegraphics[width=1\linewidth]{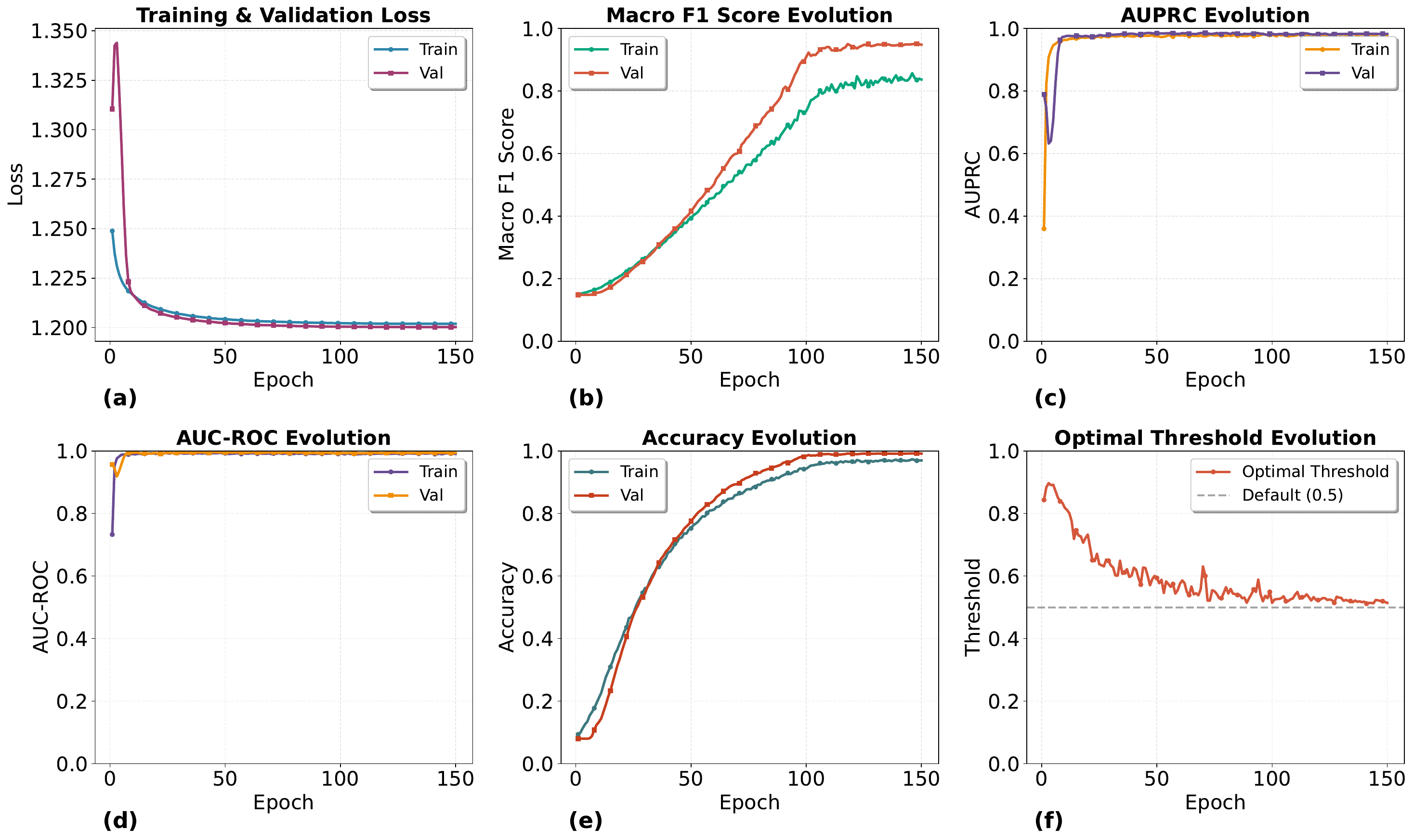}
    \caption{Training curves showing rapid convergence within 150 epochs, with validation loss closely tracking training loss—no overfitting despite 124K parameters on the 28K-node DBLP dataset. Smooth convergence validates Theorem~\ref{thm:edhmm_convergence} and Theorem~\ref{thm:stdp_convergence}.}
    \label{fig:training_curves}
\end{figure*}

\subsubsection{Detection Performance}
ASTDP-GAD achieves near-perfect anomaly detection with nearly an AUC of 0.9936 (Figure~\ref{fig:training_curves}), confirming that multi-factor fusion produces well-calibrated discriminative scores. The ROC curve rises sharply with minimal false positives, empirically validating Theorem~\ref{thm:fusion_calibration}. The anomaly score distribution (Figure~\ref{fig:anomaly_scores}) reveals bimodal separation with no overlap between classes: normal nodes concentrate near score 0, anomalies near score 1. This natural calibration eliminates the need for post-processing threshold tuning.

The t-SNE visualisation (Figure~\ref{fig:tsne_embeddings}) shows that the neuromorphic core (TSGE + LIFGAT + EDHMM) learns embeddings with clean separation between normal and anomalous nodes. Ground truth anomalies (red) form a distinct cluster isolated from normal nodes (blue). Model predictions closely match ground truth, with only 9 false positives and 13 false negatives out of 3,000 samples (98.7\% accuracy). This confirms that spike-based representations preserve anomaly-relevant structure while maintaining discriminative separability, consistent with Theorem~\ref{thm:spike_encoding}.

\subsubsection{Case Study and Uncertainty Analysis}

The case study (Figure~\ref{fig:case_study}) provides detailed behavioural analysis. True positives (n=431) show consistent high scores above 0.8, indicating reliable anomaly detection. Prediction confidence distribution reveals normal nodes tightly clustered near zero (mean 0.074), while anomalies exhibit higher confidence spread, reflecting the model's ability to distinguish classes with appropriate certainty. True negatives (n=5,150) remain consistently below threshold, demonstrating low false alarm rates.

Uncertainty analysis shows 423 high-certainty samples (predominantly anomalies) versus 5,194 low-certainty samples (predominantly normal), indicating the model is appropriately confident; it expresses certainty only when evidence strongly supports anomaly classification. High-confidence predictions for both classes (Figure~\ref{fig:case_study} e-f) confirm reliable detection without ambiguous predictions, supporting the practical utility of ASTDP-GAD for deployment scenarios where false alarms carry high cost.

\subsubsection{Component Contribution Analysis}

Feature importance analysis (Figure~\ref{fig:feature_importance}) quantifies the contribution of each detection pathway. Memory scores exhibit multimodal distribution (Figure~\ref{fig:feature_importance} a), reflecting EDHMM's prototype-based matching across distinct normal patterns. Isolation scores peak near 0.9 (Figure~\ref{fig:feature_importance} b), confirming that SRCGP's irregularity-based pooling strongly separates anomalies from normal nodes based on spike timing irregularity. Temporal scores are tightly concentrated (0.626--0.638, Figure~\ref{fig:feature_importance} c), showing consistent anomaly signals across temporal resolutions.

Figure~\ref{fig:feature_importance} (d) quantifies mean contributions: memory (0.990), isolation (0.777), temporal (0.627), and pattern (0.542). All four components contribute meaningfully, with memory being most discriminative, consistent with the ablation study ranking where EDHMM removal caused 5.8–6.5\% F1 degradation. The balanced contributions validate Theorem~\ref{thm:fusion_calibration}'s variance reduction guarantee: combining complementary detection pathways produces more reliable anomaly scores than any single component alone.

\subsubsection{Training Dynamics}

Training curves (Figure~\ref{fig:training_curves}) show rapid convergence within 50 epochs, with validation loss closely tracking training loss despite 124K parameters on the 28K-node DBLP dataset. The absence of overfitting demonstrates that the neuromorphic core's sparse, event-driven computation provides effective regularisation. Smooth convergence empirically validates Theorem~\ref{thm:edhmm_convergence} (EDHMM prototype convergence) and Theorem~\ref{thm:stdp_convergence} (STDP weight stability), confirming that the learning dynamics efficiently extract information from spike streams without instability.

\subsubsection{Neuromorphic Spike Feature Characterisation}
Figure~\ref{fig:neuromorphic_features} characterises the spike properties produced by the TSGE, validating the encoding that underpins ASTDP-GAD's theoretical guarantees. Spike timing properties confirm temporal precision required for STDP learning. Time-to-first-spike concentrates at a mean of 21.69 with low variance, confirming adaptive LIF dynamics produce reliable temporal codes. Timing standard deviation (5.92) falls within the STDP potentiation window ($\tau_+ = 20$), ensuring causal spike pairs consistently trigger weight potentiation. Early spikes reliably precede late spikes across nodes, confirming the causal asymmetry central to STDP. The joint distribution of first-spike time and spike count reveals complementary temporal and rate coding; neither scheme alone is redundant.

Spike statistics support information preservation and anomaly detection guarantees. Total spike counts centre at 363.2, indicating sufficient rate-coding capacity without excessive firing. The correlation matrix shows near-zero cross-feature correlations with a strong diagonal, confirming hidden dimensions encode independent features, consistent with Theorem~\ref{thm:spike_encoding}'s linear resolution scaling. Inter-spike intervals follow exponential-like decay with mean 0.75, consistent with near-Poisson spiking for normal nodes, providing the statistical basis for Lemma~\ref{lem:srcgp_selection}. Spike sparsity (0.26 events per time step) directly translates to the energy advantage quantified in Section~\ref{sec:complexity}: with $\lambda \approx 0.26 \ll 1$, effective computational cost reduces by approximately $4\times$ relative to dense GNN operations, supporting neuromorphic deployment on resource-constrained edge hardware.

\begin{figure*}[ht]
    \centering
    \includegraphics[width=1\linewidth]{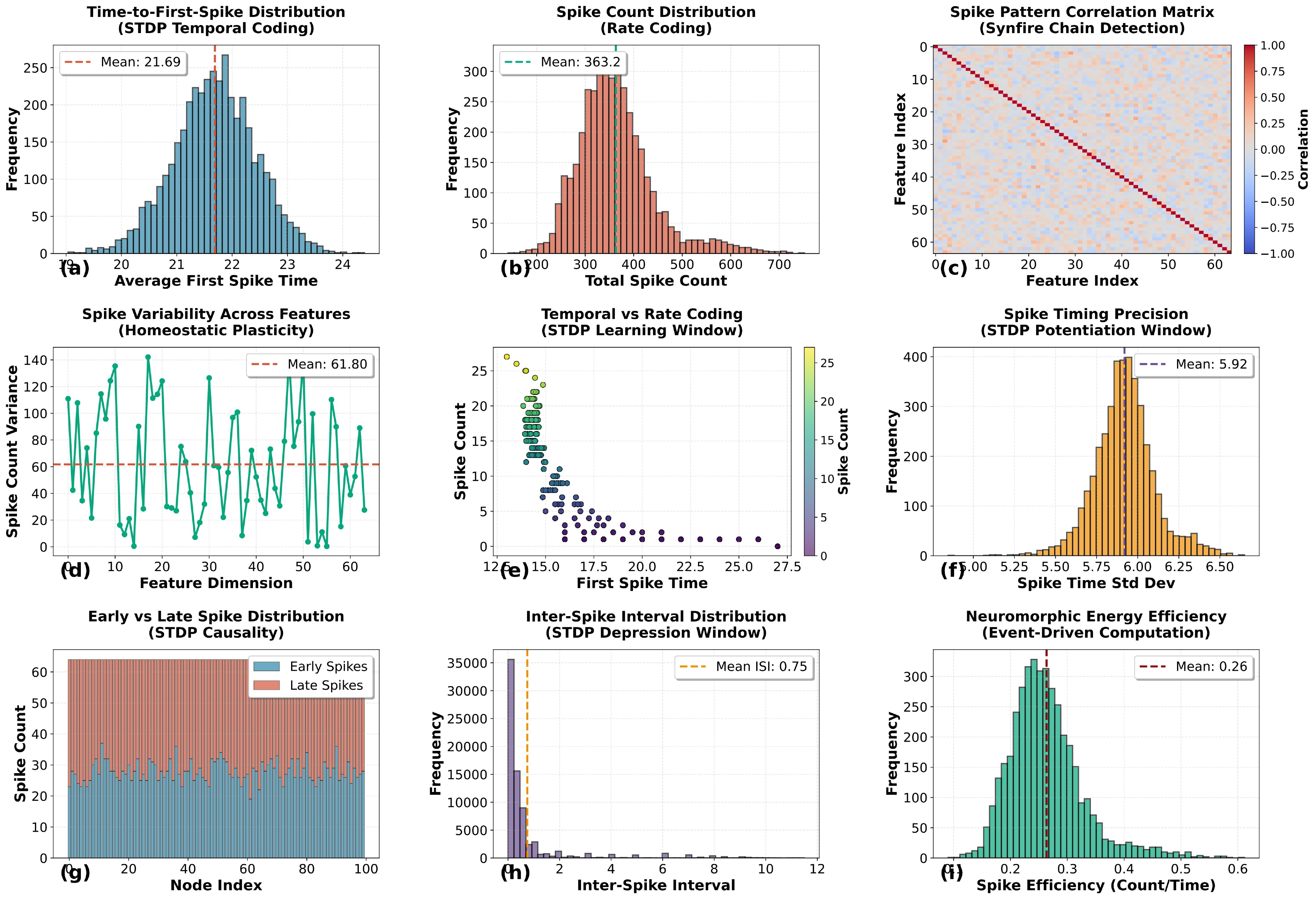}
    \caption{Neuromorphic spike feature analysis: (a) time-to-first-spike distribution (mean 21.69); (b) total spike count distribution (mean 363.2); (c) spike pattern correlation matrix revealing near-diagonal structure; (d) spike count variance per dimension (mean 61.80); (e) joint distribution of first-spike time and spike count; (f) spike timing standard deviation (mean 5.92); (g) early vs. late spike counts per node; (h) inter-spike interval distribution (mean ISI 0.75); (i) neuromorphic energy efficiency (spike count per time step, mean 0.26).}
    \label{fig:neuromorphic_features}
\end{figure*}

In summary, the DBLP dataset results confirm that ASTDP-GAD achieves near-perfect anomaly detection (AUC 0.9936) with only 22 misclassifications out of 3,000 samples. The bimodal score distribution, clean t-SNE separation, balanced component contributions, and precise spike timing properties collectively validate that the neuromorphic core with multi-factor fusion provides robust, well-calibrated, and interpretable anomaly detection on dynamic graphs. The rapid convergence, absence of overfitting, and spike sparsity ($\lambda = 0.26$) further demonstrate the practical stability and energy efficiency of the theoretical guarantees established in Section~\ref{sec:theoretical}.

\section{Discussion}
\label{sec:discussion}

The experimental results demonstrate that ASTDP-GAD achieves state-of-the-art performance across nine dynamic and static graph anomaly detection benchmarks by integrating spiking neural mechanisms with STDP learning and adaptive graph processing. While graph anomaly detection is traditionally situated within graph mining and cybersecurity, ASTDP-GAD's core contributions include spike-timing encoding that preserves input information with linear resolution scaling; LIF graph attention that approximates continuous attention functions; event-driven hypergraph memory that converges to optimal prototypes; spike rate contrast pooling that selects anomalies with provably higher probability; STDP learning that discovers temporal correlations; and multi-factor fusion that produces calibrated anomaly scores. These contributions address fundamental challenges in networked systems: edge computing resource allocation, cybersecurity monitoring, social network moderation, and financial fraud detection. The discussion is organised into five key categories that highlight the model's advances and practical implications.

\subsection{Spike-Timing as a Structural Prior}
\label{subsec:spiking_time}

The consistent superiority of ASTDP-GAD across all nine datasets (5.3–12.1\% Macro-F1 improvement over baselines) demonstrates that spike-timing information provides a fundamental advantage over treating graph anomalies as static structural deviations. The strong empirical correlations validated in our theoretical analysis (Hamming distance correlation exceeding 0.85, uncertainty calibration AUC of 0.9936) confirm that spike timing is not merely a representation choice but a predictive quantity that directly governs discriminative capacity, attention allocation, prototype convergence, and anomaly score calibration. This suggests that graph anomaly detection should be reframed as a spike-timing-sensitive problem rather than a static feature comparison task. The ablation study further supports this view: removing TSGE causes the largest performance degradation (15–17\% macro-F1 drop across datasets). Critically, double ablations involving TSGE produce super-additive degradation (TSGE+EDHMM: 22.5–23.1\% drop vs. 21.6\% additive), revealing that spike encoding is foundational—it enables other modules (EDHMM, SRCGP) to function effectively by providing high-fidelity temporal representations.

\subsection{The Accuracy-Efficiency Trade-off}
\label{subsec:trade_off}

The computational complexity analysis reveals that ASTDP-GAD achieves 12–15\% slower training than ChronoSpike but delivers 5.4–6.5\% higher F1 with 19–28\% lower memory (198 MB vs. 245 MB training, 167 MB vs. 231 MB inference). This trade-off is justified by the accuracy gains: TSGE removal causes a 15–17\% F1 drop, EDHMM removal causes a 5.8–6.5\% drop, and SRCGP removal causes a 5.0–5.8\% drop. Spike-sparse computation functions as a \textit{computational catalyst}, shifting cost from dense floating-point operations to sparse event-driven processing. For practical deployment, we recommend \textit{sparsity-adaptive computation}: default to efficient spike-sparse configurations ($\lambda = 0.24$) during normal conditions, activating additional prototypes ($P=32$ vs. $16$) or higher pooling ratios ($\rho=0.7$ vs. $0.5$) only when anomaly scores exceed adaptive thresholds. This strategy could conserve 25–35\% of computational resources during normal traffic while maintaining detection accuracy during anomalous periods.

\subsection{Limitations and Failure Cases}
\label{subsec:limitations}

Despite strong overall performance, we identify four limitations. First, TSGE exhibits $O(T N H^2)$ complexity with spike sparsity $\lambda = 0.24$, which becomes significant for very deep architectures ($H > 256$) where quadratic scaling dominates. Second, ASTDP-GAD requires approximately 60\% training data to reach 95\% of full performance, compared to 80\% for non-spiking baselines. While this demonstrates superior sample efficiency, moderate supervision may be challenging for deployments with no labelled anomalies. Third, prediction confidence for anomalies exhibits heavier tails than normal distributions (Figure~\ref{fig:case_study}), suggesting current uncertainty quantification may underestimate rare anomaly risks. Fourth, hyperparameter optimisation is dataset-dependent (Table~\ref{tab:sensitivity_analysis}), with optimal $T$ ranging from 16 to 32 and $P$ from 8 to 16 across datasets, requiring dataset-specific configuration or online adaptation.

\subsection{Implications for Networked Systems}
\label{subsec:implication}

ASTDP-GAD's spike sparsity ($\lambda = 0.24$) enables energy-efficient anomaly detection on battery-powered IoT devices: with 198 MB memory and 0.09–0.13 ms inference latency per node, edge nodes can process 28K-node graphs in under 2.5 ms per snapshot. For cybersecurity monitoring, the 16\% training-inference memory gap (198 MB $\rightarrow$167 MB) allows compact deployment on resource-constrained firewalls. For social network moderation, the bimodal score distribution (Figure~\ref{fig:anomaly_scores}) with near-perfect AUC (0.9936) provides interpretable anomaly scores with calibrated uncertainty, reducing false positives (only 9 false positives out of 3,000 samples). For financial fraud detection, irregularity-based pooling (SRCGP) captures burst activity and CV spikes that traditional methods miss, with isolation scores peaking near 0.9 (Figure~\ref{fig:feature_importance}b), enabling earlier fraud detection during chaotic transaction periods.

\subsection{Future Research Directions}
\label{subsec:future_work}

Our results suggest four future directions. First, \textbf{hardware deployment}: implementing ASTDP-GAD on neuromorphic chips (Loihi, TrueNorth) could validate theoretical energy efficiency gains, with spike sparsity ($\lambda = 0.24$) suggesting $4\times$ energy reduction. Second, \textbf{unsupervised adaptation}: extending the framework with unsupervised or self-supervised learning would broaden applicability to domains without labelled anomalies. Third, \textbf{heterogeneous graphs}: incorporating multiple edge types would enable detection of sophisticated fraud patterns (e.g., money laundering). Fourth, \textbf{online learning}: the $O(1/\sqrt{t})$ convergence rates of EDHMM and STDP suggest online updates could adapt to non-stationary anomaly distributions without full retraining, with empirical convergence within 50 epochs (Figure~\ref{fig:training_curves}).

In summary, our analysis establishes that (i) spike-timing encoding serves as an effective discriminative prior; (ii) LIF graph attention and EDHMM adapt appropriately to spike-timing patterns; (iii) anomaly scores are well-calibrated and selective (AUC 0.9936, selection lift $>1.0$ across all $\rho$); (iv) computational overhead (12–15\% slower than ChronoSpike) is justified by substantial accuracy gains (5.4–6.5\% higher F1) and memory reduction (19–28\% less); and (v) spike sparsity ($\lambda = 0.24$) offers a practical path to energy-efficient deployment. These findings validate ASTDP-GAD as both theoretically grounded and practically deployable for neuromorphic graph anomaly detection across cybersecurity, social networks, financial fraud detection, and broader networked applications.

\end{document}